\documentclass[sigconf,nonacm,pbalance]{acmart}
\usepackage{placeins}
\usepackage{pdflscape}

\usepackage{iftex}
\RequireXeTeX

\usepackage{amsmath}
\usepackage{amssymb}
\usepackage{array}
\usepackage{colortbl}
\microtypesetup{expansion=false}
\usepackage{etoolbox}
\AtBeginEnvironment{tabular}{\addfontfeatures{Numbers={Lining,Monospaced}}}
\AtBeginEnvironment{tabular*}{\addfontfeatures{Numbers={Lining,Monospaced}}}
\makeatletter
\apptocmd{\normalsize}{%
  \advance\baselineskip 0pt plus 0.6pt\relax
  \abovedisplayskip=5pt plus 1pt minus 3pt\relax
  \belowdisplayskip=5pt plus 1pt minus 3pt\relax
  \abovedisplayshortskip=1pt plus 1pt minus 1pt\relax
  \belowdisplayshortskip=3pt plus 1pt minus 2pt\relax}{}{}
\makeatother
\title[]{ScienceClaw: Benchmarking Continual Self-Evolution of
  AI-for-Science Agents Across the Natural and Social Sciences}

\author{Mingda Zhang}
\affiliation{%
  \institution{The Chinese University of Hong Kong, Shenzhen}
  \city{Shenzhen}
  \country{China}}
\email{mingdazhang@ieee.org}

\author{Wenjin Liu}
\affiliation{%
  \institution{Nanyang Technological University}
  \city{Singapore}
  \country{Singapore}}
\email{wenjinliu23@outlook.com}

\author{Tiesunlong Shen}
\authornote{Corresponding authors.}
\affiliation{%
  \institution{Centre for Biomedical Data Science, Duke-NUS Medical School}
  \city{Singapore}
  \country{Singapore}}
\email{tiesunlong@nus.edu.sg}

\author{Zikai Xiao}
\affiliation{%
  \institution{Zhejiang University}
  \city{Hangzhou}
  \country{China}}
\email{zikai@zju.edu.cn}

\author{Zhenghong Lin}
\affiliation{%
  \institution{Nanyang Technological University}
  \city{Singapore}
  \country{Singapore}}
\email{hongzhenglin970323@gmail.com}

\author{Qing Xu}
\affiliation{%
  \institution{Jurisprudence Research Association, China Law Society}
  \city{Beijing}
  \country{China}}
\email{20180090@ynu.edu.cn}

\author{Erik Cambria}
\affiliation{%
  \institution{Nanyang Technological University}
  \city{Singapore}
  \country{Singapore}}
\email{cambria@ntu.edu.sg}

\author{Xiaoying Tang}
\affiliation{%
  \institution{The Chinese University of Hong Kong, Shenzhen}
  \city{Shenzhen}
  \country{China}}
\email{tangxiaoying@cuhk.edu.cn}

\author{Haoran Luo}
\authornotemark[1]
\affiliation{%
  \institution{Nanyang Technological University}
  \city{Singapore}
  \country{Singapore}}
\email{haoran.luo@ntu.edu.sg}

\renewcommand{\shortauthors}{Zhang et al.}

\begin{document}

\begin{abstract}
Large language model agents are accelerating scientific automation, yet
verified executions rarely become persistent program-level improvements, and
existing evaluations do not examine this process across sequential tasks in
both the natural and social sciences. We formalize ScienceClaw as
fixed-parameter program self-evolution that unifies task solving, scientific
verification, and program updates. ScienceClaw-Eval spans 23 disciplines and
measures scientific correctness, evolutionary gain, retention, cross-dataset
transfer, and evolution cost through sequential streams and independent reset
evaluation. Our framework repairs executable workflows through multi-turn
interaction, converts re-execution-verified failure--success trajectories into
linked Skill and Operator candidates, and retains an update only when
source-task replay reproduces the repair and independent scientific tasks
improve. Code is available at \url{https://github.com/beita6969/ScienceClaw}.
\end{abstract}



\maketitle

\begin{figure}[t]
  \vspace*{16.8pt}
  \centering
  \includegraphics[width=\columnwidth]{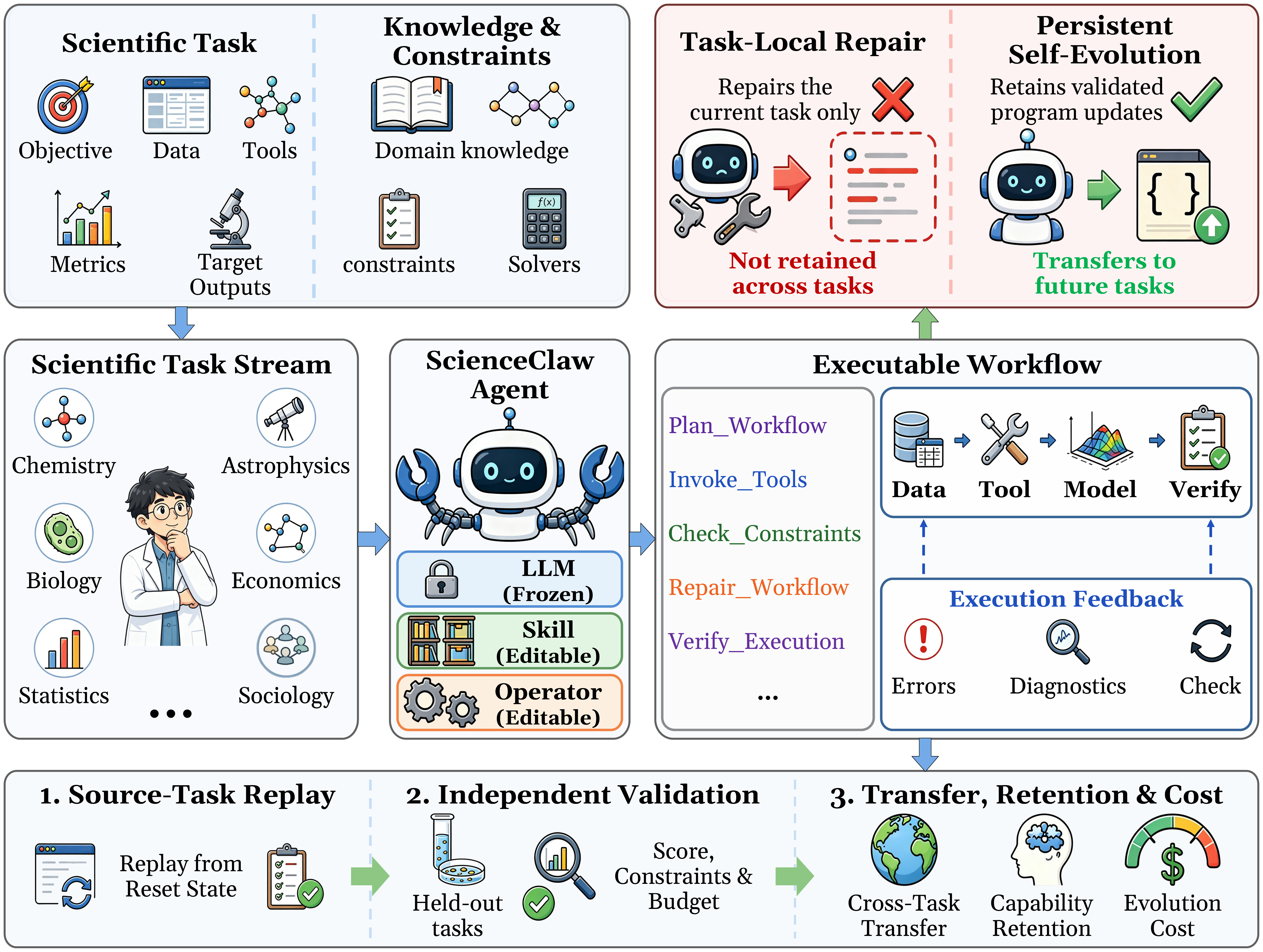}
  \caption{Task formulation of verifiable program-level self-evolution for
  AI-for-Science agents.}
  \Description{Scientific tasks and domain constraints form sequential task
  streams. ScienceClaw uses a fixed foundation model with editable Skills and
  typed Operators to construct executable workflows, and evaluates repairs by
  source replay, independent validation, transfer, retention, and cost.}
  \label{fig:task-formulation}
\end{figure}

\begin{figure*}[t]
  \centering
  \includegraphics[width=\textwidth]{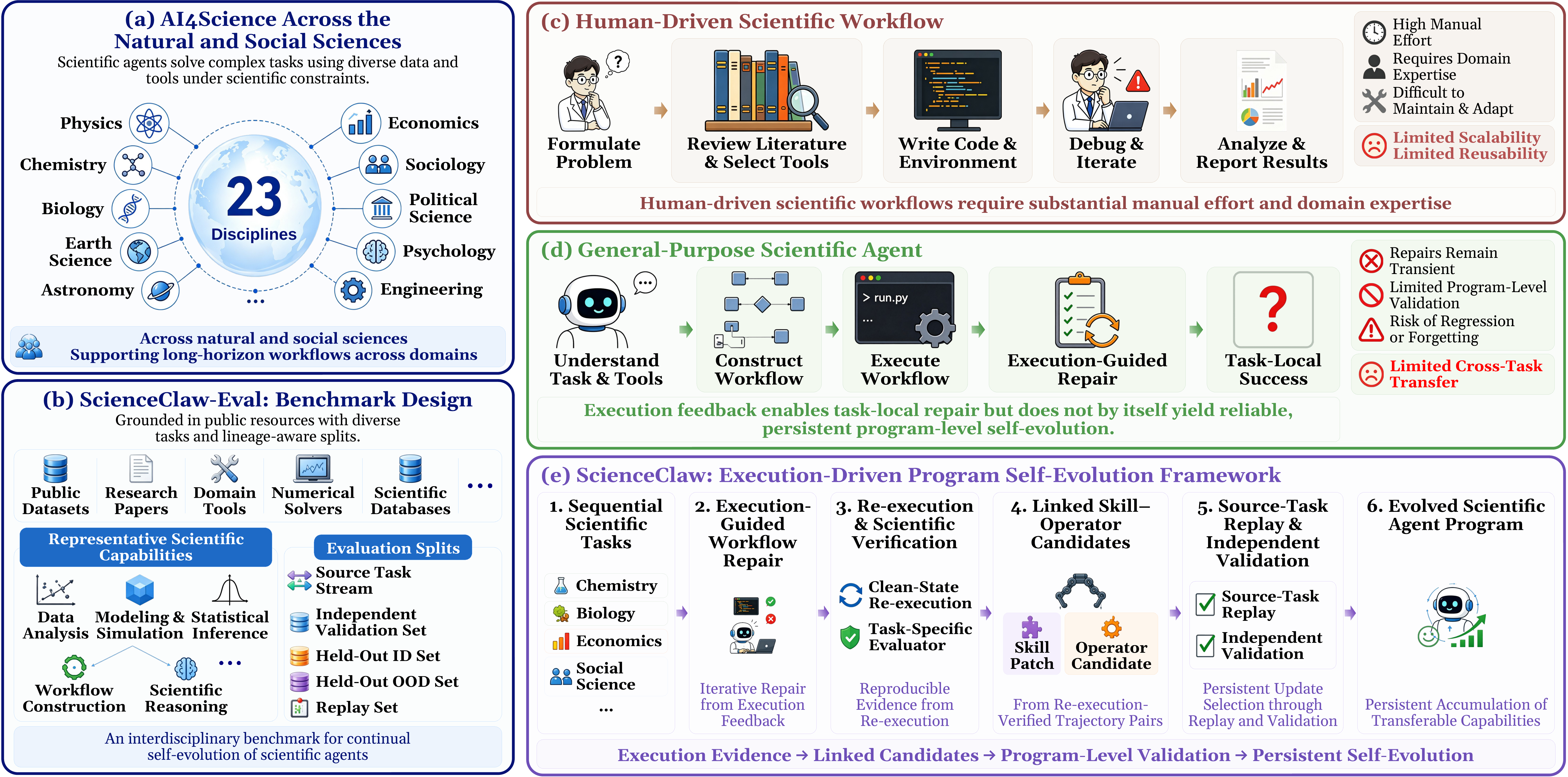}
  \caption{Overview of ScienceClaw. ScienceClaw-Eval spans 23 disciplines
  across the natural and social sciences, while ScienceClaw turns verified
  execution evidence into persistent Skill--Operator program updates.}
  \Description{Five panels show the 23-discipline AI-for-Science scope,
  ScienceClaw-Eval benchmark design, limitations of human-driven scientific
  workflows and general-purpose scientific agents, and the ScienceClaw
  framework for persistent program self-evolution.}
  \label{fig:scienceclaw-overview}
\end{figure*}

\section{Introduction}
\label{sec:introduction}

Large language model agents are increasingly automating scientific problem solving~\cite{bran2023chemcrow,lu2024ai}. As illustrated in Figure~\ref{fig:task-formulation}, they connect
natural-language reasoning with data, domain tools, and executable computation
to construct end-to-end workflows~\cite{yao2022react,wang2024executable}. AI4Science tasks stress this capability
through heterogeneous dependencies, long-horizon interactions, and stringent
constraints, yet also expose runtime states, numerical diagnostics, and
constraint checks that verify intermediate decisions. Scientific agents
therefore offer a flexible cross-domain interface beyond fixed scripts and
rules. Despite this promise, scientific agents remain limited in
continual self-evolution, since repairs rarely persist.

Training-based methods improve workflow orchestration and error repair through
supervised fine-tuning, reinforcement learning, or separate task optimization,
but acquire most improvements before
deployment~\cite{shen2026sciagentgym,zhang2026flowsteer,zhao2026a2flow,agrawal2026gepa}. Training-free methods instead update tools, Skills, control
logic, or agent harnesses from execution trajectories
~\cite{lu2026beyond,yang2026skillopt,ni2026trace2skill,ling2026pace,lin2026agentic,shinn2023reflexion,wang2023voyager,machine2025open}.
Scientific efforts, however, remain concentrated on individual tools or
Skills~\cite{lu2026beyond,team2026s1}, while program-level evolution is evaluated
mainly in general reasoning, coding, or system-operation environments. No
unified setting therefore couples in-task verification of executability and
scientific validity, linked Skill--Operator updates, and independent
program-level evaluation. Scientific-agent and sequential self-evolution benchmarks do not
cover constrained task streams spanning the natural and social
sciences~\cite{shen2026sciagentgym,liu2026benchmarking,xu2026researchclawbench,jiang2026sea,zheng2026seagym,alizadeh2026ai}. Thus, task- or component-level gains do not establish
reliable and transferable scientific self-evolution across disciplines and rounds.

This raises a central question: \textbf{How can verified scientific executions
drive persistent and transferable program-level self-evolution without
updating foundation-model parameters?} We answer it with ScienceClaw, a
verifiable AI4Science framework connecting multi-turn workflow orchestration
and program-level self-evolution. ScienceClaw represents the editable program
as Skills and typed Operators, repairs workflows through execution feedback,
and derives linked candidates from re-execution-verified failure--success
trajectories. A candidate persists only if source-task replay reproduces the
repair and independent scientific tasks improve.
Figure~\ref{fig:scienceclaw-overview} summarizes this process. In summary, our three main
contributions are as follows:

\begin{itemize}
  \item \textbf{Novel ScienceClaw Task.} We formalize program-level
  self-evolution over sequential AI4Science problems, turning verified
  executions into persistent and transferable updates without modifying
  foundation-model parameters.

  \item \textbf{ScienceClaw-Eval.} We construct a 23-discipline benchmark
  across the natural and social sciences with sequential task streams and
  independent reset evaluations.

  \item \textbf{ScienceClaw Framework.} We propose an execution-driven
  framework coupling multi-turn workflow repair with independently evaluated
  Skill--Operator updates, allowing scientific capabilities to accumulate
  across tasks and disciplines without model training.
\end{itemize}

\begin{figure*}[t]
  \centering
  \includegraphics[width=\textwidth]{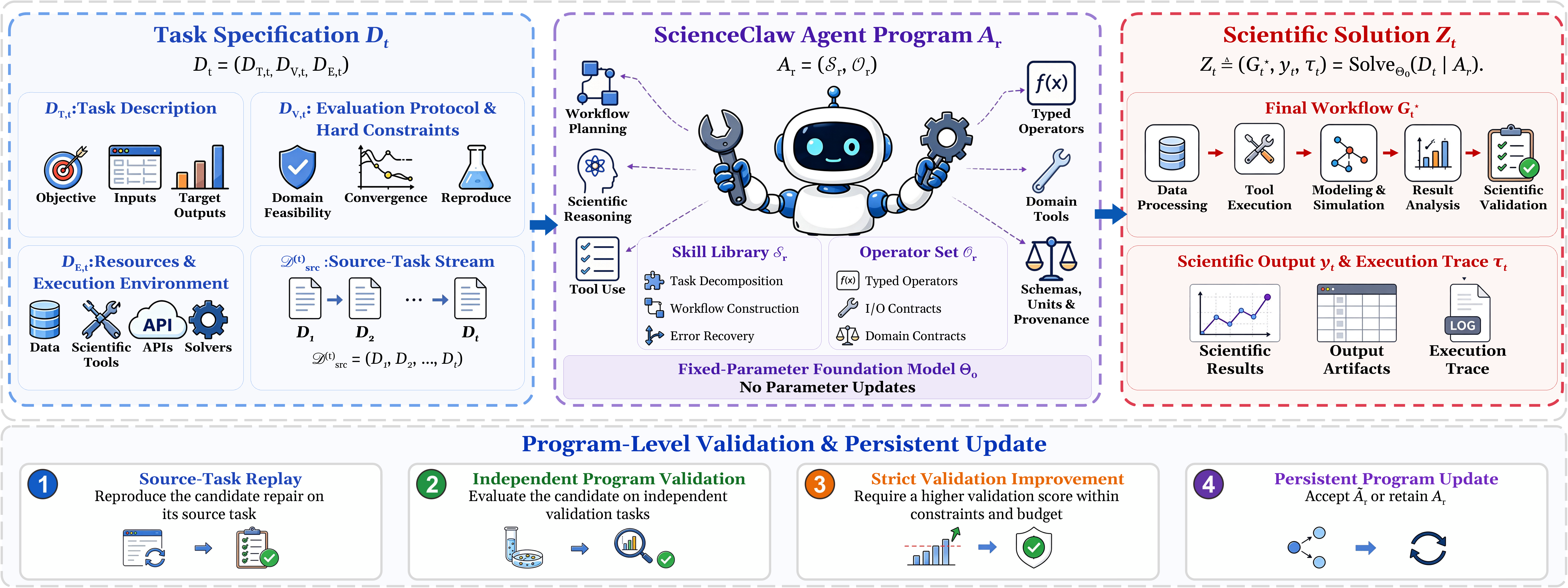}
  \caption{Definition and overview of the ScienceClaw task. Given task
  specification $D_t$ and agent program $A_r$, ScienceClaw produces scientific
  solution $Z_t$ and retains a candidate update only after source-task replay
  and independent program validation.}
  \Description{A source-task stream supplies structured scientific-task
  specifications to a fixed-parameter agent with editable Skills and typed
  Operators. The agent constructs scientific solution $Z_t$, and a candidate
  program is retained only after source-task replay, independent validation,
  strict improvement, and a persistent-update decision.}
  \label{fig:task-overview}
\end{figure*}

\begin{figure*}[t]
  \centering
  \includegraphics[width=\textwidth]{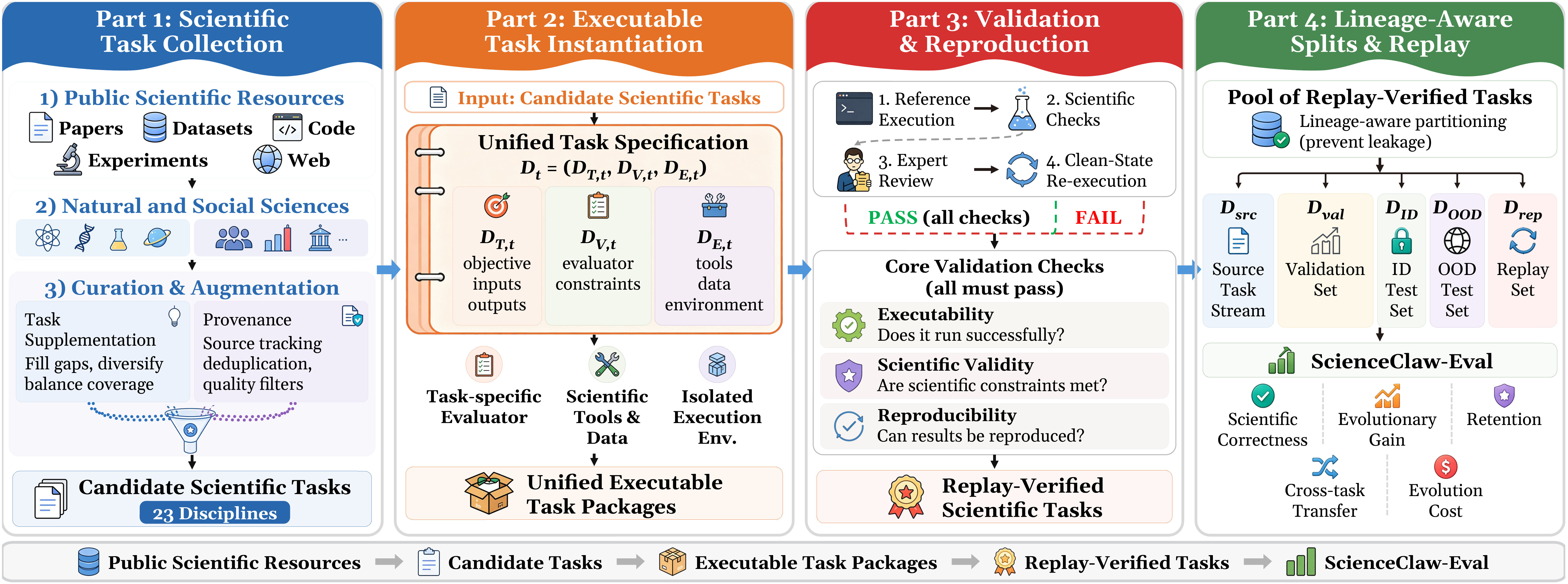}
  \caption{Construction of ScienceClaw-Eval, from scientific-task collection
  and executable instantiation to scientific validation, reproducibility
  checking, and lineage-aware evaluation splits.}
  \Description{A four-stage pipeline collects natural- and social-science
  tasks, packages each task with an evaluator and isolated environment,
  verifies executability, scientific validity, and reproducibility, and creates
  source, validation, in-distribution, out-of-distribution, and replay splits.}
  \label{fig:benchmark-construction}
\end{figure*}

\section{Related Work}
\label{sec:related-work}

\textbf{Scientific Agents and Evaluation.} General-purpose LLM agents continue
to exhibit limited reliability and generalization in specialized scientific
tasks, motivating domain-specific tools, interactive execution environments,
and scientific-agent benchmarks~\cite{shen2026sciagentgym,alizadeh2026ai,tian2024scicode,chen2025scienceagentbench}.
Evaluation has progressed from static question answering to executable
workflows that connect scientific data, domain software, and code execution
~\cite{shen2026sciagentgym,lupidi2026airs}. Interactive environments assess multi-step tool use and error recovery~\cite{wang2022scienceworld}, while more recent benchmarks extend
evaluation to the end-to-end research lifecycle, including idea generation,
experiment analysis, and iterative refinement using real papers and data
~\cite{shen2026sciagentgym,liu2026benchmarking,lupidi2026airs,xu2026researchclawbench,bragg2026astabench}. Discipline-specific benchmarks further examine the
computational reproducibility of social-science findings and the physical
fidelity of scientific programs~\cite{alizadeh2026ai,zhang2026orbit}.
ScienceClaw-Eval instead assesses whether agents accumulate and
retain capabilities across tasks.

\textbf{Workflow Orchestration and Agent Self-Evolution.} Early LLM agents
mainly relied on hand-crafted prompts, static toolsets, or predefined workflows,
with tool interfaces and control flow typically fixed before deployment
~\cite{yao2022react,hong2024metagpt,wu2023autogen}. Automated workflow
optimization subsequently introduced execution-guided search and progressive
graph editing, while later methods abstract reusable Operators or distill
Skills from execution trajectories
~\cite{zhang2026flowsteer,zhao2026a2flow,yang2026skillopt,ni2026trace2skill,ding2026atomic}. More recent work evolves tools, Skills, control logic, and agent
harnesses without updating foundation-model parameters
~\cite{lu2026beyond,team2026s1,ling2026pace,lin2026agentic,shinn2023reflexion,wang2023voyager,machine2025open}. In parallel, sequential
evaluation frameworks measure evolutionary gain, cross-task transfer,
retention, and forgetting~\cite{jiang2026sea,zheng2026seagym,lopez2017gradient}. Although
scientific agents begin to evolve tools or Skills, program-level
self-evolution has not been evaluated on executable task streams
spanning the natural and social sciences, the gap ScienceClaw-Eval targets.

\section{Task: ScienceClaw}
\label{sec:task}

ScienceClaw formalizes continual program-level self-evolution over sequential
AI4Science problems while keeping foundation-model parameters fixed. As shown
in Figure~\ref{fig:task-overview}, verified evidence from long-horizon tool
interaction is converted from transient task context into persistent program
updates.

Let $D_t=(D_{T,t},D_{V,t},D_{E,t})$ denote task $t$. Here, $D_{T,t}$
specifies its objective, inputs, initial conditions, and target outputs;
$D_{V,t}$ defines the evaluation protocol and hard constraints, including unit
consistency, domain feasibility, convergence, and reproducibility; and
$D_{E,t}$ provides data, scientific tools, solvers, simulators, and the
execution environment. At round $r$, the editable program is
$A_r=(\mathcal S_r,\mathcal O_r)$, where $\mathcal S_r$ contains Skills for
decomposition, workflow construction, and recovery, and $\mathcal O_r$
contains typed executable Operators with explicit input, output, and domain
contracts. Conditioned on $D_t$ and $A_r$, a fixed-parameter foundation model
constructs and executes
\begin{equation}
  Z_t \triangleq
  \left(G_t^{\star},y_t,\tau_t\right)
  =\operatorname{Solve}_{\Theta_0}\!\left(D_t\mid A_r\right).
  \label{eq:task-solving}
\end{equation}
The scientific solution $Z_t$ comprises the final workflow $G_t^{\star}$,
replayed from a reset task state, its output $y_t$, and execution trace
$\tau_t$. A task-specific evaluator checks both $y_t$ and $\tau_t$ against
$D_{V,t}$, making validity depend on the result and its computation.

Let $\mathcal E_r$ be the verified evidence accumulated through round $r$ and
$\mathcal P(\mathcal E_r)$ its induced candidate programs. Feasibility and
selection are defined by
\begin{equation}
\begin{aligned}
\mathcal F_r
  &=\left\{A\in\mathcal P(\mathcal E_r)\;\middle|\;
    \begin{gathered}
      R_{\mathrm{src}}(A)=1,\quad
      \mathbf H_{\mathrm{val}}(A)=\mathbf 1,\\
      \mathbf C_{\mathrm{val}}(A)\preceq\mathbf B
    \end{gathered}\right\},\\
\widetilde A_r
  &\in\underset{A\in\mathcal F_r}{\arg\max}\;Q_{\mathrm{val}}(A),
    \qquad \mathcal F_r\neq\varnothing.
\end{aligned}
  \label{eq:candidate-selection}
\end{equation}
The persistent program update is then
\begin{equation}
\begin{aligned}
A_{r+1}
  &=\begin{cases}
    \widetilde A_r,
      & \mathcal F_r\neq\varnothing\ \land\
        Q_{\mathrm{val}}(\widetilde A_r)>Q_{\mathrm{val}}(A_r),\\
    A_r, & \text{otherwise},
  \end{cases}\\
\Theta_{r+1}&=\Theta_r=\Theta_0.
\end{aligned}
  \label{eq:program-update}
\end{equation}
Here, $\mathcal F_r$ is the feasible set; $R_{\mathrm{src}}(A)=1$ requires
source-task reproduction of every new repair; and $Q_{\mathrm{val}}(A)$ is the
aggregate scientific score on independent validation tasks.
$\mathbf H_{\mathrm{val}}(A)$ collects non-negotiable execution and scientific
integrity checks, with $\mathbf 1$ the all-ones vector. Unsolved tasks lower
$Q_{\mathrm{val}}$, while hard-constraint violations make
$\mathbf H_{\mathrm{val}}(A)\ne\mathbf 1$.
$\mathbf C_{\mathrm{val}}(A)$ and $\mathbf B$ are the resource-cost vector and
component-wise budget; $\preceq$ denotes component-wise comparison.
Equation~(\ref{eq:program-update}) keeps $\Theta_0$ fixed and retains the
incumbent unless a feasible candidate strictly improves validation performance.

\section{Benchmark: ScienceClaw-Eval}
\label{sec:benchmark}

ScienceClaw-Eval benchmarks continual scientific-agent self-evolution with 23
research tasks across 23 natural- and social-science disciplines. It combines
sequential task processing with independent reset evaluation. Systems share the
orchestration model, initial program $A_0$, source order, tools, and budget. We
describe its construction, protocol, and metrics (Appendix~\ref{app:bench-spec}).

\subsection{ScienceClaw-Eval Construction}
\label{sec:benchmark-construction}

Figure~\ref{fig:benchmark-construction} shows four stages: task collection,
executable instantiation, validation, and evaluation partitioning. We collect
tasks with explicit objectives, accessible inputs, and verifiable outcomes from
public scientific resources, and supplement them to improve disciplinary
coverage, workflow diversity, and horizon length while preserving provenance.
Each task is instantiated as $D_t$ with an isolated environment and a
task-specific evaluator. Domain reviewers inspect evaluator logic, constraints,
and reference executions; reset replay then verifies executability, scientific
validity, and reproducibility. Provenance and construction lineage~\cite{wilkinson2016fair} determine
the source stream $\mathcal D_{\mathrm{src}}$, independent validation set
$\mathcal D_{\mathrm{val}}$, held-out ID set $\mathcal D_{\mathrm{ID}}$, and
same-discipline cross-dataset OOD set $\mathcal D_{\mathrm{OOD}}$~\cite{koh2021wilds}.
$\mathcal D_{\mathrm{rep}}$ contains frozen source-task instances encountered
before each evaluation point, measuring retention without supplying new
evolution evidence. Each task record fixes its source, partition rule,
evaluator, and statistics before experiments begin.

\subsection{ScienceClaw-Eval: Evaluation Design}
\label{sec:benchmark-evaluation}

Only source trajectories and evaluator feedback supply evolution evidence;
$\mathcal D_{\mathrm{val}}$ serves candidate selection. At each evaluation
point, $A_r$ is frozen, transient state is cleared, and the snapshot is
evaluated on $\mathcal D_{\mathrm{ID}}$, $\mathcal D_{\mathrm{OOD}}$, and
$\mathcal D_{\mathrm{rep}}$.

For instance $i$, the evaluator jointly examines output and trace.
$z_i(A_r)=1$ only when execution completes within budget, task-native metrics
meet their acceptance rule (Appendix~\ref{app:bench-evaluator}), and every scientific hard constraint holds;
otherwise $z_i(A_r)=0$. We macro-average discipline-level success:
\begin{equation}
  \operatorname{MacroSR}(A_r)
  =\frac{1}{23}\sum_{d=1}^{23}
    \frac{1}{|\mathcal I_d|}
    \sum_{i\in\mathcal I_d} z_i(A_r),
  \label{eq:macro-success-rate}
\end{equation}
where $\mathcal I_d$ contains instances in discipline $d$.
$Q_{\mathrm{val}}$ in Eq.~(\ref{eq:candidate-selection}) is MacroSR on
$\mathcal D_{\mathrm{val}}$. Changes from $A_0$ on
$\mathcal D_{\mathrm{ID}}$ and $\mathcal D_{\mathrm{OOD}}$ measure in-benchmark
gain and same-discipline cross-dataset transfer; $\mathcal D_{\mathrm{rep}}$
measures retention against the snapshot immediately after each source task.
We also report
natural- and social-science success, task-native metrics, multi-turn repair,
long-horizon completion, and execution-cost diagnostics. Evaluators,
partitions, and executable configurations remain frozen across methods.

\def\ScienceClawSectionOnly{1}
\ifdefined\ScienceClawSectionOnly
\let\ScienceClawEnd\relax
\else
\documentclass[sigconf,nonacm]{acmart}

\usepackage{amsmath}
\ifPDFTeX
  \usepackage{amssymb}
\fi
\microtypesetup{expansion=false}
\setlength{\emergencystretch}{1em}


\title[ScienceClaw]{ScienceClaw: Benchmarking Continual Self-Evolution of
  AI-for-Science Agents Across the Natural and Social Sciences}

\settopmatter{authorsperrow=3}

\author{Mingda Zhang}
\affiliation{%
  \institution{The Chinese University of Hong Kong, Shenzhen}
  \city{Shenzhen}
  \country{China}}
\email{mingdazhang@ieee.org}

\author{Wenjin Liu}
\affiliation{%
  \institution{Nanyang Technological University}
  \city{Singapore}
  \country{Singapore}}
\email{wenjinliu23@outlook.com}

\author{Tiesunlong Shen}
\authornote{Corresponding authors.}
\affiliation{%
  \institution{Centre for Biomedical Data Science, Duke-NUS Medical School}
  \city{Singapore}
  \country{Singapore}}
\email{tiesunlong@nus.edu.sg}

\author{Zikai Xiao}
\affiliation{%
  \institution{Zhejiang University}
  \city{Hangzhou}
  \country{China}}
\email{zikai@zju.edu.cn}

\author{Zhenghong Lin}
\affiliation{%
  \institution{Nanyang Technological University}
  \city{Singapore}
  \country{Singapore}}
\email{hongzhenglin970323@gmail.com}

\author{Qing Xu}
\affiliation{%
  \institution{Jurisprudence Research Association, China Law Society}
  \city{Beijing}
  \country{China}}
\email{20180090@ynu.edu.cn}

\author{Erik Cambria}
\affiliation{%
  \institution{Nanyang Technological University}
  \city{Singapore}
  \country{Singapore}}
\email{cambria@ntu.edu.sg}

\author{Xiaoying Tang}
\affiliation{%
  \institution{The Chinese University of Hong Kong, Shenzhen}
  \city{Shenzhen}
  \country{China}}
\email{tangxiaoying@cuhk.edu.cn}

\author{Haoran Luo}
\authornotemark[1]
\affiliation{%
  \institution{Nanyang Technological University}
  \city{Singapore}
  \country{Singapore}}
\email{haoran.luo@ntu.edu.sg}

\renewcommand{\shortauthors}{Zhang et al.}

\begin{document}

\begin{abstract}
This file contains the current English draft of the ScienceClaw methodology.
The complete abstract will be added after the benchmark and experiments are finalized.
\end{abstract}


\maketitle

\setcounter{section}{4}
\setcounter{equation}{4}
\def\ScienceClawEnd{\end{document}}
\fi

\section{Methodology: ScienceClaw}
\label{sec:method}

Scientific workflow failures often surface only when heterogeneous data,
domain software, and numerical procedures execute. ScienceClaw turns
reproducible repairs into linked Skill--Operator updates through
execution-guided orchestration (Section~\ref{sec:workflow}) and program
self-evolution (Section~\ref{sec:self-evolution}), with foundation-model
parameters fixed. A repair persists only after it reproduces on its source task
and improves the complete program on independent tasks, rather than on a single execution trace
(Appendix~\ref{app:impl}).

\subsection{Execution-Guided Scientific Workflow Orchestration}
\label{sec:workflow}

\paragraph{Typed scientific workflow representation.}
ScienceClaw represents the workflow of task $t$ after interaction step $k$ as
a directed graph with port-level data dependencies:
\begin{equation}
\begin{aligned}
G_{t,k}
  &=\bigl(N_{t,k},E_{t,k},\psi_{t,k}\bigr),\\
\psi_{t,k}(v)
  &=\bigl(o_v,\eta_v\bigr),\quad v\in N_{t,k},\\
\bigl((u,p),(v,q)\bigr)\in E_{t,k}
  &\Longrightarrow
  \operatorname{Compat}_{D_t}\!\left(
    \Gamma_{u,p}^{\mathrm{out}},
    \Gamma_{v,q}^{\mathrm{in}}
  \right)=1.
\end{aligned}
\label{eq:typed-workflow}
\end{equation}
$N_{t,k}$ contains Operators and task tools; $\psi_{t,k}$ maps node $v$ to
executable $o_v$ and configuration $\eta_v$; and an edge links output port $p$
of $u$ to input port $q$ of $v$. Schema
$\Gamma=(\mathrm{type},\mathrm{shape},\mathrm{unit},\mathrm{provenance})$
records port metadata. $\operatorname{Compat}_{D_t}$ requires matching type and
shape and explicit, provenance-recorded unit conversion; the same schemas
define Operator boundaries below. Units are required when applicable, and
provenance records upstream data and the transformations applied along every edge.

\paragraph{Execution-guided orchestration.}
ScienceClaw constructs $G_{t,k}$ through atomic edits after tool execution
~\cite{zhang2026flowsteer,wang2024executable}. Let
$\mathcal R^{\mathcal S}_{t,r}=\operatorname{Retrieve}(D_t;\mathcal S_r)$ and
$\mathcal R^{\mathcal O}_{t,r}=\operatorname{Retrieve}(D_t;\mathcal O_r)$ be
retrieved Skills and Operators, $\mathcal T_t$ the tools in $D_{E,t}$, and
$\mathcal R_{t,r}=(\mathcal R^{\mathcal S}_{t,r},
\mathcal R^{\mathcal O}_{t,r}\cup\mathcal T_t)$ their execution context. Given
history $H_{t,k}$ and fixed $\Theta_0$, the policy samples
\begin{equation}
  (a_{t,k},\nu_{t,k})
  \sim\pi_{\Theta_0}\!\left(
    \cdot\mid D_t,G_{t,k},H_{t,k},\mathcal R_{t,r}
  \right).
\label{eq:edit-policy}
\end{equation}
Here, $\pi_{\Theta_0}$ is fixed; $a_{t,k}$ adds, removes, or modifies a node,
dependency, or configuration, and $\nu_{t,k}$ records contributing Skill and
Operator identifiers. Checkpoint $\chi_{t,k}$ stores completed outputs and
runtime state; execution proceeds as
\begin{equation}
\bigl(G_{t,k+1},\chi_{t,k+1},f_{t,k}\bigr)
  =\operatorname{Execute}_{D_{E,t}}
   \bigl(G_{t,k},\chi_{t,k},a_{t,k}\bigr).
\label{eq:workflow-transition}
\end{equation}
$H_{t,k+1}$ appends the edit and execution record. The executor checks schemas
and reruns only affected descendants; $f_{t,k}$ records runtime and numerical
diagnostics, $\nu_{t,k}$ attributes reproduced repairs, and checkpointing limits
long-horizon recomputation. Scientific validity is assessed by reset replay
below.

\paragraph{Reproducible execution evidence.}
Because checkpoints may retain transient state, candidate evidence is
regenerated in reset environment $D_{E,t}$ under the replication protocol and
evaluator in $D_{V,t}$:
\begin{equation}
\begin{aligned}
(y_{t,k},\tau_{t,k})
  &=\operatorname{Replay}_{D_{E,t}}(G_{t,k}),\\
(q_{t,k},\mathbf h_{t,k},\mathbf c_{t,k})
  &=\operatorname{Eval}_{D_{V,t}}(y_{t,k},\tau_{t,k}).
\end{aligned}
\label{eq:execution-evidence}
\end{equation}
With $D_t$ and $G_{t,k}$, these outputs define $e_{t,k}$: $y_{t,k}$ is the
scientific output; $\tau_{t,k}$, the executed path and diagnostics; $q_{t,k}$,
the task score; $\mathbf h_{t,k}$, hard-constraint indicators; and
$\mathbf c_{t,k}$, resource use. $\operatorname{Pass}_{t}(e_{t,k})=1$ iff
applicable convergence, physical or statistical validity, and reproducibility
requirements hold within budget. Replay alone does not authorize persistence, which is decided by the strict-improvement rule on independent validation in Section~\ref{sec:self-evolution}.

\subsection{Execution-Driven Program Self-Evolution}
\label{sec:self-evolution}

\paragraph{Evolution trigger and repair attribution.}
Evolution instance $i$ on $D_i=D_t$ ends at replay-verified success
$e_i^+=e_{t,k_i^+}$ with $\operatorname{Pass}_i(e_i^+)=1$. Its preceding
endpoint $e_i^-=e_{t,k_i^-}$ is the latest replay-verified failure, or
$\varnothing$ if none; the endpoints define the shortest repair:
\begin{equation}
\delta_i=
\begin{cases}
  \langle a_{t,k}\rangle_{k_i^-\leq k<k_i^+},
    & e_i^-\neq\varnothing,\\
  \varnothing, & e_i^-=\varnothing.
\end{cases}
\label{eq:repair-interval}
\end{equation}
Let $G_i^+=G_{t,k_i^+}$ and $\tau_i^+=\tau_{t,k_i^+}$. Edit types and graph
topology separate strategy changes from executable capability:
\begin{equation}
\begin{aligned}
\Delta_i^{\mathcal S}
  &=\Pi_{\mathrm{ctrl}}(\delta_i),\\
\Delta_i^{\mathcal O}
  &=\operatorname{CC}_{\Gamma}\!\left(
      G_i^+,\Pi_{\mathrm{exec}}
      (G_i^+,\delta_i,\tau_i^+;\mathcal O_r)
    \right).
\end{aligned}
\label{eq:repair-attribution}
\end{equation}
$\Pi_{\mathrm{ctrl}}$ selects topology, routing, and configuration edits that
preserve executable payloads;
$\Pi_{\mathrm{exec}}$ selects generated, repaired, or unregistered executable
nodes; and $\operatorname{CC}_{\Gamma}$ groups them into maximal connected
subgraphs with compatible boundaries. Thus a reproduced repair can yield linked
Skill and Operator candidates; without $e_i^-$, only the latter is formed. No
auxiliary LLM judge is required to rank or select candidates.

\paragraph{Linked Skill--Operator abstraction.}
Let $\nu_i^{\mathcal S}$ collect the Skill identifiers in $\delta_i$. For
$\Delta_i^{\mathcal S}\neq\varnothing$, its control edits define
\begin{equation}
\Delta\mathcal S_i
  =\operatorname{Patch}_{\Theta_0}
  \bigl(\mathcal S_r[\nu_i^{\mathcal S}];
        \Delta_i^{\mathcal S},e_i^-,e_i^+\bigr).
\label{eq:skill-candidate}
\end{equation}
where $\mathcal S_r[\nu_i^{\mathcal S}]$ selects attributed entries: an empty
identifier set creates a Skill, otherwise only selected instructions change;
$\Delta_i^{\mathcal S}=\varnothing$ yields no Skill candidate. Each executable
component $U\in\Delta_i^{\mathcal O}$ has
\begin{equation}
\widehat o_{i,U}
  =\bigl(G_i^+[U],\Gamma_{\partial U},\kappa_{i,U},\rho_{i,U}\bigr).
\label{eq:operator-candidate}
\end{equation}
Here, $G_i^+[U]$ is the executable subgraph; $\partial U=(\partial^-U,
\partial^+U)$ its external input and output ports; $\Gamma_{\partial U}$ their
schemas; $\kappa_{i,U}$ the preconditions, postconditions, and applicability
conditions derived from $D_i$ and the boundary slice of $\tau_i^+$; and
$\rho_{i,U}$ the component provenance. Boundary replay restores recorded inputs
at $\partial^-U$ and checks isolated outputs at $\partial^+U$ against
$\tau_i^+$ under $D_{V,i}$, including numerical tolerances and stochastic
replication.
$\widehat o_{i,U}$ enters $\Delta\mathcal O_i$ only if
$\operatorname{BReplay}_{D_{V,i}}(\widehat o_{i,U};\tau_i^+)=1$.
Bundle $B_i=(\Delta\mathcal S_i,\Delta\mathcal O_i)$ couples the strategy change
with its executable capability.

\ifdefined\ScienceClawSectionOnly
\begin{table*}[t]
\centering
\caption{IID and cross-dataset OOD results across 23 disciplines. Discipline and metric names are abbreviated. The best score in each discipline and split is in bold and the second best is underlined.}
\label{tab:scienceclaw-23discipline-iid-ood}
\fontsize{7.6}{8.9}\selectfont
\setlength{\tabcolsep}{0pt}
\renewcommand{\arraystretch}{1.0}
\begin{tabular*}{0.99\textwidth}{@{}>{\raggedright\arraybackslash}p{88pt}@{\hspace{6pt}}*{6}{>{\centering\arraybackslash}p{26pt}@{\extracolsep{0pt}\hspace{3pt}}>{\centering\arraybackslash}p{26pt}@{\extracolsep{\fill}}}>{\centering\arraybackslash}p{26pt}@{\extracolsep{0pt}\hspace{3pt}}>{\centering\arraybackslash}p{26pt}@{}}
\toprule
 & \multicolumn{2}{c}{Frozen ($A_0$)} & \multicolumn{2}{c}{SkillOpt} & \multicolumn{2}{c}{TTE} & \multicolumn{2}{c}{AHE} & \multicolumn{2}{c}{EvoMaster} & \multicolumn{2}{c}{EvoScientist} & \multicolumn{2}{c}{ScienceClaw} \\
\cmidrule(lr){2-3}\cmidrule(lr){4-5}\cmidrule(lr){6-7}\cmidrule(lr){8-9}\cmidrule(lr){10-11}\cmidrule(lr){12-13}\cmidrule(lr){14-15}
Discipline (metric) & IID & OOD & IID & OOD & IID & OOD & IID & OOD & IID & OOD & IID & OOD & IID & OOD \\
\midrule
\multicolumn{15}{@{}l}{\textbf{Higher is better ($\uparrow$)}} \\
Agricultural sci. (PQ+) & 69.4761 & 64.6829 & 69.9837 & 66.7649 & 73.5549 & 72.0626 & 73.1602 & 72.0490 & \underline{75.1756} & 73.1862 & 74.7019 & \underline{73.2184} & \textbf{76.9612} & \textbf{75.7629} \\
Biological sci. (Spearman) & 0.5057 & 0.3745 & 0.5358 & 0.3659 & 0.5599 & 0.3975 & 0.5642 & 0.3968 & 0.5674 & \underline{0.4103} & \underline{0.5683} & 0.4089 & \textbf{0.5816} & \textbf{0.4234} \\
Biomedical sci. (DSC) & 0.7816 & 0.7088 & 0.7794 & 0.7521 & 0.8169 & 0.7784 & 0.8214 & 0.7811 & 0.8348 & \underline{0.7877} & \underline{0.8387} & 0.7798 & \textbf{0.8556} & \textbf{0.8226} \\
Chemical sci. (ROC-AUC) & 0.6518 & 0.6205 & 0.6830 & 0.6518 & 0.7232 & 0.7054 & 0.7299 & 0.7143 & \underline{0.7344} & 0.7165 & 0.7299 & \underline{0.7188} & \textbf{0.7545} & \textbf{0.7478} \\
Creative arts (SDR (dB)) & 8.9208 & 8.0993 & 9.4971 & 8.3085 & 9.8649 & 9.0812 & 9.7900 & 9.0046 & \underline{9.9453} & 9.1835 & 9.9112 & \underline{9.2653} & \textbf{10.1867} & \textbf{9.5619} \\
Education (10-mask acc.) & 0.5965 & 0.6202 & 0.6603 & 0.6626 & 0.6561 & 0.6467 & 0.6616 & 0.6567 & \underline{0.6720} & \underline{0.6731} & 0.6579 & 0.6594 & \textbf{0.6878} & \textbf{0.7087} \\
Engineering (DCASE score) & 0.5709 & 0.4663 & 0.6138 & 0.4659 & 0.6243 & 0.4943 & 0.6317 & 0.4902 & \underline{0.6380} & 0.4992 & 0.6310 & \underline{0.5017} & \textbf{0.6535} & \textbf{0.5269} \\
Health sci. (Clin. utility) & 0.4971 & 0.6834 & 0.5170 & 0.7132 & 0.5278 & 0.7380 & 0.5234 & 0.7182 & \underline{0.5383} & 0.7417 & 0.5367 & \underline{0.7428} & \textbf{0.5514} & \textbf{0.7796} \\
Indigenous (chrF++) & 15.4451 & 13.7205 & 16.5641 & 14.8441 & 16.3493 & 14.1870 & 16.7301 & 14.8969 & \underline{16.9074} & \underline{15.0442} & 16.6453 & 14.7130 & \textbf{17.4351} & \textbf{15.9778} \\
Computing sci. (pass@1) & 0.7656 & 0.3750 & 0.7969 & 0.4063 & 0.8438 & 0.4219 & 0.8438 & 0.4219 & \underline{0.8594} & \underline{0.4375} & 0.8438 & \underline{0.4375} & \textbf{0.8750} & \textbf{0.4531} \\
Language \& culture (LAS) & 0.7288 & 0.5599 & 0.7623 & 0.5713 & 0.7786 & 0.6042 & 0.7761 & 0.6004 & \underline{0.7845} & \underline{0.6267} & 0.7759 & 0.5879 & \textbf{0.8138} & \textbf{0.6497} \\
Law (mAP) & 0.7513 & 0.6918 & 0.8158 & 0.7652 & 0.8074 & 0.7706 & 0.8211 & 0.7720 & \underline{0.8292} & \underline{0.7863} & 0.8170 & 0.7642 & \textbf{0.8554} & \textbf{0.8289} \\
Mathematics (Oracle acc.) & 0.8438 & 0.7344 & 0.9063 & 0.7969 & 0.9063 & 0.7656 & 0.9063 & 0.8125 & \underline{0.9219} & \underline{0.8438} & 0.9063 & 0.8125 & \textbf{0.9531} & \textbf{0.8750} \\
Philosophy (F1) & 0.4178 & 0.4166 & 0.4599 & 0.4428 & 0.4699 & 0.4846 & 0.4636 & 0.4868 & 0.4717 & \underline{0.4943} & \underline{0.4737} & 0.4925 & \textbf{0.4866} & \textbf{0.5138} \\
Psychology (Micro acc.) & 0.6088 & 0.5615 & 0.6364 & 0.6064 & 0.6269 & 0.6084 & 0.6355 & 0.6149 & 0.6467 & \underline{0.6343} & \underline{0.6476} & 0.6284 & \textbf{0.6618} & \textbf{0.6586} \\
\midrule
\multicolumn{15}{@{}l}{\textbf{Lower is better ($\downarrow$)}} \\
Built env. (NRMSE \%) & 51.0997 & 53.5009 & 47.7143 & 49.6083 & 46.7210 & 49.5570 & 46.7362 & 49.0341 & \underline{46.2726} & \underline{46.7640} & 47.2762 & 49.4801 & \textbf{45.0065} & \textbf{45.0970} \\
Commerce (MASE) & 1.6009 & 1.8740 & 1.5268 & 1.6867 & 1.5671 & 1.7502 & 1.5227 & 1.7066 & \underline{1.4984} & 1.6930 & 1.5281 & \underline{1.6846} & \textbf{1.4653} & \textbf{1.5986} \\
Earth sci. (RMSE (K)) & 1.0562 & 1.1364 & 1.0476 & 1.0379 & 1.0157 & \underline{1.0005} & 1.0094 & 1.0114 & 1.0010 & 1.0042 & \underline{1.0008} & 1.0079 & \textbf{0.9721} & \textbf{0.9580} \\
Economics (sMAPE (\%)) & 12.7394 & 18.1343 & 11.8721 & 16.8216 & 12.1577 & 16.9410 & 11.7390 & 16.2028 & \underline{11.6350} & \underline{16.1758} & 11.7621 & 16.7195 & \textbf{11.3320} & \textbf{15.5025} \\
Environmental sci. (CRPS) & 0.8104 & 0.6504 & 0.7786 & 0.6456 & 0.7517 & 0.6325 & 0.7565 & 0.6118 & \underline{0.7304} & \underline{0.5944} & 0.7326 & 0.5979 & \textbf{0.7164} & \textbf{0.5751} \\
History (cMER-micro) & 0.0185 & 0.0296 & 0.0173 & 0.0271 & 0.0177 & 0.0268 & \underline{0.0172} & 0.0264 & 0.0172 & \underline{0.0261} & 0.0172 & 0.0268 & \textbf{0.0166} & \textbf{0.0251} \\
Human society (nRMSE) & 0.0113 & 0.0213 & 0.0107 & 0.0201 & 0.0104 & 0.0192 & 0.0103 & 0.0191 & \underline{0.0103} & 0.0192 & 0.0104 & \underline{0.0191} & \textbf{0.0100} & \textbf{0.0181} \\
Physical sci. (MAE) & 36.6271 & 44.6757 & 35.1699 & 41.9077 & 33.4027 & 39.0907 & 33.5864 & 39.7923 & 33.0932 & \underline{38.2995} & \underline{32.8654} & 39.0137 & \textbf{32.1581} & \textbf{37.1306} \\
\bottomrule
\end{tabular*}
\end{table*}

\fi

\paragraph{Program-level validation and update.}
Atomic application
$\widetilde A_i=\operatorname{Apply}(A_r;B_i)$ assigns version identifiers
$\omega_i$ and requires
\begin{equation}
\begin{aligned}
(H_i^{\mathrm{src}},e_i^{\mathrm{src}},\tau_i^{\mathrm{src}})
  &=\operatorname{Run}_{D_i}(\widetilde A_i),\\
R_{\mathrm{src}}(\widetilde A_i)
  &=\operatorname{Pass}_{i}(e_i^{\mathrm{src}})
    \land \operatorname{Use}_{i}
    (\omega_i;H_i^{\mathrm{src}},\tau_i^{\mathrm{src}}).
\end{aligned}
\label{eq:candidate-validation}
\end{equation}
Here, $\operatorname{Run}_{D_i}$ returns reset-run history
$H_i^{\mathrm{src}}$, evidence $e_i^{\mathrm{src}}$, and trace
$\tau_i^{\mathrm{src}}$ as in Eq.~(\ref{eq:execution-evidence}), while
$\operatorname{Use}_{i}$ checks every $\omega_i$. Only
$R_{\mathrm{src}}=1$ candidates enter $\mathcal P(\mathcal E_r)$ and
$\mathcal D_{\mathrm{val}}$;
\ifdefined\ScienceClawSectionOnly
Equation~(\ref{eq:program-update})
\else
the task-level strict-improvement rule
\fi
then retains strict budget-feasible improvements with $\mathcal E_r$ and
foundation-model parameters fixed at $\Theta_0$ throughout every round (Appendix~\ref{app:theory}).

\ifdefined\ScienceClawSectionOnly\else
\bibliographystyle{ACM-Reference-Format}
\bibliography{../references}
\fi
\ScienceClawEnd

\section{Experimental Evaluation}
\label{sec:experiments}

\textbf{Setup.} All systems process the common 23-discipline source streams.
Evolving methods update their programs, whereas the frozen control remains at
$A_0$. We freeze the snapshot after each of seven rounds, clear transient
state, and evaluate it on the held-out same-discipline target
datasets in $\mathcal D_{\mathrm{OOD}}$; OOD results never generate or select
updates. Qwen3.5-27B orchestrates every task for all systems, GPT-5.4 performs the
evolution of ScienceClaw and every baseline, and each discipline
contributes 64 IID and 64 OOD instances. We compare ScienceClaw with
an evolution-disabled frozen agent and methods that evolve Skills, scientific
tools, agent harnesses, or persistent scientific-agent programs and memory:
SkillOpt~\cite{yang2026skillopt}, test-time tool evolution
(TTE)~\cite{lu2026beyond}, Agentic Harness Engineering
(AHE)~\cite{lin2026agentic}, EvoMaster~\cite{zhu2026evomaster}, and
EvoScientist~\cite{lyu2026evoscientist}. Each baseline runs its best released
configuration through fixed ScienceClaw-Eval adapters under the same models,
tools, source stream, validation data, and update budget (Appendix~\ref{app:impl-baselines}); adapters and initial
assets are fixed without access to ID/OOD instances or evaluators. Two
ScienceClaw variants complete the comparison: RuleEvo evolves by deterministic
trace projection alone, and Qwen-only evolution replaces GPT-5.4 with
Qwen. Three PhD reviewers per discipline and a natural- or social-science
expert audited each evaluator and its reference executions.
Table~\ref{tab:scienceclaw-23discipline-iid-ood} reports task-native scores,
compared only within each discipline because metrics differ in units and
direction; round-wise analyses use MacroSR on
$\mathcal D_{\mathrm{OOD}}$ (Eq.~(\ref{eq:macro-success-rate})) as the OOD
macro score.


\setlength{\dblfloatsep}{7pt plus 2pt minus 2pt}
\newlength{\fulltw}\setlength{\fulltw}{\textwidth}
\renewcommand{\topfraction}{0.95}\renewcommand{\dbltopfraction}{0.95}\setcounter{topnumber}{4}\setcounter{dbltopnumber}{3}


\begin{table*}[t]
\centering
\caption{Mean of IID and cross-dataset OOD results of the ablation variants across 23 disciplines. Discipline and metric names are abbreviated. The best score in each discipline is in bold and the second best is underlined.}
\label{tab:scienceclaw-ablation-iid}
\fontsize{7.6}{8.9}\selectfont
\setlength{\tabcolsep}{0pt}
\renewcommand{\arraystretch}{1.0}
\begin{tabular*}{0.99\textwidth}{@{}>{\raggedright\arraybackslash}p{88pt}@{\hspace{6pt}}*{13}{>{\centering\arraybackslash}p{26pt}@{\extracolsep{\fill}}}>{\centering\arraybackslash}p{26pt}@{}}
\toprule
Discipline (metric) & \multicolumn{1}{c}{\makebox[3.3em][c]{\fontsize{6.1}{6.9}\selectfont\shortstack[c]{Frozen}}} & \multicolumn{1}{c}{\makebox[3.3em][c]{\fontsize{6.1}{6.9}\selectfont\shortstack[c]{RuleEvo}}} & \multicolumn{1}{c}{\makebox[3.3em][c]{\fontsize{6.1}{6.9}\selectfont\shortstack[c]{Qwen-only}}} & \multicolumn{1}{c}{\makebox[3.3em][c]{\fontsize{6.1}{6.9}\selectfont\shortstack[c]{w/o Skill\\evol.}}} & \multicolumn{1}{c}{\makebox[3.3em][c]{\fontsize{6.1}{6.9}\selectfont\shortstack[c]{w/o Oper.\\evol.}}} & \multicolumn{1}{c}{\makebox[3.3em][c]{\fontsize{6.1}{6.9}\selectfont\shortstack[c]{w/o Work.\\evol.}}} & \multicolumn{1}{c}{\makebox[3.3em][c]{\fontsize{6.1}{6.9}\selectfont\shortstack[c]{w/o link.\\evol.}}} & \multicolumn{1}{c}{\makebox[3.3em][c]{\fontsize{6.1}{6.9}\selectfont\shortstack[c]{w/o src.\\replay}}} & \multicolumn{1}{c}{\makebox[3.3em][c]{\fontsize{6.1}{6.9}\selectfont\shortstack[c]{w/o sci.\\constr.}}} & \multicolumn{1}{c}{\makebox[3.3em][c]{\fontsize{6.1}{6.9}\selectfont\shortstack[c]{w/o ind.\\valid.}}} & \multicolumn{1}{c}{\makebox[3.3em][c]{\fontsize{6.1}{6.9}\selectfont\shortstack[c]{w/o cross-\\task mem.}}} & \multicolumn{1}{c}{\makebox[3.3em][c]{\fontsize{6.1}{6.9}\selectfont\shortstack[c]{w/o IID\\select.}}} & \multicolumn{1}{c}{\makebox[3.3em][c]{\fontsize{6.1}{6.9}\selectfont\shortstack[c]{w/o OOD\\transfer}}} & \multicolumn{1}{c}{\makebox[3.3em][c]{\fontsize{6.1}{6.9}\selectfont\shortstack[c]{ScienceClaw\\(full)}}} \\
\midrule
\multicolumn{15}{@{}l}{\textbf{Higher is better ($\uparrow$)}} \\
Agricultural sci. (PQ+) & 67.0795 & 71.5023 & 72.0207 & \underline{74.1946} & 71.3347 & 72.0963 & 72.5891 & 70.7277 & 69.4636 & 71.4427 & 73.5959 & 68.5843 & 72.1931 & \textbf{76.3621} \\
Biological sci. (Spearman) & 0.4401 & 0.4617 & 0.4692 & \underline{0.4897} & 0.4728 & 0.4871 & 0.4793 & 0.4841 & 0.4799 & 0.4819 & 0.4707 & 0.4482 & 0.4622 & \textbf{0.5025} \\
Biomedical sci. (DSC) & 0.7452 & 0.7724 & 0.7953 & \underline{0.8154} & 0.7929 & 0.8005 & 0.8016 & 0.7841 & 0.7726 & 0.7770 & 0.8054 & 0.7488 & 0.8072 & \textbf{0.8391} \\
Chemical sci. (ROC-AUC) & 0.6362 & 0.6875 & 0.7087 & 0.7232 & 0.7076 & \underline{0.7254} & 0.7176 & 0.6964 & 0.6908 & 0.7121 & 0.7232 & 0.6920 & 0.7098 & \textbf{0.7511} \\
Creative arts (SDR (dB)) & 8.5101 & 9.2885 & 9.1862 & \underline{9.5970} & 9.2316 & 9.4530 & 9.3943 & 9.4257 & 9.3184 & 9.5349 & 9.5109 & 9.0730 & 9.4250 & \textbf{9.8743} \\
Education (10-mask acc.) & 0.6083 & 0.6382 & 0.6579 & 0.6654 & 0.6761 & \underline{0.6766} & 0.6722 & 0.6728 & 0.6627 & 0.6690 & 0.6652 & 0.6254 & 0.6331 & \textbf{0.6982} \\
Engineering (DCASE) & 0.5186 & 0.4976 & 0.5114 & \underline{0.5703} & 0.5398 & 0.5568 & 0.5356 & 0.5368 & 0.5240 & 0.5371 & 0.5671 & 0.4736 & 0.5516 & \textbf{0.5902} \\
Health sci. (Clin. utility) & 0.5903 & 0.6240 & 0.6186 & \underline{0.6460} & 0.6163 & 0.6246 & 0.6380 & 0.6299 & 0.6034 & 0.6281 & 0.6428 & 0.6041 & 0.6358 & \textbf{0.6655} \\
Indigenous (chrF++) & 14.5828 & 15.4084 & 15.9032 & 15.5616 & 16.0139 & 16.0826 & 15.9824 & \underline{16.1154} & 15.8247 & 16.1141 & 15.7938 & 15.4115 & 15.6417 & \textbf{16.7064} \\
Computing sci. (pass@1) & 0.5703 & 0.5859 & 0.6016 & \underline{0.6484} & 0.6172 & 0.6328 & 0.6250 & 0.6172 & 0.6094 & 0.6250 & 0.6406 & 0.5625 & 0.6406 & \textbf{0.6641} \\
Language \& culture (LAS) & 0.6443 & 0.6441 & 0.6687 & \underline{0.7097} & 0.6830 & 0.6937 & 0.6852 & 0.7009 & 0.6905 & 0.7028 & 0.6956 & 0.6229 & 0.6632 & \textbf{0.7317} \\
Law (mAP) & 0.7215 & 0.7791 & 0.7863 & 0.7881 & \underline{0.8102} & 0.7967 & 0.8068 & 0.7959 & 0.7769 & 0.7813 & 0.7971 & 0.7743 & 0.7898 & \textbf{0.8422} \\
Mathematics (Oracle acc.) & 0.7891 & 0.8594 & 0.8516 & 0.8594 & 0.8750 & \underline{0.8906} & 0.8750 & 0.8594 & 0.8438 & 0.8438 & 0.8828 & 0.8359 & 0.8828 & \textbf{0.9141} \\
Philosophy (F1) & 0.4172 & 0.4581 & 0.4715 & 0.4672 & \underline{0.4786} & 0.4773 & 0.4781 & 0.4555 & 0.4553 & 0.4680 & 0.4666 & 0.4383 & 0.4454 & \textbf{0.5002} \\
Psychology (Micro acc.) & 0.5851 & 0.6177 & 0.6144 & 0.6194 & 0.6346 & \underline{0.6402} & 0.6301 & 0.6138 & 0.6092 & 0.6242 & 0.6298 & 0.6122 & 0.6209 & \textbf{0.6602} \\
\midrule
\multicolumn{15}{@{}l}{\textbf{Lower is better ($\downarrow$)}} \\
Built env. (NRMSE \%) & 52.3003 & 48.2400 & 48.0417 & 46.8999 & 47.4011 & \underline{46.3133} & 46.8281 & 47.1144 & 47.7366 & 46.9883 & 46.8659 & 48.7799 & 47.1814 & \textbf{45.0517} \\
Commerce (MASE) & 1.7375 & 1.6348 & 1.6234 & 1.6220 & \underline{1.5815} & 1.5819 & 1.6117 & 1.5874 & 1.6189 & 1.5968 & 1.6018 & 1.6564 & 1.6396 & \textbf{1.5320} \\
Earth sci. (RMSE (K)) & 1.0963 & 1.0506 & 1.0268 & 0.9952 & 1.0220 & \underline{0.9925} & 0.9982 & 1.0033 & 1.0035 & 0.9987 & 0.9978 & 1.0409 & 1.0221 & \textbf{0.9650} \\
Economics (sMAPE (\%)) & 15.4369 & 15.1645 & 14.9715 & 14.4654 & \underline{14.0826} & 14.1269 & 14.9184 & 14.1704 & 14.2476 & 14.1622 & 14.5381 & 16.2718 & 15.4171 & \textbf{13.4172} \\
Environmental sci. (CRPS) & 0.7304 & 0.6994 & 0.6845 & 0.6659 & 0.6889 & 0.6723 & 0.6707 & 0.6756 & 0.6752 & \underline{0.6656} & 0.6837 & 0.7144 & 0.6870 & \textbf{0.6458} \\
History (cMER-micro) & 0.0241 & 0.0243 & 0.0230 & 0.0227 & \underline{0.0217} & 0.0220 & 0.0227 & 0.0230 & 0.0234 & 0.0228 & 0.0220 & 0.0242 & 0.0238 & \textbf{0.0209} \\
Human society (nRMSE) & 0.0163 & 0.0158 & 0.0162 & \underline{0.0145} & 0.0154 & 0.0151 & 0.0156 & 0.0158 & 0.0154 & 0.0154 & 0.0149 & 0.0166 & 0.0153 & \textbf{0.0141} \\
Physical sci. (MAE) & 40.6514 & 39.0216 & 38.3645 & 35.8828 & 37.1914 & 36.4845 & 36.7402 & 37.4532 & 37.3627 & 36.6232 & \underline{35.8646} & 40.0907 & 37.0317 & \textbf{34.6443} \\
\bottomrule
\end{tabular*}
\par\vspace{6pt}
\centering
\includegraphics[width=\textwidth,trim=0 0 0 0,clip]{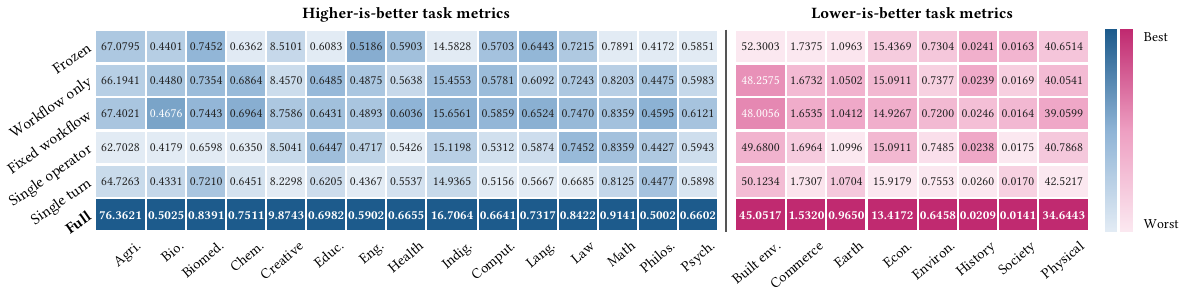}
  \setlength{\abovecaptionskip}{4pt}
  \captionof{figure}{Mean of IID and OOD task-native scores for the single-component and
  execution-structure ablation variants across 23 disciplines (variants not
  listed in Table~\ref{tab:scienceclaw-ablation-iid}, with the frozen agent and the full system for reference). Colours are normalised
  within each discipline column (darker is better), and the two groups follow higher-is-better and lower-is-better metrics.}
  \Description{Heatmap with the frozen agent, four single-component and
  execution-structure ablation variants, and the full system as rows and 23
  disciplines as columns, split into higher-is-better and lower-is-better
  metrics, showing the mean of IID and OOD task-native scores.}
  \label{fig:ablation-mean-map}
\end{table*}




\begin{figure*}[t]
  \centering
  \begin{minipage}[t]{0.685\fulltw}
    \vspace{0pt}
\begin{minipage}[t]{0.685\fulltw}
\vspace{0pt}
\centerline{\fontsize{7.8}{9.4}\selectfont\textbf{(a)}\ Continual self-evolution}\vspace{3pt}
\centering
\fontsize{7.3}{8.8}\selectfont
\setlength{\tabcolsep}{0pt}
\renewcommand{\arraystretch}{1.20}
\begin{tabular*}{\linewidth}{@{}c*{8}{@{\hspace{2pt}\extracolsep{\fill}}c}@{\hspace{6pt}\extracolsep{\fill}}c@{\hspace{2pt}\extracolsep{\fill}}c@{\hspace{2pt}\extracolsep{\fill}}c@{\hspace{2pt}\extracolsep{\fill}}c@{}}
\toprule
 & \multicolumn{8}{c}{OOD macro score (\%) after round} & \multicolumn{4}{c}{Round-7 summary} \\
\cmidrule(lr){2-9}\cmidrule(lr){10-13}
Method & 0 & 1 & 2 & 3 & 4 & 5 & 6 & 7 & Gain (pp) & Prom. & Rej. & Rate (\%) \\
\midrule
Frozen & 77.72 & 77.72 & 77.72 & 77.72 & 77.72 & 77.72 & 77.72 & 77.72 & 0.00 & 0 & 0 & -- \\
RuleEvo & 77.72 & 78.40 & 79.21 & 78.87 & 79.62 & 80.37 & 80.03 & 81.32 & 3.60 & 37 & 124 & 22.98 \\
Qwen-only & 77.72 & 79.42 & 80.16 & 79.89 & 81.25 & 81.93 & 81.59 & 83.22 & 5.50 & 49 & 112 & 30.43 \\
EvoScientist & 77.72 & 79.69 & 81.18 & 82.34 & 82.07 & 84.04 & 85.19 & 86.35 & 8.63 & 56 & 105 & 34.78 \\
EvoMaster & 77.72 & \underline{80.03} & \underline{81.59} & \textbf{82.88} & \underline{82.54} & \underline{84.85} & \underline{86.07} & \underline{87.50} & \underline{9.78} & \underline{63} & 98 & \underline{39.13} \\
\midrule
ScienceClaw & 77.72 & \textbf{80.43} & \textbf{82.68} & \underline{82.47} & \textbf{86.14} & \textbf{87.57} & \textbf{89.81} & \textbf{91.30} & \textbf{13.59} & \textbf{70} & 91 & \textbf{43.48} \\
\bottomrule
\end{tabular*}
\end{minipage}
\par\medskip
    \noindent\makebox[0.685\fulltw][s]{
\begin{minipage}[t]{0.332\fulltw}
\vspace{0pt}
\centerline{\fontsize{7.8}{9.4}\selectfont\textbf{(b)}\ Cross-family forward transfer (pp)}\vspace{3pt}
\centering
\fontsize{7.1}{8.5}\selectfont
\setlength{\tabcolsep}{0pt}
\renewcommand{\arraystretch}{1.3}
\begin{tabular*}{\linewidth}{@{}l@{\extracolsep{\fill}}ccccc@{}}
\toprule
\shortstack[l]{Source\\family} & \shortstack{Life/\\health} & \shortstack{Phys./\\Earth} & \shortstack{Eng./\\comp.} & \shortstack{Social/\\behav.} & \shortstack{Hum./\\law} \\
\midrule
Life/health & \textbf{13.08} & 2.90 & 1.77 & 0.62 & 0.33 \\
Phys./Earth & 2.20 & \textbf{15.32} & 3.47 & 0.36 & \makebox[0pt][r]{$-$}0.33 \\
Eng./comp. & 1.92 & 3.54 & \textbf{15.04} & 1.84 & 1.64 \\
Social/behav. & 1.10 & 0.32 & 0.95 & \textbf{14.37} & 4.27 \\
Hum./law & 0.27 & \makebox[0pt][r]{$-$}0.51 & 1.26 & 3.61 & \textbf{15.66} \\
\bottomrule
\end{tabular*}
\end{minipage}
\hfill
\begin{minipage}[t]{0.332\fulltw}
\vspace{0pt}
\centerline{\fontsize{7.8}{9.4}\selectfont\textbf{(c)}\ Transfer, forgetting and reliability}\vspace{3pt}
\centering
\fontsize{7.1}{8.5}\selectfont
\setlength{\tabcolsep}{0pt}
\renewcommand{\arraystretch}{1.3}
\begin{tabular*}{\linewidth}{@{}l@{\extracolsep{\fill}}ccccc@{}}
\toprule
Method & \shortstack{FWT/BWT\\$\uparrow$} & \shortstack{Neg.\\$\downarrow$} & \shortstack{Forg.\ $\downarrow$\\avg/max} & \shortstack{Attain.\\$\uparrow$} & \shortstack{Err.\\$\downarrow$} \\
\midrule
SkillOpt & 1.42/0.26 & 31.39 & 0.88/2.45 & 92.67 & 6.52 \\
TTE & 1.75/0.41 & 26.97 & 0.76/2.17 & 93.63 & 7.25 \\
EvoScientist & 2.74/1.23 & 17.32 & 0.42/1.23 & 94.43 & 4.58 \\
EvoMaster & \underline{3.19/1.53} & \underline{14.33} & \underline{0.33/1.05} & \underline{94.99} & \underline{3.80} \\
ScienceClaw & \textbf{4.20/3.66} & \textbf{11.89} & \textbf{0.13/0.51} & \textbf{96.42} & \textbf{1.23} \\
\bottomrule
\end{tabular*}
\end{minipage}
}
  \end{minipage}\hfill
  \begin{minipage}[t]{0.300\fulltw}
    \vspace{0pt}
    \centering
    \includegraphics[width=\linewidth]{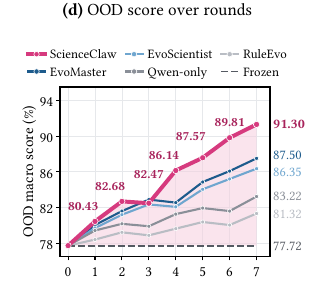}\par
    \includegraphics[width=\linewidth]{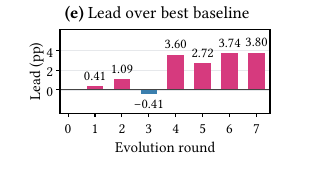}
  \end{minipage}
  \par\medskip
  \includegraphics[width=\fulltw,trim=0 16pt 0 0,clip]{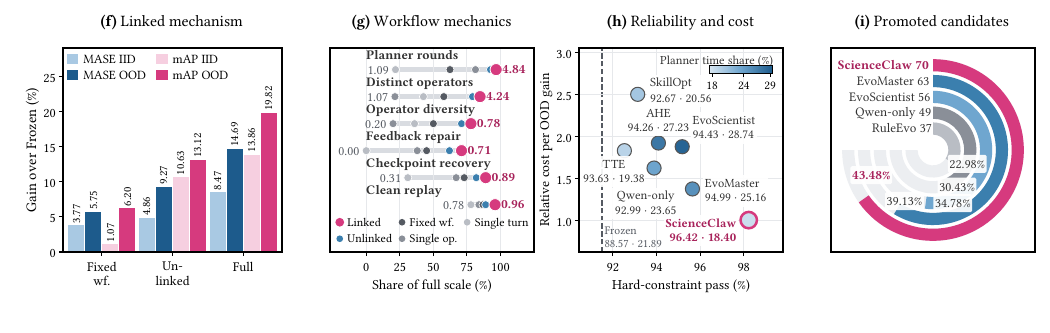}
  \setlength{\abovecaptionskip}{4pt}
  \caption{Continual evolution, transfer, and mechanisms. (a) OOD macro score (MacroSR on $\mathcal D_{\mathrm{OOD}}$, \%) of each snapshot after round $r$, with gain over $A_0$ in pp and promoted (Prom.) and rejected (Rej.) candidates out of 161. (b) Forward transfer (pp): OOD gain on the target family after evolving on the source family. (c) Forward and backward transfer (FWT/BWT, pp), negative-transfer share (Neg., \%), average and maximum forgetting on $\mathcal D_{\mathrm{rep}}$ (Forg., pp), continual-target attainment (Attain., \%), and erroneous promotions (Err., \%). (d) OOD macro score by round. (e) Lead of ScienceClaw over the best other system (pp). (f) Gain over Frozen (\%) of the linked-mechanism variants on Commerce (MASE) and Law (mAP). (g) Workflow mechanics: planner rounds, distinct Operators, Operator diversity, feedback repair, checkpoint recovery, and clean replay. (h) Hard-constraint pass rate against cost per OOD gain (ScienceClaw = 1), coloured by planner wall-time share, with Attain. and planner share annotated. (i) Promoted candidates and promotion rate. Bold and underline mark the best and second best in (a) and (c); in (b), bold marks the within-family gain.}
  \label{fig:eval-module}
  \Description{Panels (a) to (c) are tables of continual self-evolution, cross-family forward transfer and transfer, forgetting and reliability; panels (d) and (e) plot the OOD score over rounds and the lead over the best baseline; panels (f) to (i) show the linked mechanism, workflow mechanics, reliability and cost, and promoted candidates.}
\end{figure*}

\textbf{RQ1: Does persistent program evolution improve cross-dataset
performance?} Yes. ScienceClaw is best in all 23 disciplines on both splits
(Table~\ref{tab:scienceclaw-23discipline-iid-ood}), and its mean relative
margin over the strongest baseline in each discipline grows from 2.58\% on IID
to 4.19\% on OOD (12.25\% and 16.45\% over the frozen agent), and it is positive in every discipline, from 1.82\% to 3.73\% on IID and from 3.05\% to 6.21\% on OOD (Appendix~\ref{app:data-margin}). The largest OOD margins appear in Indigenous studies (6.21\%) and Law (5.42\%), the smallest in Physical sciences (3.05\%). The wider OOD
margin is consistent with the gate: a candidate persists only after it
reproduces on its source task and improves independent validation tasks, which
filters out updates fitted to one source, whereas single-asset methods can fall
below the frozen agent (SkillOpt in three of 46 cells). The baselines order by how much of
the program they evolve: their mean OOD gain over the frozen agent rises from
5.74\% for SkillOpt (Skills) through 8.82\% for TTE (tools) and 9.73\% for AHE
(harness) to 10.55\% for EvoScientist and 11.95\% for EvoMaster (programs and
memory), against 16.45\% for ScienceClaw, above all of them. ScienceClaw also beats the frozen agent in every discipline, by 11.57\% to 23.34\% on OOD, and its margin is similar on higher-is-better (4.26\%) and lower-is-better metrics (4.06\%). Its OOD score falls 11.73\% below its IID score on average, against 16.06\% for the frozen agent and 13.46\% to 15.45\% for the baselines, so the evolved program degrades least of all systems when the dataset changes.

\textbf{RQ2: Does evolution keep improving OOD performance over rounds?} Yes.
ScienceClaw's OOD macro score rises from 77.72 at $A_0$ to 91.30 after round 7
(+13.59 pp), against +9.78 pp for EvoMaster, the strongest baseline, and its
lead over the best other system grows from 0.41 pp after round 1 to 3.80 pp
(Figure~\ref{fig:eval-module}a,d,e). Round 3 is the only round in which another
system leads, by 0.41 pp. From round 3 to round 7 ScienceClaw adds 8.83 pp, nearly twice the
4.6 pp of EvoMaster. The curve keeps rising because ScienceClaw promotes
more candidates than any other system, 70 of the 161 submitted (43.48\%
against 22.98\% to 39.13\%; Figure~\ref{fig:eval-module}i), so the gate is
selective without stalling. RuleEvo and Qwen-only, which change only
how candidates are produced, reach +3.60 and +5.50 pp, so candidate quality, not the gate
alone, drives the gain. None of the five evolving systems dips by more than 0.34 pp in any round.

\textbf{RQ3: Does evolution transfer across disciplines and retain earlier
repairs?} Yes. Evolving on one discipline family lifts the OOD score of the
others: 18 of the 20 cross-family pairs are positive, with a mean gain of
1.58 pp against 14.69 pp within a family (Figure~\ref{fig:eval-module}b), so most of the evolved capability is discipline specific and the part that transfers rarely harms.
Transfer is strongest between related families: Social/behavioural and
Humanities/law (+4.27, +3.61), and Engineering/computing and Physical/Earth
(+3.54, +3.47). The two negative pairs, both between Physical/Earth and
Humanities/law, stay below 0.6 pp in magnitude. Retention shows the same
pattern (Figure~\ref{fig:eval-module}c; Appendix~\ref{app:data-continual}): ScienceClaw reaches 4.20 pp forward (all 25 pairs) and
3.66 pp backward transfer with 0.13 pp average and 0.51 pp maximum forgetting
and 11.89\% negative transfer, against 0.33 pp, 1.05 pp, and 14.33\% for
EvoMaster. Source replay protects earlier repairs, and independent validation guards
the other families against regressions.

\textbf{RQ4: Which mechanisms produce the gain?} The linked evolution and the
replay gate. The full system is best in all 23 disciplines
(Table~\ref{tab:scienceclaw-ablation-iid}; Appendices~\ref{app:impl-ablation} and~\ref{app:data-ablation}), and every ablation keeps only part
of its gain over the frozen agent, with gain measured as the average relative
drop below the full system (13.68\% for Frozen; Appendix~\ref{app:clarify-gate}). Updating Skills and Operators
without committing them as one bundle keeps 60\% of the gain, a loss comparable
to removing Operator evolution (59\%); on Commerce and Law the linked system gains
14.69\% and 19.82\% on OOD over Frozen, against 9.27\% and 13.12\% unlinked
(Figure~\ref{fig:eval-module}f). The gate matters even more: without IID
selection only 15\% of the gain remains, and without scientific constraints,
source replay, or the independent validator, 46\%, 55\%, and 59\%, with losses
2.4 to 2.8 times larger on OOD than on IID. RuleEvo and Qwen-only keep only about a third and 44\%. In 20 of 23 disciplines the runner-up
to the full system removes a single evolution component, so no component carries the
gain alone. Execution
structure is the base on which evolution acts: single-turn and single-operator
execution fall below the frozen agent (Figure~\ref{fig:ablation-mean-map}), and
Linked Full runs 4.84 planner rounds and 4.24 distinct Operators per task, repairs 71\% of failures from feedback, recovers 89\% after
interruption, and replays 96\% cleanly, against 0, 31\%, and 78\% for a single-turn
agent (Figure~\ref{fig:eval-module}g; case studies of Commerce and Law in Appendix~\ref{app:case}).

\textbf{RQ5: Is the gain reliable and efficient?} Yes. ScienceClaw passes the
hard constraints on 98.23\% of instances, 2.58 points above EvoMaster, makes the
fewest erroneous promotions in Figure~\ref{fig:eval-module}c (1.23\% against
3.80\% to 7.25\%), and costs less per OOD point than every other method (1.37 to 2.50 times as much, 27.1\% less than EvoMaster),
while its planner takes the smallest share of wall-clock time (18.40\%;
Figure~\ref{fig:eval-module}h; Appendix~\ref{app:data-continual}). Every compared evolving method passes more constraints than the frozen agent (91.51\%), and ScienceClaw adds the most (+6.73 points). Its target attainment of 96.42\% is 1.43 points above EvoMaster and 7.85 above the frozen agent (88.57\%). The replay and validation that block invalid updates therefore do not trade against cost. Appendix~\ref{app:clarify-cost} gives the cost accounting behind these ratios, and Appendix~\ref{app:data-continual} defines the reliability metrics.

\section{Conclusion}
\label{sec:conclusion}

Across 23 disciplines, ScienceClaw's replay-gated linked Skill--Operator
evolution yields sustained OOD gains, positive cross-family transfer, and
stronger retention at a lower cost per gain, without updating any model parameters.
These results support evaluating scientific-agent evolution by persistence, transfer, validity,
and reproducibility, rather than isolated task scores. ScienceClaw-Eval measures scientific correctness, evolutionary gain,
retention, cross-dataset transfer, and evolution cost through sequential streams and
independent reset evaluation.

\section{Limitations and Ethical Considerations}
\label{sec:limitations}

Automated evaluators may miss domain errors (Appendix~\ref{app:audit}), and updates may transfer
incorrectly. ScienceClaw should support, not replace, experts; high-stakes use
requires provenance, licensing, privacy safeguards, and independent review.
The 23-discipline scope and cross-dataset splits do not eliminate shared
conventions or within-field variation. Reported trajectories are final
snapshots, so broader claims need repeated runs and discipline-level
intervals. Persistent executable updates can reuse errors and expand the
attack surface, requiring versioned program snapshots, replay and rollback, sandboxing,
least-privilege credentials, auditable provenance, and human approval for
irreversible actions.

\section{Generative AI Usage}
\label{sec:genai-usage}

Generative AI assisted language and figures; the authors verified and take
responsibility for the paper.

\bibliographystyle{ACM-Reference-Format}
\bibliography{references}

\clearpage
\appendix
\theoremstyle{acmdefinition}
\newtheorem{remark}[theorem]{Remark}
\newtheorem{assumption}[theorem]{Assumption}
\definecolor{tbdred}{RGB}{190,30,30}
\newcommand{\tbd}[1]{{\color{tbdred}\textbf{[TBD: #1]}}}
\newcolumntype{L}[1]{>{\raggedright\arraybackslash}p{#1}}
\theoremstyle{acmplain}
\numberwithin{equation}{section}
\renewcommand{\dbltopfraction}{0.99}\renewcommand{\dblfloatpagefraction}{0.75}\renewcommand{\topfraction}{0.95}\renewcommand{\floatpagefraction}{0.75}\renewcommand{\textfraction}{0.05}\renewcommand{\bottomfraction}{0.9}\setcounter{bottomnumber}{3}\setcounter{dbltopnumber}{3}
\setlength{\textfloatsep}{10pt plus 2pt minus 4pt}\setlength{\dbltextfloatsep}{10pt plus 2pt minus 4pt}
\section{Formal Properties of the Replay-Gated Update}
\label{app:theory}

This appendix states what the update rule of Section~\ref{sec:task} guarantees.
The results concern the
rule as implemented (Appendix~\ref{app:impl}) and mark where the
implementation is stricter than the literal Eq.~(\ref{eq:program-update}).

\subsection{Setting and Conventions}
\label{app:theory-setting}

\paragraph{Validation outcomes.}
For a program $A$ and a validation episode $j\in\mathcal D_{\mathrm{val}}$, let
$u_j(A)=(z_j(A),s_j(A),\mathbf h_j(A),\mathbf c_j(A))$ collect the success
indicator, a bounded normalized task score, the hard-constraint indicators, and
the resource use of the reset solve of $j$ under $A$. Write
$M(A)=\operatorname{MacroSR}(A)$ on $\mathcal D_{\mathrm{val}}$
(Eq.~(\ref{eq:macro-success-rate})) and $S(A)$ for the mean normalized score.
The implementation compares programs lexicographically,
\begin{equation}
\begin{aligned}
A\succ A'
\;\Longleftrightarrow\;
&M(A)>M(A')\ \ \text{or}\\
&\bigl(M(A)=M(A')\ \land\ S(A)-S(A')\ge\varepsilon\bigr),
\end{aligned}
\label{eq:order}
\end{equation}
with a score margin $\varepsilon>0$ that exceeds the re-solve noise of a single
episode. The literal rule $Q_{\mathrm{val}}=M$ is the limit $\varepsilon=\infty$,
where ties are never broken.

\begin{lemma}[$\succ$ is a strict order]
\label{lem:order}
The relation $\succ$ is irreflexive and transitive.
\end{lemma}
\begin{proof}
Irreflexivity: $M(A)>M(A)$ is false and $S(A)-S(A)=0<\varepsilon$. Transitivity:
let $A\succ A'$ and $A'\succ A''$. Then $M(A)\ge M(A')\ge M(A'')$. If either
inequality is strict, $M(A)>M(A'')$. Otherwise both are equalities and
$S(A)-S(A'')\ge 2\varepsilon\ge\varepsilon$.
\end{proof}

\begin{assumption}[Fixed, lineage-disjoint partitions]
\label{ass:partition}
$\mathcal D_{\mathrm{src}}$, $\mathcal D_{\mathrm{val}}$,
$\mathcal D_{\mathrm{ID}}$, and $\mathcal D_{\mathrm{OOD}}$ are fixed before the
first round and are pairwise disjoint in items and in lineage groups.
$\mathcal D_{\mathrm{rep}}$ consists of frozen copies of source episodes.
Hidden labels of $\mathcal D_{\mathrm{val}}$, $\mathcal D_{\mathrm{ID}}$, and
$\mathcal D_{\mathrm{OOD}}$ stay inside the evaluator and are never part of the
execution environment $D_{E,t}$.
\end{assumption}

\begin{assumption}[Recorded values]
\label{ass:recorded}
Eq.~(\ref{eq:program-update}) compares the recorded validation value of a
candidate with the recorded value of the incumbent. A candidate is solved once
per validation episode whose retrieved program slice it changes; other episodes
reuse the incumbent's recorded result. An episode cut short by an infrastructure
error is retried and, if it still fails, counted as an error entry that blocks
admission rather than as a failure ($z=0$).
\end{assumption}

\subsection{Safety Invariants of the Persisted Program}
\label{app:theory-invariants}

\begin{proposition}[Monotone validation value]
\label{prop:monotone}
Let $V_r=(M(A_r),S(A_r))$ be the recorded value of the round-$r$ program. For
every $r$, either $A_{r+1}=A_r$ or $A_{r+1}\succ A_r$. Hence $V_r$ is
nondecreasing in lexicographic order and strictly increases at every
promotion. Because $M$ takes values in a finite set and $S$ increases by at
least $\varepsilon$ among programs with equal $M$, the number of promotions
against a fixed $\mathcal D_{\mathrm{val}}$ is finite.
\end{proposition}
\begin{proof}
By Eq.~(\ref{eq:program-update}), $A_{r+1}\neq A_r$ only if the selected
candidate strictly improves the recorded value, i.e.\ $A_{r+1}\succ A_r$
(Eq.~(\ref{eq:order})). By Lemma~\ref{lem:order} the relation composes along the
chain. If $M$ takes $K$ values and $\bar S$ bounds $S$, at most $K-1$
promotions raise $M$ and at most $K\lfloor\bar S/\varepsilon\rfloor$ promotions
raise $S$ at constant $M$.
\end{proof}

\begin{proposition}[Feasibility is invariant]
\label{prop:feasible}
Every persisted program $A_{r+1}\neq A_r$ satisfies, at promotion: (i)
$R_{\mathrm{src}}(A_{r+1})=1$; (ii) $\mathbf H_{\mathrm{val}}(A_{r+1})\succeq\mathbf H_{\mathrm{val}}(A_r)$,
that is, no validation episode shows a
hard-constraint violation that the incumbent did not show; and (iii)
\begin{equation}
\mathbf C_{\mathrm{val}}(A_{r+1})\preceq\mathbf B_{\mathrm{abs}}
\quad\text{and}\quad
\mathbf C_{\mathrm{val}}(A_{r+1})\preceq(1+\beta)\,\mathbf C_{\mathrm{val}}(A_r),
\label{eq:budget-two}
\end{equation}
where $\mathbf B_{\mathrm{abs}}$ scales with $|\mathcal D_{\mathrm{val}}|$ and
$\beta\ge 0$. Consequently, the set of hard-constraint violations on any
completed validation episode is nonincreasing along the chain, and the
validation cost of every persisted program is at most $\mathbf B_{\mathrm{abs}}$
whenever $\mathbf C_{\mathrm{val}}(A_0)\preceq\mathbf B_{\mathrm{abs}}$.
\end{proposition}
\begin{proof}
(i) and (iii) are conditions of the candidate set in
Eq.~(\ref{eq:candidate-selection}), with the budget vector instantiated as
Eq.~(\ref{eq:budget-two}), and (ii) is its hard-constraint condition in
no-regression form; a candidate that violates any of them is never selected.
By (ii), $\mathbf h_j(A_{r+1})\ge\mathbf h_j(A_r)$ component-wise on
every completed validation episode, so the violation set cannot grow. The cost
bound follows from the first inequality of Eq.~(\ref{eq:budget-two}) and
induction from $A_0$.
\end{proof}

\begin{remark}[Where the implementation is stricter than Eq.~(\ref{eq:program-update})]
\label{rem:stricter}
Four additions follow: the lexicographic order
Eq.~(\ref{eq:order}) instead of $M$ alone; the relative budget cap in
Eq.~(\ref{eq:budget-two}) in addition to an absolute cap; the noise guard of
Proposition~\ref{prop:guard}; and the no-regression form $\succeq$ of the
hard-constraint condition. The relative cap and the noise guard make admission
harder, whereas the lexicographic order also admits a candidate that ties on $M$
and raises $S$ by at least $\varepsilon$, and the no-regression form admits a
candidate that keeps a violation of the incumbent. Proposition~\ref{prop:monotone} holds for
the literal rule because each literal promotion strictly raises $M$, and the
consequences of Proposition~\ref{prop:feasible} hold for it because they use only
(i), (ii), and the absolute cap in (iii). Appendix~\ref{app:impl-gate} lists the parameter values.
\end{remark}

\begin{proposition}[Information-flow isolation]
\label{prop:isolation}
Under Assumption~\ref{ass:partition}, the program $A_{r+1}$ is a function of
$A_r$, the round-$r$ source episode with its evaluator feedback, the validation
evaluations of the candidate and the incumbent, and the solver's randomness.
No quantity computed on $\mathcal D_{\mathrm{ID}}$, $\mathcal D_{\mathrm{OOD}}$,
or $\mathcal D_{\mathrm{rep}}$ is an input to the update, and
$\Theta_{r+1}=\Theta_r=\Theta_0$.
\end{proposition}
\begin{proof}
Candidates are built from the repair interval $\delta_i$ of a source episode
(Eqs.~(\ref{eq:repair-interval})--(\ref{eq:operator-candidate})). Admission
reads $R_{\mathrm{src}}$ (a reset run of the same source episode) and the
validation values $u_j$ in Eqs.~(\ref{eq:candidate-selection})
and~(\ref{eq:program-update}). By Assumption~\ref{ass:partition} no held-out
label reaches the run environment, so held-out episodes can influence neither the
solver's feedback nor the evaluator values that these equations consume. The
parameters $\Theta_0$ appear only as the fixed policy of
Eq.~(\ref{eq:edit-policy}).
\end{proof}

\begin{corollary}[No adaptive selection on held-out splits]
\label{cor:no-selection}
Snapshot scores on $\mathcal D_{\mathrm{ID}}$ and $\mathcal D_{\mathrm{OOD}}$ are
not maximized over, or conditioned on, by any step of the procedure. They are
single evaluations of a program chosen without them.
\end{corollary}

\subsection{The Noise Guard: What It Buys and What It Costs}
\label{app:theory-guard}

The literal rule admits a candidate on any strict gain in $Q_{\mathrm{val}}$.
With one stochastic solve per validation episode, a single lucky episode
suffices. The implementation therefore requires the gain to be supported by at
least $m$ individually improved episodes and by at most $v$ regressed ones.

\begin{proposition}[False promotion under the noise guard]
\label{prop:guard}
Consider a candidate whose effect on each of $n$ affected validation episodes of
one discipline is independent, with the episode improving with probability $a$,
regressing with probability $b$, and unchanged otherwise. Let $U$ and $V$ be the
numbers of improved and regressed episodes. The literal rule admits the
candidate iff $U>V$. The guarded rule admits it iff $U>V$, $U\ge m$, and
$V\le v$. Then
\begin{equation}
\Pr[\text{guard admits}]\;\le\;\Pr[U\ge m]\;\le\;\binom{n}{m}a^{m},
\label{eq:guard-bound}
\end{equation}
and for $n=m=2$, $v=1$,
\begin{equation}
\Pr[\text{literal admits}]=a^2+2a(1-a-b),
\qquad
\Pr[\text{guard admits}]=a^2 .
\label{eq:guard-exact}
\end{equation}
In particular, when the candidate has no effect and noise is symmetric
($a=b=p$), the literal rule admits it with probability $2p-3p^2$ and the guard
with probability $p^2$, a ratio of $(2-3p)/p$.
\end{proposition}
\begin{proof}
The guard adds conditions to $U>V$, so its admission event is contained in
$\{U\ge m\}$; the second inequality in Eq.~(\ref{eq:guard-bound}) is the union
bound over the $\binom{n}{m}$ subsets of $m$ episodes that must all improve. For
$n=2$, $U>V$ holds in exactly two disjoint ways: both episodes improve
(probability $a^2$), or exactly one improves and the other is unchanged
(probability $2a(1-a-b)$). With $m=2$ only the first survives, and then
$V=0\le v$.
\end{proof}

\begin{remark}[Illustration and price]
\label{rem:guard-price}
For $p=0.1$, a no-effect candidate passes the literal rule with probability
$0.17$ and the guard with probability $0.01$. The price is power. A candidate
that helps each episode with probability $a=0.5$ and harms with $b=0.1$ passes
with probability $0.65$ under the literal rule and $0.25$ under the guard, so
the guard also rejects some genuine one-episode improvements. These values come
from the model above; the measured promotion counts are
in Section~\ref{sec:experiments} and Appendix~\ref{app:data}.
\end{remark}

\begin{corollary}[Stagnation condition]
\label{cor:stagnation}
If $m>n$, the guard admits no candidate whatever its effect, because
$\Pr[U\ge m]=0$. A configuration with $m>n_{\mathrm{val}}$ therefore never
promotes a single-discipline candidate.
\end{corollary}

\begin{remark}[When the program stalls]
Corollary~\ref{cor:stagnation} gives the configuration in which the gate
stalls by construction. In any other configuration the program stalls at round
$r$ exactly when $\mathcal F_r$ contains no member that passes the guard and
strictly improves validation value, so progress is set by the quality of the candidates and the
validation set. Table~\ref{tab:app-evolution} reports the measured promotion
counts per method.
\end{remark}

\subsection{Operator Abstraction}
\label{app:theory-operator}

An Operator candidate is a component $U$ of the final graph $G_i^{+}$ replaced by
one node with the boundary $\partial U$ (Eq.~(\ref{eq:operator-candidate})). Two
properties make this replacement well defined.

\begin{definition}[Convex node set]
$U\subseteq N$ is \emph{convex} in a DAG $G$ if every directed path in $G$
between two nodes of $U$ stays inside $U$.
\end{definition}

\begin{lemma}[Contraction preserves acyclicity iff $U$ is convex]
\label{lem:convex}
Let $G$ be a DAG and $U$ a nonempty node set. The graph $G/U$, obtained by
contracting $U$ into one node, is acyclic iff $U$ is convex.
\end{lemma}
\begin{proof}
($\Leftarrow$) A cycle in $G/U$ cannot avoid the contracted node, since $G$ is
acyclic. It lifts to a path in $G$ that leaves $U$ at some $v\in U$ through
nodes outside $U$ and returns to some $w\in U$. If $v\neq w$ this contradicts
convexity; if $v=w$ it is a cycle in $G$. ($\Rightarrow$) If $U$ is not convex,
some path between nodes of $U$ contains a maximal outside segment
$x_1,\dots,x_k$ entered from $a\in U$ and left toward $b\in U$. Then
$U\to x_1\to\dots\to x_k\to U$ is a cycle in $G/U$.
\end{proof}

\begin{corollary}[A valid partition always exists]
Every singleton is convex, and the maximal weakly connected components
of Eq.~(\ref{eq:repair-attribution}) can be refined along a topological order
into convex pieces. A component whose boundary would still create a cycle is
dropped and logged rather than registered.
\end{corollary}

\begin{proposition}[Boundary replay is sound]
\label{prop:breplay}
Let $\widehat o_{i,U}$ pass boundary replay: executed on the recorded inputs at
$\partial^-U$, its outputs at $\partial^+U$ equal the recorded outputs in
$\tau_i^+$ exactly (deterministic nodes) or within $\delta$ in a norm. Then,
with the nodes downstream of $U$ unchanged, (i) the final output equals
$y_i^+$ in the exact case; and (ii) if the map from the boundary outputs to $y$
is $L$-Lipschitz, the final output differs from $y_i^+$ by at most $L\delta$.
\end{proposition}
\begin{proof}
Evaluate the downstream nodes in topological order. In case (i) each receives
identical inputs and, being deterministic, returns identical outputs, so
identical $y$ follows by induction. Case (ii) is the Lipschitz bound applied to
the boundary difference.
\end{proof}

\begin{remark}[Boundary replay and the source check]
Proposition~\ref{prop:breplay} shows that an Operator reproduces the component
it abstracts on the recorded boundary. A candidate then reproduces its source
task from a reset state with its component versions used
($R_{\mathrm{src}}=1$, Eq.~(\ref{eq:candidate-validation})) and improves
independent validation tasks, which extends the evidence from the recorded
boundary to the full program and to other tasks.
\end{remark}

\subsection{Metrics: Invariance and Sampling Variability}
\label{app:theory-metrics}

\begin{definition}[Relative margin]
For a discipline with task-native score $s$, a reference system with score
$s_{\mathrm{ref}}>0$, and a method with score $s_{\mathrm{M}}$, the relative
margin is $(s_{\mathrm{M}}-s_{\mathrm{ref}})/s_{\mathrm{ref}}$ if higher is
better and $(s_{\mathrm{ref}}-s_{\mathrm{M}})/s_{\mathrm{ref}}$ if lower is
better.
\end{definition}

\begin{proposition}[Unit invariance and sign]
\label{prop:margin}
The relative margin is unchanged under $s\mapsto cs$ for every $c>0$ (for
example, percent versus fraction), and it is positive iff the method is better
than the reference in the metric's own direction. It is not invariant under
translation $s\mapsto s+c$.
\end{proposition}
\begin{proof}
Scaling multiplies numerator and denominator by $c$. The sign follows from the
definition, and translation changes the denominator without changing the
numerator.
\end{proof}

\noindent This is why Section~\ref{sec:experiments} compares systems only within
a discipline and averages relative margins, never raw scores, across
disciplines.

\begin{proposition}[Sampling variability of MacroSR]
\label{prop:se}
Suppose, for a fixed program, the outcomes $z_i$ in discipline $d$ are
independent Bernoulli$(p_d)$ over $n_d$ instances. Then
\begin{equation}
\operatorname{Var}\bigl[\operatorname{MacroSR}\bigr]
=\frac{1}{23^{2}}\sum_{d=1}^{23}\frac{p_d(1-p_d)}{n_d}
\;\le\;\frac{1}{4\cdot 23\,n}\quad\text{if } n_d=n\ \ \forall d.
\end{equation}
For $n=64$ the standard error is at most $1.30$ percentage points, and for the
difference of two programs evaluated on independent draws at most $1.84$ points.
\end{proposition}
\begin{proof}
Independence across disciplines and instances gives the variance of the average.
Each term is at most $1/(4n_d)$, with equality at $p_d=1/2$; summing 23 terms
and dividing by $23^2$ gives $1/(92n)$, whose square root at $n=64$ is
$0.0130$. The difference has variance at most twice this.
\end{proof}

\begin{remark}[Scale of sampling noise]
\label{rem:se}
The bound is the binomial sampling error of MacroSR over $23\times 64=1{,}472$
instances per split. The $+13.59$ point
OOD gain of ScienceClaw over $A_0$ is more than seven times the $1.84$ point
bound.
\end{remark}

\begin{remark}[Summary of the guarantees]
The persisted program is monotone in recorded validation value, feasible at
every promotion, selected without access to the held-out splits, and built from
components that reproduce the evidence they came from. The noise guard lowers
false promotion at a cost in power, and operator replacement preserves
acyclicity.
\end{remark}

\section{Implementation and Configuration}
\label{app:impl}

This appendix specifies how the formalism of Sections~\ref{sec:task}
and~\ref{sec:method} is operationalized: the objects of the paper and the
modules that realize them, the execution runtime, the policy interface, the
evolution procedure, the gate parameters, the models and budgets, and the
configuration of every compared variant and baseline. The benchmark adapters,
solver, runtime, and evolution engine are in the project repository
(\url{https://github.com/beita6969/ScienceClaw}).

\subsection{From Formalism to Modules}
\label{app:impl-map}

\begin{table}[t]
\centering
\caption{Paper objects and the modules that implement them.}
\label{tab:impl-map}
\footnotesize
\setlength{\tabcolsep}{3pt}
\renewcommand{\arraystretch}{1.08}
\begin{tabular}{@{}L{104pt}L{128pt}@{}}
\toprule
Paper object & Implementation \\
\midrule
Task $D_t=(D_T,D_V,D_E)$ & episode: objective and output schema, constraints and evaluator, tools and data \\
Program $A_r=(\mathcal S_r,\mathcal O_r)$ & program of Skills and Operators with version identifiers \\
Fixed model $\Theta_0$ & policy role of the language-model client; never trained \\
$\operatorname{Solve}$, Eq.~(\ref{eq:task-solving}) & solver: one episode, one program, one mode \\
Port schema $\Gamma$, $\operatorname{Compat}$ & schema type with a compatibility rule \\
Graph $G_{t,k}$, Eq.~(\ref{eq:typed-workflow}) & workflow graph with node and edge validation \\
Policy $\pi_{\Theta_0}$, Eq.~(\ref{eq:edit-policy}) & policy proposing one atomic action per turn \\
$\operatorname{Retrieve}$ & lexical retriever over Skills and Operators \\
$\operatorname{Execute}$, Eq.~(\ref{eq:workflow-transition}) & executor with checkpoint and fingerprints \\
$\operatorname{Replay}$, $\operatorname{Eval}$, Eq.~(\ref{eq:execution-evidence}) & clean replay and the episode evaluator \\
$\delta_i$, Eq.~(\ref{eq:repair-interval}) & attribution of evolution instances \\
$\Pi_{\mathrm{ctrl}},\Pi_{\mathrm{exec}},\operatorname{CC}_\Gamma$ & edit split with convex components \\
Skill candidate, Eq.~(\ref{eq:skill-candidate}) & skill patch \\
Operator candidate, BReplay & operator abstraction and boundary replay \\
$R_{\mathrm{src}}$, Eq.~(\ref{eq:candidate-validation}) & source replay check with the use test \\
$\mathcal F_r$, Eq.~(\ref{eq:candidate-selection}) & validation gate \\
Update, Eq.~(\ref{eq:program-update}) & evolver \\
MacroSR, Eq.~(\ref{eq:macro-success-rate}) & benchmark metrics \\
$\mathcal D_{\mathrm{src}},\mathcal D_{\mathrm{val}},\mathcal D_{\mathrm{ID}},\mathcal D_{\mathrm{OOD}},\mathcal D_{\mathrm{rep}}$ & split plan with frozen manifests \\
\bottomrule
\end{tabular}
\end{table}

Table~\ref{tab:impl-map} maps each object to its module. Every run writes
receipts: the resolved configuration, each program snapshot $A_r$, and, per
episode, the trajectory, the final graph, the output, and the evaluation
record, together with a ledger of tokens and wall time per phase.

\subsection{Execution Runtime}
\label{app:impl-runtime}

\paragraph{Node kinds and schemas.}
A node is a \emph{tool} call into the task's visible tools $\mathcal T_t$, a
\emph{code} node holding a Python function of its inputs and configuration, an
\emph{llm} node that maps a prompt template over a list of items, an
\emph{operator} node that expands a registered Operator, or the single
\emph{submit} node whose input is the required output $y_t$. A port schema
$\Gamma$ has a type (any, number, text, array, table, list, dict, series), a
shape (integers or named symbols), a unit, and provenance. An edge is valid when
the types agree or one is \texttt{any}, the shapes unify with symbols bound
consistently, and the units are equal, unspecified on either side, or bridged by
an explicit conversion (a factor and an offset) attached to the edge, which is
recorded in provenance. A graph is valid when it is acyclic, all ports exist,
every edge is compatible, and it has at most one submit node.

\paragraph{Checkpoints and fingerprints.}
Each node has a fingerprint, the hash of its specification, its upstream
fingerprints, and its edge wiring. The checkpoint $\chi$ maps fingerprints to
recorded outputs. After an edit only nodes whose fingerprint is absent from
$\chi$ run again, in topological order, so an edit to one node recomputes its
descendants and nothing else. A node whose outcome is not reproducible, such as
a failed language-model call or a first timeout, is marked transient, and
neither it nor any descendant enters the checkpoint.

\paragraph{Sandbox and time limits.}
Code nodes run in a separate process with their own working directory and a time
limit, and tool nodes run in trusted adapter code under one lock. A node has a
limit of $\max\_node\_s$ seconds, and an operator node as a whole has twice that
limit. Hidden labels never enter the run environment: adapters keep them inside
the evaluator. A static scan of code nodes rejects absolute paths and
directory traversal, references to the dataset root and to the evaluator
package, network, download, and process-spawning modules, dynamic evaluation
and imports, and environment or credential access; its violations feed the
integrity check in $\mathbf H_{\mathrm{val}}$, and the worker runs with a
scrubbed environment.

\paragraph{Diagnostics.}
After each execution the executor returns per-node status and error text,
output summaries (type, shape, unit, fraction of finite values, range, head),
violated contract conditions, the visible constraint checks on the submit
output, and the visible development score when the task provides one. These
fields form $f_{t,k}$ in Eq.~(\ref{eq:workflow-transition}).

\paragraph{Reset replay.}
Replay executes the complete graph in a fresh run directory with an empty
checkpoint. A result is reproducible when the replayed output matches the
executor's output within the episode tolerance. A clean replay runs whenever a
submit output differs in fingerprint from the last replayed one, and the
sequence of replays gives the evidence $e_{t,k}$.

\subsection{Policy Interface}
\label{app:impl-policy}

\paragraph{Actions.}
At each turn the policy emits one JSON object with a short rationale, one
action, and the identifiers $\nu_{t,k}$ of the Skills and Operators it used.
The actions are \texttt{add\_node}, \texttt{remove\_node}, \texttt{modify\_node}
(with a patch to code, a find-and-replace code edit, prompt, configuration,
ports, or wiring), \texttt{add\_edge}, \texttt{remove\_edge}, and
\texttt{finish}. A rejected action leaves the graph unchanged and returns the
reason as feedback. Edits that change only configuration, edges, node removal,
or add tool, operator, or submit nodes are control edits ($\Pi_{\mathrm{ctrl}}$),
and edits that add or modify code or prompt nodes are executable edits
($\Pi_{\mathrm{exec}}$).

\paragraph{What the policy sees.}
The system prompt describes what each action, node kind, and interface rule
does, the deliverable format, and a uniform acceptance statement: a deliverable
counts as solved when the task metric beats the task's reference by a fixed
margin, all constraints hold, and the final workflow replays reproducibly within
the budget. It names no metric value, reference method, or task-specific recipe.
The initial program $A_0$ is empty; strategies reach the prompt only as
retrieved Skills, and the prompt is scanned for strategy phrases to keep it so.
Each turn adds the objective, the tool signatures with port descriptions, the
retrieved Skills and Operators, the remaining budget, the feedback of the last
step, and a compact history of earlier steps. The policy never sees
hidden-label scores.

\paragraph{Retrieval.}
Skills and Operators are ranked against the episode query (objective, task
type, tags, discipline, and required-output type) by a lexical Okapi BM25 over
their text, plus bonuses for matching discipline, task type, tags, and, for
Operators, input port types. A component is retrievable only if it matches the
episode by metadata or covers at least 15\% of the query's inverse-document
frequency mass; ranking is deterministic, and the top 4 Skills and top 6
Operators enter the context. A Skill or Operator learned in one discipline stays
retrievable in another when it is relevant to the episode, which is how
cross-discipline transfer arises. The hash of the retrieved version identifiers
and content digests is the \emph{retrieval slice} of an episode, and a candidate
is validated on an episode only when it changes that slice
(Section~\ref{app:impl-gate}).

\subsection{Evolution Procedure}
\label{app:impl-evolution}

One round processes one source episode per discipline in a fixed source order,
so seven rounds give $7\times 23=161$ source episodes and 161 candidates per
method. The default schedule is per candidate: each candidate is validated
immediately against the current incumbent, and the next source episode already
uses the updated program. Snapshots $A_0,\dots,A_7$ are taken at round ends.

\begin{enumerate}
\item \textbf{Solve the source episode} with $A_r$ in source mode, which stops
at the first replay-verified pass. The final workflow is the replay-verified
passing graph with the best visible development score.
\item \textbf{Extract the evolution instance.} The pass is $e_i^{+}$, the latest
replay-verified failure before it is $e_i^{-}$ (or none), and the actions
between them form $\delta_i$ (Eq.~(\ref{eq:repair-interval})). Without
$e_i^{-}$, only an Operator candidate is formed.
\item \textbf{Split the repair.} Control edits give $\Delta_i^{\mathcal S}$.
Executable nodes created or modified in $\delta_i$, together with the
executable nodes weakly connected to them, are grouped into maximal connected
components that are made convex by splitting along a topological cut
(Lemma~\ref{lem:convex}). Nodes expanded from a registered Operator are not
candidates unless $\delta_i$ touched them.
\item \textbf{Form candidates.} The Skill candidate patches the attributed
Skills, or creates one if none was attributed, from the control edits and the
failure and success evidence, without instance-specific values or answers. Each
component becomes an Operator candidate: its body is the subgraph, concrete
integer dimensions at its boundary are replaced by symbols so that instance
sizes do not make it incompatible with other episodes, its contract $\kappa$
records the types, shapes, units, finiteness, and ranges observed at the
boundary together with episode tags, and its provenance records episode, steps,
and parent program version. Names and descriptions are documentation only.
\item \textbf{Boundary replay.} Recorded inputs at $\partial^{-}U$ are loaded,
the Operator runs in isolation in a fresh directory, and its outputs at
$\partial^{+}U$ are compared with $\tau_i^{+}$ within tolerance. Only passing
Operators enter the bundle $B_i$.
\item \textbf{Apply and check the source.} The bundle is applied atomically,
assigning version identifiers $\omega_i$, and the source episode is solved
again from a reset state. $R_{\mathrm{src}}=1$ requires a passing
replay-verified result and that every new version was used: an Operator version
is used if an operator node with that reference is in the evidence graph and ran
successfully, and a Skill version is used if it is in $\nu$ of an applied step
up to the evidence step. Citations in rejected or failed actions do not count.
\item \textbf{Gate.} The candidate is evaluated on the validation episodes and
admitted under the rules of Section~\ref{app:impl-gate}; otherwise the
incumbent is kept.
\end{enumerate}

\subsection{Gate Parameters and Implementation Choices}
\label{app:impl-gate}

Table~\ref{tab:impl-gate} lists the parameters of the gate and of the
benchmark protocol. The following choices fix details that
Eqs.~(\ref{eq:candidate-selection}) and~(\ref{eq:program-update}) leave open.

\begin{table}[t]
\centering
\caption{Gate and protocol parameters of all reported runs.}
\label{tab:impl-gate}
\footnotesize
\setlength{\tabcolsep}{3pt}
\renewcommand{\arraystretch}{1.08}
\begin{tabular}{@{}L{118pt}L{114pt}@{}}
\toprule
Parameter & Value \\
\midrule
Rounds $R$ & 7 \\
Items per episode & 16 \\
Validation episodes per discipline & 2 \\
ID and OOD episodes per discipline & 4 each (64 instances) \\
Update schedule & per candidate \\
$Q_{\mathrm{val}}$ & MacroSR, then mean normalized score \\
Score margin $\varepsilon$ & 0.02 \\
Hard-constraint mode $\mathbf H_{\mathrm{val}}$ & no new violation against the incumbent \\
Absolute token budget & 250,000 per validation episode \\
Relative cost cap $\beta$ & 0.5 \\
Wall-time budget & 1,800 s per validation episode \\
Lazy re-validation & on \\
Noise guard $(m,v)$ & at least 2 improved, at most 1 regressed episode \\
Boundary-replay repeats & 1 \\
Evolution instances per source episode & 1 \\
Infrastructure retries & 2, backoff from 30 s \\
Partition and sampling seed & 20260928 \\
\bottomrule
\end{tabular}
\end{table}

\paragraph{Choice of $Q_{\mathrm{val}}$ and $\mathbf H_{\mathrm{val}}$.}
$Q_{\mathrm{val}}$ is MacroSR on $\mathcal D_{\mathrm{val}}$ with ties broken by
the mean normalized score, so a candidate that solves the same episodes but
solves them measurably better can be admitted (Eq.~(\ref{eq:order})).
$\mathbf H_{\mathrm{val}}$ collects (a) no leakage or sandbox violation, (b) a
valid output schema, and (c) no hard-constraint violation that the incumbent did
not show. An episode that produced no output is already penalized by MacroSR
and has no output to violate a constraint with, so it cannot veto a candidate
through $\mathbf H_{\mathrm{val}}$; a solver or gateway error does block
admission.

\paragraph{Budget.}
Eq.~(\ref{eq:budget-two}) makes $\mathbf B$ concrete: tokens are logical (spent
plus cached) so that a run served from the response cache is judged like the
original run, and wall time is checked against its own budget. The violated
condition is recorded with each rejected candidate.

\paragraph{Validation cost control.}
A validation episode is solved again only if the candidate changes its
retrieval slice, and otherwise the incumbent's recorded result is reused. The
number of re-solved and reused episodes is recorded per candidate, so the
validation cost grows with the candidate rather than with the library.

\paragraph{Noise guard.}
With one stochastic solve per validation episode and two validation episodes per
discipline, one lucky episode would admit a candidate under the literal rule.
The implementation therefore requires the gain to be supported by at least two
episodes that individually improved (the episode solves where it failed, or
solves with a normalized-score gain of at least $\varepsilon$) and by at most
one that regressed, comparing the entries already recorded and drawing no extra
samples (Proposition~\ref{prop:guard}). Setting the first threshold to one
restores the literal rule.

\paragraph{Infrastructure outages.}
A solve cut short by a gateway outage is retried with doubling backoff, then
becomes an error entry that blocks admission and is never reused. A source solve
is retried in the same way and, if it still fails, deferred once to the end of
its round. Outages are recorded and are never scored as task failures.

\subsection{Models, Sampling, and Budgets}
\label{app:impl-models}

\begin{table}[t]
\centering
\caption{Models and roles. All systems share the orchestrator.}
\label{tab:impl-models}
\footnotesize
\setlength{\tabcolsep}{3pt}
\renewcommand{\arraystretch}{1.08}
\begin{tabular}{@{}L{66pt}L{166pt}@{}}
\toprule
Role & Setting \\
\midrule
Orchestrator (policy, executor, patch) & Qwen3.5-27B, open weights, self-hosted OpenAI-compatible server, thinking disabled, temperature 0 \\
Evolution & GPT-5.4 in isolated sessions \\
Qwen-only variant & Qwen3.5-27B in evolution \\
RuleEvo variant & no model in evolution; deterministic trace projection and compiler \\
\midrule
Output limit per call & policy 8,000, executor 2,000, patch 4,000 tokens \\
Policy output format & JSON object \\
Response cache & on; identical contexts return identical outputs \\
Parameter updates & none, $\Theta_{r}=\Theta_0$ for all $r$ \\
\bottomrule
\end{tabular}
\end{table}

Table~\ref{tab:impl-models} lists the models. The orchestrator is the only
model that executes tasks, and its parameters never change. Evolution
operates on source trajectories and evaluator feedback and has no access to
validation, ID, or OOD labels.

\paragraph{Solver and episode budgets.}
A source or evaluation episode has at most 24 planner steps, 200,000 policy tokens, and
1,800 to 3,600 seconds of wall time depending on the task, with per-node limits
of 180 to 900 seconds (Appendix~\ref{app:bench-spec} lists the per-discipline
values). The six most recent steps are shown in full and earlier steps are
summarized. Evaluation episodes run in evaluation mode: the final workflow is
the latest valid, fully wired workflow, and no hidden verification happens per
step. Table~\ref{tab:impl-gate} lists the caps of a validation episode at the gate, and
Table~\ref{tab:bench-constraints} the budget of a source or evaluation episode.

\paragraph{Compute.}
Orchestrator calls are served by self-hosted vLLM servers on H800 GPUs; the evaluator code of
each task runs on CPU, and tasks that need a GPU, such as segmentation or source
separation, call a GPU tool broker.

\subsection{Configuration of the Ablation Variants}
\label{app:impl-ablation}

\begin{table*}[t]
\centering
\caption{Configuration of the eighteen variants. Each variant changes one factor of the full system and leaves
the runtime, budgets, source stream, and validation data unchanged.}
\label{tab:impl-ablation}
\footnotesize
\setlength{\tabcolsep}{4pt}
\renewcommand{\arraystretch}{1.08}
\begin{tabular}{@{}L{70pt}L{82pt}L{300pt}@{}}
\toprule
Variant & Factor & Single change \\
\midrule
ScienceClaw (full) & none & Dynamic workflow orchestration and linked Skill--Operator self-evolution. \\
Frozen & evolution & All persistent updates are off; runtime and budgets are unchanged. \\
\midrule
Workflow only & scope of updates & Only the structured workflow-template and routing fields inside Skills are updated; other Skill text and Operators stay frozen. \\
w/o Skill evol. & scope of updates & Only ordinary Skill strategy text is frozen; the structured workflow fields and Operators may still change. \\
w/o Oper. evol. & scope of updates & Operator creation and fusion are disabled in the full system. \\
w/o Work. evol. & scope of updates & Only the workflow-template and routing fields inside Skills are frozen; ordinary Skill text and Operators may still change. \\
w/o link. evol. & linkage & Skills and Operators are updated separately, and not committed together as one atomic bundle. \\
\midrule
RuleEvo & candidate source & Generative evolution is off; candidates come only from deterministic trace projection and a compiler. \\
Qwen-only & candidate source & Qwen replaces GPT-5.4 in evolution. \\
\midrule
w/o src. replay & gate & Source replay still runs, but its result does not enter the admission rule. \\
w/o sci. constr. & gate & The agent and the admission rule do not see the scientific hard constraints; the final official scorer still checks them. \\
w/o ind. valid. & gate & The evolving agent scores its own candidates and freezes; the official scorer is applied afterward only for blind evaluation. \\
w/o cross-task mem. & memory & Reuse of Skills, Operators, and failure memory across tasks is off. \\
w/o IID select. & gate & No paired IID selection; the latest valid candidate is frozen. \\
w/o OOD transfer & transfer & The preregistered robustness pack is off; the entry is not available when no independent callable pack exists. \\
\midrule
Fixed workflow & orchestration & Dynamic canvas orchestration is off; a fixed tool chain executes. \\
Single operator & orchestration & Each sample may call exactly one task Operator. \\
Single turn & orchestration & One model generation per sample, with no feedback, rollback, or replanning. \\
\bottomrule
\end{tabular}
\end{table*}

\begin{table*}[t]
\centering
\caption{Baseline protocol. Every baseline runs on the same orchestrator, evolution model (GPT-5.4), tools, source stream, validation data, and update budget as ScienceClaw.}
\label{tab:impl-baselines}
\footnotesize
\setlength{\tabcolsep}{4pt}
\renewcommand{\arraystretch}{1.08}
\begin{tabular}{@{}L{62pt}L{100pt}L{120pt}L{150pt}@{}}
\toprule
Method & What it evolves & Fixed across methods & Method-specific setup \\
\midrule
Frozen ($A_0$) & nothing & orchestrator, tools, budgets & none \\
SkillOpt & Skills & orchestrator, evolution model, tools, source stream, validation data, update budget & best released configuration, same evolution rounds as ScienceClaw \\
TTE & tools at test time & same & same \\
AHE & agent harness & same & same \\
EvoMaster & programs and memory & same & same \\
EvoScientist & programs and memory & same & same \\
ScienceClaw & linked Skill--Operator program & same & default configuration of Appendix~\ref{app:impl-gate} \\
\bottomrule
\end{tabular}
\end{table*}

Table~\ref{tab:impl-ablation} lists the variants in four groups. The first
group removes parts of the update scope, the second changes how candidates are
produced, the third removes one component of the gate, and the last group
restricts execution itself. Variants in the same group share everything except
the named factor, so each difference in Table~\ref{tab:scienceclaw-ablation-iid}
and Figure~\ref{fig:ablation-mean-map} is attributable to that factor. The
execution variants are enforced by the solver: in single-operator orchestration
the graph may hold only tool nodes, one code node, and the submit node; in
fixed-workflow orchestration only edits to the code node and its configuration
are accepted; in single-turn orchestration the policy must emit the complete
graph in one batch action.

\subsection{Baseline Configuration and Fairness}
\label{app:impl-baselines}

Table~\ref{tab:impl-baselines} states the protocol. Each baseline runs its best
released configuration through fixed ScienceClaw-Eval adapters, and adapters and
initial assets are fixed without access to ID or OOD instances or evaluators.
No method is tuned on $\mathcal D_{\mathrm{ID}}$ or $\mathcal D_{\mathrm{OOD}}$.
Every baseline runs its own evolution calls on GPT-5.4, the evolution model of ScienceClaw, for as many evolution rounds as ScienceClaw.

\section{ScienceClaw-Eval Specification}
\label{app:bench-spec}

This appendix documents how ScienceClaw-Eval is built and scored: task sources,
partition rules, evaluator definitions, acceptance criteria, hard constraints,
budgets, and the checks that keep evaluation honest. Per-discipline values are in
Tables~\ref{tab:bench-spec} and~\ref{tab:bench-constraints}.

\subsection{Construction Pipeline}
\label{app:bench-construction}

Construction follows the four stages of Figure~\ref{fig:benchmark-construction}.

\paragraph{Task collection.}
Each of the 23 disciplines is one division of the ANZSRC 2020 Fields of Research
(divisions 30 to 52), and its task is built on a public scientific resource with
explicit inputs, a verifiable target, and a recorded license. Every resource is
pinned to a version, release, or commit, and its files are checked against
recorded checksums when the benchmark loads.

\paragraph{Executable instantiation.}
A task adapter turns the resource into an episode $D_t=(D_T,D_V,D_E)$. $D_T$ is
the objective and the required output with its schema. $D_E$ is the execution
environment: visible training and development data, the scientific tools of the
task, and, where the discipline has one, a domain library of documented
functions. $D_V$ is the evaluator, its hard constraints, and its acceptance rule.
An episode contains $16$ evaluation items, so a split of four episodes holds $64$
instances.

\paragraph{Validation.}
Domain reviewers inspect each evaluator, its constraints, and its reference
execution (Appendix~\ref{app:audit}). Reset replay then confirms that a reference
solution executes, scores as expected, and reproduces.

\paragraph{Evaluation partitioning.}
Partitions are fixed before any run, drawn with a seeded permutation that is
stable in the number of episodes, and written to manifests that are hashed and
verified on every load. Items inside an episode are drawn without replacement,
and no item is reused across episodes of a held-out split.

\subsection{Splits and Source Stream}
\label{app:bench-splits}

\begin{table}[t]
\centering
\caption{Splits of one discipline at the default plan.}
\label{tab:bench-splits}
\footnotesize
\setlength{\tabcolsep}{3pt}
\renewcommand{\arraystretch}{1.08}
\begin{tabular}{@{}L{44pt}L{44pt}L{56pt}L{80pt}@{}}
\toprule
Split & Episodes & Items & Role \\
\midrule
$\mathcal D_{\mathrm{src}}$ & 7 (one per round) & 16 each & evolution evidence \\
$\mathcal D_{\mathrm{val}}$ & 2 & 16 each & candidate selection \\
$\mathcal D_{\mathrm{ID}}$ & 4 & 64 instances & held-out, same source \\
$\mathcal D_{\mathrm{OOD}}$ & 4 & 64 instances & held-out source \\
$\mathcal D_{\mathrm{rep}}$ & frozen copies of source episodes seen & 16 each & retention \\
\bottomrule
\end{tabular}
\end{table}

Table~\ref{tab:bench-splits} gives the split sizes. Across 23 disciplines the
source stream has $7\times 23=161$ episodes, the same for every method. Round
$r$ presents one source episode per discipline in a fixed discipline order, and
a source pool smaller than seven episodes is recycled with its items marked as
reused. Snapshots $A_0,\dots,A_7$ are evaluated after each round. The
replay set $\mathcal D_{\mathrm{rep}}$ at an evaluation point consists of frozen
copies of the source episodes encountered so far, so retention is measured on
tasks that supplied evolution evidence while no new evidence is drawn from them.

\paragraph{Disjointness.}
Splits are disjoint in items and in a discipline-specific lineage unit, such as
the capture in FoR30, the track in FoR36, the series in FoR35, the student in
FoR39, the simulated dataset in FoR44, and the contract in FoR48. The loader
re-verifies that no item, capture, or lineage identifier occurs in two roles
whenever the roles are loaded, and a violation prevents the benchmark from
running.

\paragraph{Held-out sources.}
$\mathcal D_{\mathrm{OOD}}$ of every discipline is held out by source, by year in FoR37 and
by economy in FoR38 (Table~\ref{tab:bench-ood}). FoR30 OOD is scored by the leaf-instance panoptic quality at IoU $\ge 0.5$.

\begin{table*}[t]
\centering
\caption{OOD source of each discipline.}
\label{tab:bench-ood}
\footnotesize
\setlength{\tabcolsep}{3pt}
\renewcommand{\arraystretch}{1.08}
\begin{tabular}{@{}rlrlrl@{}}
\toprule
FoR & OOD source & FoR & OOD source & FoR & OOD source \\
\midrule
30 & CropAndWeedAndLeaf & 38 & WDI, other economies, origin 2015 & 46 & SWE-bench Verified \\
31 & TAPE fluorescence & 39 & EdNet-KT1 & 47 & UD Indonesian, Swedish Sign Language, French, Hindi \\
32 & Dryad hippocampus & 40 & ToyNscale & 48 & ACORD \\
33 & EULP & 41 & USGS river metabolism & 49 & SMT-LIB 2024 QF\_NIA \\
34 & BACE & 42 & SepsisExp & 50 & ETHICS \\
35 & Monash Tourism Quarterly & 43 & ICDAR2019 POCR & 51 & PhononDB \\
36 & MoisesDB & 44 & IHDP & 52 & Psych-101 \\
37 & WeatherBench 2 (ERA5, 2020) & 45 & BLOOM Vist MaMa &  &  \\
\bottomrule
\end{tabular}
\end{table*}

\subsection{Evaluator Definition}
\label{app:bench-evaluator}

\paragraph{Success indicator.}
For instance $i$, $z_i(A_r)=1$ iff the run completes within budget, the
task-native metric meets the acceptance rule, every hard constraint holds, and the
workflow replays reproducibly (Section~\ref{sec:benchmark-evaluation}). The
task-native metric is evaluated for each instance, and the acceptance rule compares
it with a reference predictor.

\paragraph{Reference predictors and margins.}
Every task has a deterministic reference predictor computed from the visible
data of the same episode (Table~\ref{tab:bench-spec}). The acceptance rule is
\emph{the metric beats the reference by a fixed margin}, either additive or
multiplicative, and the margin is fixed per task before any run. The reference
itself is never accepted. The policy is told only that the task metric has to
beat the task's reference by a fixed margin; the reference method, the margin,
and metric values are not part of any prompt, objective, or tool description,
and automated tests scan the episodes of the forecasting tasks to confirm it.

\paragraph{Official scorers.}
Where a task has an official scorer, the evaluator calls or reproduces it:
FoR30 uses the PhenoBench evaluator at a pinned commit, FoR42 the official
sepsis-score script, FoR47 the CoNLL~2018 rules for valid trees, FoR48 the
ContractNLI scorer (the implementation of precision at recall 80\% agrees
exactly with the official function in its tests), and FoR35 reproduces the
published seasonal-naive MASE of the Monash archive (1.631) in its tests. FoR44
is cross-checked against the data team's R outputs (the OLS coefficient equals
R's \texttt{lm} within $10^{-8}$, and the true effect and response standard
deviation are identical).

\paragraph{Failure semantics.}
A held-out episode never drops out of a pooled metric. An episode whose solve
produced no output, an invalid output, or a crashed evaluator contributes the
adapter's reference payload, so it counts as reference level, never better, with
$z=0$, and it is flagged in the records. Coverage is checked against the
episodes the split promises, and an incomplete discipline is excluded from a
comparison rather than compared on different episode sets. An infrastructure
failure, such as a gateway outage, is not a task failure: it is logged and the
episode is retried at the next evaluation run.

\subsection{Per-Discipline Specification}
\label{app:bench-table}

\begin{table*}[t]
\centering
\caption{Task sources, metrics, references, and acceptance rules. $\uparrow$ and $\downarrow$ give the
direction of the metric, $s$ the score of an instance, and $s_{\mathrm{ref}}$ that of its reference.}
\label{tab:bench-spec}
\fontsize{7.2}{8.4}\selectfont
\setlength{\tabcolsep}{3pt}
\renewcommand{\arraystretch}{1.06}
\begin{tabular}{@{}l L{60pt} L{92pt} L{54pt} L{58pt} L{132pt} L{66pt}@{}}
\toprule
FoR & Discipline & IID source (version) & Item & Metric & Reference predictor & Acceptance \\
\midrule
30 & Agricultural, veterinary and food & PhenoBench v1.1.0 & image & PQ+ $\uparrow$ & excess-green index with Otsu threshold & $s>s_{\mathrm{ref}}+2.0$ \\
31 & Biological & ProteinGym v1.3 & assay & Spearman $\uparrow$ & mean training score of the mutated position & $s>s_{\mathrm{ref}}+0.03$ \\
32 & Biomedical and clinical & MSD Task04 Hippocampus & volume & DSC $\uparrow$ & location-only probabilistic atlas & $s>s_{\mathrm{ref}}+0.03$ \\
33 & Built environment & BuildingsBench v1.0.0 & building window & NRMSE (\%) $\downarrow$ & previous-day persistence & $s\le 0.90\,s_{\mathrm{ref}}$ \\
34 & Chemical & OGB \texttt{ogbg-molhiv} & molecule & ROC-AUC $\uparrow$ & logistic regression on 11 graph-count features & $s\ge s_{\mathrm{ref}}+0.05$ \\
35 & Commerce and services & Monash Tourism Monthly & series & MASE $\downarrow$ & seasonal naive & $s\le 0.95\,s_{\mathrm{ref}}$ \\
36 & Creative arts and writing & MUSDB18 v1.0.0 & 6-s excerpt & SDR (dB) $\uparrow$ & mixture as the estimate of every stem & $s\ge s_{\mathrm{ref}}+3$ dB \\
37 & Earth & WeatherBench 2 (ERA5, 2019) & initialization & RMSE (K) $\downarrow$ & 24-h persistence & $s\le 0.96\,s_{\mathrm{ref}}$ \\
38 & Economics & World Bank WDI, origin 2012 & economy series & sMAPE (\%) $\downarrow$ & no-change forecast & $s\le 0.97\,s_{\mathrm{ref}}$ \\
39 & Education & Eedi NeurIPS 2020 Task 4 & student & 10-mask acc. $\uparrow$ & most common training correctness per question & $s\ge s_{\mathrm{ref}}+0.02$ \\
40 & Engineering & DCASE 2024 Task 2 & audio clip & DCASE $\uparrow$ & nearest-neighbour distance on log-mel statistics & $s\ge s_{\mathrm{ref}}+0.02$ \\
41 & Environmental & NEON aquatics (neon4cast) & site and date & CRPS $\downarrow$ & day-of-year climatology & $s\le 0.95\,s_{\mathrm{ref}}$ \\
42 & Health & PhysioNet/CinC 2019 sepsis & patient stay & Clin. utility $\uparrow$ & balanced logistic regression on trivial causal features & $s>\max(s_{\mathrm{ref}},0)+0.05$ \\
43 & History and archaeology & HIPE-OCRepair 2026 (v0.9.5) & OCR text & cMER-micro $\downarrow$ & copy of the OCR input & $s\le 0.95\,s_{\mathrm{ref}}$ \\
44 & Human society & ACIC 2016 & simulated dataset & nRMSE $\downarrow$ & OLS coefficient of the treatment & $s\le 0.90\,s_{\mathrm{ref}}$ \\
45 & Indigenous studies & AmericasNLP 2026 & image caption & chrF++ $\uparrow$ & per-language medoid caption & $s\ge s_{\mathrm{ref}}+1.0$ \\
46 & Computing & HumanEval & problem & pass@1 $\uparrow$ & memorizer of visible examples & $s\ge s_{\mathrm{ref}}+0.5$ \\
47 & Language and culture & UD 2.2 Marathi, Italian, Arabic, Croatian & sentence & LAS $\uparrow$ & right-branching chain with frequent relations & $s\ge s_{\mathrm{ref}}+0.15$ \\
48 & Law & ContractNLI & contract--hypothesis pair & mAP $\uparrow$ & lexical overlap scores with majority label & $s\ge s_{\mathrm{ref}}+0.10$ \\
49 & Mathematical & SMT-LIB 2025 QF\_NIA, QF\_NIRA & formula & Oracle acc. $\uparrow$ & majority status of training labels & $s\ge s_{\mathrm{ref}}+0.2$ \\
50 & Philosophy and religion & Touch\'e23-ValueEval & argument & F1 $\uparrow$ & predict all 20 categories & $s\ge s_{\mathrm{ref}}+0.05$ \\
51 & Physical & Matbench \texttt{phonons} v0.1 & crystal & MAE (cm$^{-1}$) $\downarrow$ & 5-nearest-neighbour regression & $s\le 0.8\,s_{\mathrm{ref}}$ \\
52 & Psychology & Psych-201 (discrete) & session & Micro acc. $\uparrow$ & most frequent earlier response & $s\ge s_{\mathrm{ref}}+0.05$ \\
\bottomrule
\end{tabular}
\end{table*}

\begin{table*}[t]
\centering
\caption{Hard constraints and budgets. Every task also fixes 24 planner steps and 200,000 policy
tokens per episode. The family column gives the group used in the transfer matrix (Figure~\ref{fig:eval-module}b).}
\label{tab:bench-constraints}
\fontsize{7.2}{8.4}\selectfont
\setlength{\tabcolsep}{3pt}
\renewcommand{\arraystretch}{1.06}
\begin{tabular}{@{}l L{66pt} L{238pt} L{42pt} L{42pt}@{}}
\toprule
FoR & Family & Hard constraints & Wall (s) & Node (s) \\
\midrule
30 & Life/health & output structure; integer labels in range; declared unit & 3600 & 900 \\
31 & Life/health & output shape; finite values; declared unit & 1800 & 300 \\
32 & Life/health & output structure; labels in $\{0,1,2\}$; declared unit & 2400 & 600 \\
33 & Eng./computing & output shape; finite; $0\le y\le 1300$ kWh and a scale guard against unit errors; declared unit & 1800 & 300 \\
34 & Phys./Earth & output shape; finite; probabilities in $[0,1]$ & 1800 & 300 \\
35 & Social/behav. & output shape; finite; $y\ge 0$ with a scale guard; declared unit & 1800 & 300 \\
36 & Hum./law & output shape; finite; no silent estimate over a whole 1-s window & 2400 & 600 \\
37 & Phys./Earth & output shape; finite; 150 to 350 K (catches degrees Celsius); unit K & 1800 & 300 \\
38 & Social/behav. & output shape; finite; range by series kind; declared unit & 1800 & 300 \\
39 & Social/behav. & values in $\{-1,0,1\}$ on every target cell; query budget from signed receipts & 1800 & 300 \\
40 & Eng./computing & score vector of the right length; finite & 1800 & 300 \\
41 & Phys./Earth & output structure; finite; positive standard deviations; physical ranges (catches kelvin) & 1800 & 300 \\
42 & Life/health & output structure; binary values; not positional only; hidden causal-prefix probe; declared unit & 1800 & 300 \\
43 & Hum./law & output length; strings; length within 0.5 to 2 times the OCR input & 3600 & 900 \\
44 & Social/behav. & output shape; finite; $|\hat\tau|\le 5\,\mathrm{sd}(y)$ & 1800 & 600 \\
45 & Hum./law & output length; non-empty captions of at most 1000 characters & 1800 & 600 \\
46 & Eng./computing & list of 16 strings; every problem implemented beyond a docstring & 1800 & 180 \\
47 & Hum./law & output length; one head and one relation per word; valid tree; 37 UD relations & 2400 & 600 \\
48 & Hum./law & output length; one finite score in $[0,1]$ per span; valid labels & 2400 & 600 \\
49 & Eng./computing & list of 16 strings; labels in $\{\text{sat},\text{unsat},\text{unknown}\}$ & 2400 & 900 \\
50 & Hum./law & output shape $16\times 20$; binary entries & 1800 & 600 \\
51 & Phys./Earth & output shape; finite; frequencies in 10 to 5000 cm$^{-1}$; unit conversion edge required & 1800 & 300 \\
52 & Social/behav. & list of 16 strings; each label among the item's options & 1800 & 300 \\
\bottomrule
\end{tabular}
\end{table*}

\subsection{Task Cards}
\label{app:bench-cards}

Table~\ref{tab:bench-cards} states, for each discipline, the data a policy may inspect, the extra
tools of the episode, the domain toolkit that the library can draw on, and the form of the
required output $y$. Every episode exposes the same four generic tools: \texttt{load\_eval\_inputs}
returns the 16 evaluation items without labels, \texttt{load\_train} returns labelled examples
drawn from the source role of the episode, \texttt{load\_dev\_inputs} returns further unlabelled
items, and \texttt{score\_dev} returns the visible development signal and the score of the
reference predictor on it. The columns of the table list only what is specific to a discipline.
The toolkit entries are modules of the shared library whose description is appended to the
episode objective.

\begin{table*}[t]
\centering
\caption{Task cards. The generic tools are listed in the text. The toolkit column names the module
of the shared scientific library available to the task. The last column gives the structure of the required
output $y$ for the 16 items of an episode ($k$ is the number of items, $n$ the length of an item).}
\label{tab:bench-cards}
\fontsize{7.2}{8.4}\selectfont
\setlength{\tabcolsep}{3pt}
\renewcommand{\arraystretch}{1.06}
\begin{tabular}{@{}l L{78pt} L{108pt} L{84pt} L{190pt}@{}}
\toprule
FoR & Discipline & Extra tools & Toolkit & Required output $y$ \\
\midrule
30 & Agricultural, veterinary and food & none & \texttt{phenoseg} & dict of three integer arrays $(k,512,512)$: semantic labels $\{0,1,2\}$, plant instances, leaf instances \\
31 & Biological & \texttt{substitution\_features} & \texttt{proteinfit} & one float per query row, higher means fitter \\
32 & Biomedical and clinical & none & \texttt{hippo} & list of integer arrays with the shape of each volume, values in $\{0,1,2\}$ \\
33 & Built environment & \texttt{load\_history} (own building) & \texttt{loadforecast}, \texttt{tsfm} & array $(k,24)$ of hourly energy in kWh \\
34 & Chemical & \texttt{featurize\_molecules} & \texttt{molecules} & vector of $k$ probabilities \\
35 & Commerce and services & none & \texttt{forecast}, \texttt{tsfm} & array $(k,24)$ in the unit of the series \\
36 & Creative arts and writing & none & \texttt{audiosep} & array $(k,4,132300,2)$ of stem estimates \\
37 & Earth & none & \texttt{weather} & array $(k,64,32)$ of temperatures in K on a longitude--latitude grid \\
38 & Economics & none (anonymized series) & \texttt{macro} & array $(k,4)$ of levels at offsets 1 to 4 \\
39 & Education & \texttt{query\_answers} (budgeted, signed receipts) & \texttt{adaptive} & integer array $(k,948)$ with values $\{-1,0,1\}$ on every target cell \\
40 & Engineering & \texttt{log\_mel\_spectrogram} & \texttt{anomsound}, \texttt{audioenc} & one anomaly score per clip \\
41 & Environmental & none & \texttt{aquatics}, \texttt{tsfm} & dict with oxygen and temperature means and standard deviations, each $(k,30)$ \\
42 & Health & none & \texttt{sepsis} & list of binary arrays, one value per hour of each stay \\
43 & History and archaeology & \texttt{text\_mer} & \texttt{ocrfix} & list of $k$ corrected transcriptions \\
44 & Human society & \texttt{load\_covariates} & \texttt{causal} & one effect estimate per dataset \\
45 & Indigenous studies & \texttt{image\_features}, \texttt{image\_knn\_reference} & \texttt{captions}, \texttt{clip\_retrieval} & list of captions in the language of each image (16 items per episode) \\
46 & Computing & \texttt{run\_visible\_tests} & \texttt{codegen} & list of $k$ program strings \\
47 & Language and culture & \texttt{read\_conllu} & \texttt{udparse} & per sentence, a head and a relation for each word \\
48 & Law & \texttt{load\_hypotheses} & \texttt{contracts}, \texttt{textenc} & per pair, an entailment label and one score in $[0,1]$ per candidate span \\
49 & Mathematical & \texttt{z3\_check} (timeouts, status) & \texttt{logic} & list of $k$ labels in $\{\text{sat},\text{unsat},\text{unknown}\}$ \\
50 & Philosophy and religion & \texttt{load\_value\_taxonomy} & \texttt{valueeval} & binary matrix $(k,20)$ \\
51 & Physical & \texttt{featurize\_structures} & \texttt{matphonon} & one frequency in cm$^{-1}$ per crystal \\
52 & Psychology & none & \texttt{psych} & list of $k$ labels, each among the options of the item \\
\bottomrule
\end{tabular}
\end{table*}

Two design choices apply across the cards. First, the data a policy sees are those of the source
role: the held-out evaluation items never appear in \texttt{load\_train}. Second, tool outputs
that would reveal the answer are removed or anonymized, for example the series identity of FoR38 and the
target answers of FoR39, which the query tool never reveals.

\subsection{Leakage Controls}
\label{app:bench-leak}

\paragraph{Visible data.}
The policy sees an opaque handle and the inputs of each item, never a label,
date, or lineage identifier. Training and development data come from the source
roles only and are disjoint from every evaluated item; the visible development
score is computed on a development slice carved from visible data.

\paragraph{Items scored together.}
Items of one episode are scored together and all of their inputs are visible, so
an item's input must not be a near-future observation of another item's label.
Forecasting adapters enforce a minimum spacing when an episode is drawn: FoR37
episodes are lanes of 16 initializations at least 120 hours apart, and FoR33
windows of one building are at least 144 hours apart.

\paragraph{Causality and queries.}
FoR42 checks causality with a hidden probe on prefixes of the stay in addition
to a visible check against positional-only outputs, and FoR39 takes the query
budget from signed receipts of the whole solve, failing closed when they are
missing. In FoR38 economy identifiers, series kinds, and periods are opaque.

\paragraph{Sandbox.}
Generated code is scanned and runs in a separate process with a scrubbed
environment (Appendix~\ref{app:impl-runtime}). Hidden labels live inside the
evaluator and never in the run directory.

\section{Human Audit of Evaluators}
\label{app:audit}

Each discipline was assigned three PhD reviewers, and one expert from the natural or social
sciences audited every evaluator and its reference executions across disciplines. The reviewers
read the task card and the evaluator, ran the reference predictor, and checked the metric
direction and unit, the acceptance margin, the hard constraints, and the leakage controls
(Appendix~\ref{app:bench-leak}). When two reviewers disagreed, the cross-field expert read both
reasons and the code involved and decided. A defect was fixed in the adapter or in the acceptance
rule, and an unfixable property of a source benchmark was recorded in the task card.

\paragraph{Verdict agreement.}
In four disciplines (FoR31, 38, 46, and 48), the PhD reviewers also judged ten run records on source tasks each (five evaluator passes and five fails), and
Table~\ref{tab:audit-agree} compares their verdicts with those of the evaluator. The two agree on 37 of the
40 records (92.5\%), with no missed detection (an evaluator pass that the reviewers reject) and three false
rejections (an evaluator fail that the reviewers accept), and each false rejection lies at the acceptance
boundary of its metric. Tables~\ref{tab:audit-verdicts-a} and~\ref{tab:audit-verdicts-b} list every verdict
with its reason.

\begin{table}[t]
\centering
\caption{Agreement between the evaluator and the PhD reviewers on ten run records on source tasks per discipline (five evaluator passes and five fails).}
\label{tab:audit-agree}
\footnotesize
\setlength{\tabcolsep}{3pt}
\renewcommand{\arraystretch}{1.08}
\begin{tabular}{@{}lrrrrrr@{}}
\toprule
Discipline & FoR & Records & Agree & Rate & Missed & False rej. \\
\midrule
Biological & 31 & 10 & 9 & 90\% & 0 & 1 \\
Computing & 46 & 10 & 10 & 100\% & 0 & 0 \\
Law & 48 & 10 & 9 & 90\% & 0 & 1 \\
Economics & 38 & 10 & 9 & 90\% & 0 & 1 \\
\midrule
Total & & 40 & 37 & 92.5\% & 0 & 3 \\
\bottomrule
\end{tabular}
\end{table}

\paragraph{Readings fixed in advance.}
Several task statements admit more than one reading. The hard constraints of
Table~\ref{tab:bench-constraints} fix such readings before any run, so that an output that depends on
a silent choice is rejected instead of scored. Table~\ref{tab:audit-readings} lists typical cases.

\begin{table}[t]
\centering
\caption{Typical ambiguities and the hard constraint that resolves each.}
\label{tab:audit-readings}
\footnotesize
\setlength{\tabcolsep}{3pt}
\renewcommand{\arraystretch}{1.08}
\begin{tabular}{@{}l L{74pt} L{104pt}@{}}
\toprule
FoR & Ambiguity & Resolution \\
\midrule
37 & kelvin or degrees Celsius & values must lie in 150 to 350 K \\
41 & kelvin or degrees Celsius & physical ranges with positive standard deviations \\
33 & kWh or another energy unit & $0\le y\le 1300$ kWh with a scale guard \\
35 & unit and scale of a series & $y\ge 0$ with a scale guard \\
36 & whether a silent segment is estimated & no silent estimate over a whole 1-s window \\
42 & whether an answer may follow position only & positional-only answers are rejected \\
46 & whether a stub counts as an implementation & every problem implemented beyond a docstring \\
51 & wavenumber unit & frequencies in 10 to 5000 cm$^{-1}$ \\
\bottomrule
\end{tabular}
\end{table}

\begin{table*}[t]
\centering
\caption{Verdicts of the evaluator and the PhD reviewers on the run records on source tasks of FoR31 and FoR46.}
\label{tab:audit-verdicts-a}
\footnotesize
\setlength{\tabcolsep}{3pt}
\renewcommand{\arraystretch}{1.14}
\begin{tabular}{@{}l l l l l L{300pt}@{}}
\toprule
Record & Evaluator & Reviewer & Agree & Type & Reason \\
\midrule
\multicolumn{6}{@{}l}{\textit{Biological (FoR31)}} \\
FoR31-01 & Pass & Pass & yes & -- & Output shape, finiteness, and units are valid, and the metric reaches the threshold. \\
FoR31-02 & Pass & Pass & yes & -- & Uses the visible training data, with a sound procedure and no hidden labels. \\
FoR31-03 & Pass & Pass & yes & -- & Predictions are complete, and the method follows protein-prediction practice. \\
FoR31-04 & Pass & Pass & yes & -- & Output length is correct and the values are stable. \\
FoR31-05 & Pass & Pass & yes & -- & Legal tool calls, and the Spearman metric meets the requirement. \\
FoR31-06 & Fail & Fail & yes & -- & The metric does not reach the reference plus the threshold. \\
FoR31-07 & Fail & Fail & yes & -- & The output contains NaN or Inf. \\
FoR31-08 & Fail & Pass & no & false rejection & Sound output and procedure; the metric is at the acceptance boundary, where the evaluator is stricter. \\
FoR31-09 & Fail & Fail & yes & -- & Output length differs from the number of targets. \\
FoR31-10 & Fail & Fail & yes & -- & Constant prediction, with insufficient metric and domain plausibility. \\
\midrule
\multicolumn{6}{@{}l}{\textit{Computing (FoR46)}} \\
FoR46-01 & Pass & Pass & yes & -- & All 16 functions are implemented and the visible tests pass. \\
FoR46-02 & Pass & Pass & yes & -- & Function signatures and return formats are correct. \\
FoR46-03 & Pass & Pass & yes & -- & Boundary conditions are handled sensibly, and no hidden answers are used. \\
FoR46-04 & Pass & Pass & yes & -- & The code logic is complete and the result is reproducible. \\
FoR46-05 & Pass & Pass & yes & -- & The visible tests and the task requirements are both satisfied. \\
FoR46-06 & Fail & Fail & yes & -- & A function is left unimplemented. \\
FoR46-07 & Fail & Fail & yes & -- & The return type is wrong, so the task interface cannot be passed. \\
FoR46-08 & Fail & Fail & yes & -- & Some function signatures do not match. \\
FoR46-09 & Fail & Fail & yes & -- & A clear semantic error. \\
FoR46-10 & Fail & Fail & yes & -- & Depends on the external network, so independent reproduction cannot be guaranteed. \\
\bottomrule
\end{tabular}
\end{table*}

\begin{table*}[t]
\centering
\caption{Verdicts of the evaluator and the PhD reviewers on the run records on source tasks of FoR48 and FoR38.}
\label{tab:audit-verdicts-b}
\footnotesize
\setlength{\tabcolsep}{3pt}
\renewcommand{\arraystretch}{1.14}
\begin{tabular}{@{}l l l l l L{300pt}@{}}
\toprule
Record & Evaluator & Reviewer & Agree & Type & Reason \\
\midrule
\multicolumn{6}{@{}l}{\textit{Law (FoR48)}} \\
FoR48-01 & Pass & Pass & yes & -- & Score range, length, and labels are all valid. \\
FoR48-02 & Pass & Pass & yes & -- & Evidence spans correspond to the contract text. \\
FoR48-03 & Pass & Pass & yes & -- & The legal interpretation agrees with the evidence. \\
FoR48-04 & Pass & Pass & yes & -- & The mAP reaches the acceptance threshold. \\
FoR48-05 & Pass & Pass & yes & -- & The output format is complete and no hidden information is used. \\
FoR48-06 & Fail & Fail & yes & -- & Some contract spans have no score. \\
FoR48-07 & Fail & Fail & yes & -- & A score falls outside $[0,1]$. \\
FoR48-08 & Fail & Fail & yes & -- & The evidence ranking and the legal conclusion do not match. \\
FoR48-09 & Fail & Pass & no & false rejection & Sound binary output and evidence; the mAP misses the threshold only through a boundary difference. \\
FoR48-10 & Fail & Fail & yes & -- & The legal polarity is wrong. \\
\midrule
\multicolumn{6}{@{}l}{\textit{Economics (FoR38)}} \\
FoR38-01 & Pass & Pass & yes & -- & Output shape, units, and value range are correct. \\
FoR38-02 & Pass & Pass & yes & -- & Only data before the forecast origin is used. \\
FoR38-03 & Pass & Pass & yes & -- & The macroeconomic variables relate sensibly and the metric is met. \\
FoR38-04 & Pass & Pass & yes & -- & The procedure is complete and the result is reproducible. \\
FoR38-05 & Pass & Pass & yes & -- & No prediction is clearly out of range. \\
FoR38-06 & Fail & Fail & yes & -- & The sMAPE does not meet the requirement. \\
FoR38-07 & Fail & Fail & yes & -- & The GDP or unemployment forecast is out of range. \\
FoR38-08 & Fail & Pass & no & false rejection & The output is valid, and the metric exceeds the threshold only slightly because of rounding. \\
FoR38-09 & Fail & Fail & yes & -- & The number of forecast steps or variables is wrong. \\
FoR38-10 & Fail & Fail & yes & -- & Uses future information or a physically implausible negative value. \\
\bottomrule
\end{tabular}
\end{table*}

\section{Case Studies}
\label{app:case}

This appendix reports two disciplines of different families and data types: Commerce and services
(FoR35, time series, MASE) and Law (FoR48, text, mAP). The task cards, reference predictors, acceptance
rules, and hard constraints are in Appendix~\ref{app:bench-spec}.

\subsection{Commerce and Services (FoR35)}
\label{app:case-35}

The policy forecasts the next 24 months of 16 tourism series in the native unit of each series. The reference
predictor is the seasonal-naive forecast, and an instance is accepted if its MASE is at most $0.95$ times
that of the reference. The hard constraints require a finite, non-negative forecast of the right shape, with a
scale guard against unit errors. Table~\ref{tab:case-scores} lists the held-out MASE of every system, and
ScienceClaw is lowest on both splits.

\subsection{Law (FoR48)}
\label{app:case-48}

The policy scores every span of a contract against a hypothesis. The reference predictor combines lexical overlap
with the majority label, and an instance is accepted if its mAP exceeds the reference by $0.10$. The hard
constraints require one finite score in $[0,1]$ per span and valid labels. Table~\ref{tab:case-scores} lists
the held-out mAP of every system, and ScienceClaw is highest on both splits.

\begin{table}[t]
\centering
\caption{Held-out scores of the two case disciplines. MASE is lower-is-better and mAP is higher-is-better.}
\label{tab:case-scores}
\fontsize{7.8}{9.2}\selectfont
\setlength{\tabcolsep}{3pt}
\renewcommand{\arraystretch}{1.08}
\begin{tabular}{@{}lcccc@{}}
\toprule
 & \multicolumn{2}{c}{FoR35 (MASE $\downarrow$)} & \multicolumn{2}{c}{FoR48 (mAP $\uparrow$)} \\
\cmidrule(lr){2-3}\cmidrule(lr){4-5}
System & IID & OOD & IID & OOD \\
\midrule
Frozen & 1.6009 & 1.8740 & 0.7513 & 0.6918 \\
SkillOpt & 1.5268 & 1.6867 & 0.8158 & 0.7652 \\
TTE & 1.5671 & 1.7502 & 0.8074 & 0.7706 \\
AHE & 1.5227 & 1.7066 & 0.8211 & 0.7720 \\
EvoMaster & 1.4984 & 1.6930 & 0.8292 & 0.7863 \\
EvoScientist & 1.5281 & 1.6846 & 0.8170 & 0.7642 \\
ScienceClaw & 1.4653 & 1.5986 & 0.8554 & 0.8289 \\
\bottomrule
\end{tabular}
\end{table}

Table~\ref{tab:case-mech} gives the gain over the frozen agent of the fixed-workflow and unlinked variants.
On both disciplines and both splits the linked system gains more than either.

\begin{table}[t]
\centering
\caption{Gain over the frozen agent (\%) of the linkage variants on the two case disciplines.}
\label{tab:case-mech}
\fontsize{7.8}{9.2}\selectfont
\setlength{\tabcolsep}{3pt}
\renewcommand{\arraystretch}{1.08}
\begin{tabular}{@{}lcccc@{}}
\toprule
 & \multicolumn{2}{c}{FoR35} & \multicolumn{2}{c}{FoR48} \\
\cmidrule(lr){2-3}\cmidrule(lr){4-5}
Variant & IID & OOD & IID & OOD \\
\midrule
Fixed workflow & 3.77 & 5.75 & 1.07 & 6.20 \\
w/o linked evolution & 4.86 & 9.27 & 10.63 & 13.12 \\
ScienceClaw (full) & 8.47 & 14.69 & 13.86 & 19.82 \\
\bottomrule
\end{tabular}
\end{table}

\section{Scope, Limits, and Further Analyses}
\label{app:clarify}

This appendix adds notes on the benchmark, the gate, the foundation model, the instance counts, the notation, and the compute.

\subsection{Scope of ScienceClaw-Eval}
\label{app:clarify-scope}

A task of ScienceClaw-Eval is a computational task whose output can be scored by an executable
evaluator inside a sandbox: a structured prediction, a forecast, a segmentation, a program, a parse, or a
set of labels, together with the trace that produced it. The 23 tasks span the 23 disciplines of the field-of-research classification used in the main text, and each is built from a published benchmark with an official or reconstructed scorer (Table~\ref{tab:bench-spec}). A task enters the framework through $D_t=(D_{T,t},D_{V,t},D_{E,t})$ and requires an evaluator that checks a replayed output against $D_{V,t}$.

\subsection{What the Evaluators Check}
\label{app:clarify-evaluator}

An evaluator decides acceptance from the metric of an instance against a reference predictor, from
hard constraints on the output, and from the failure semantics of the run. Three layers narrow the gap between acceptance and
validity. The hard constraints catch the silent errors that a metric hides, such as a unit error, a
forecast in the wrong scale, or a positional answer (Table~\ref{tab:bench-constraints}). The leakage controls
remove tool outputs that would reveal an answer (Appendix~\ref{app:bench-leak}). The acceptance rules
require a policy to exceed a reference predictor by a stated margin (Table~\ref{tab:bench-spec}). The human audit of the evaluators is described in Appendix~\ref{app:audit}.

\subsection{Selectivity and Progress of the Gate}
\label{app:clarify-gate}

\looseness=-1 The gate is strict by design (Appendix~\ref{app:theory-guard}). Of the 161 candidates of ScienceClaw, 70
were promoted ($43.48\%$) and 91 were rejected. Table~\ref{tab:clarify-rounds} lists the change of the OOD
macro score per round. The score rises in six of seven rounds and falls only in round 3, by $0.20$ pp.

\begin{table}[!htbp]
\centering
\caption{Change of the OOD macro score of ScienceClaw per round (pp) with the cumulative score (\%).}
\label{tab:clarify-rounds}
\fontsize{7.8}{9.2}\selectfont
\setlength{\tabcolsep}{3pt}
\renewcommand{\arraystretch}{1.08}
\begin{tabular}{@{}lccccccc@{}}
\toprule
Round & 1 & 2 & 3 & 4 & 5 & 6 & 7 \\
\midrule
Change (pp) & $+2.72$ & $+2.24$ & $-0.20$ & $+3.67$ & $+1.43$ & $+2.24$ & $+1.49$ \\
Score (\%) & 80.43 & 82.68 & 82.47 & 86.14 & 87.57 & 89.81 & 91.30 \\
\bottomrule
\end{tabular}
\end{table}

In the seven rounds of the benchmark,
each of the last three rounds still adds between $1.43$ and $2.24$ pp. The
components of the gate matter for the final score. Removing source replay, the scientific constraints, or
the independent validator lowers the score of the full system by $6.18$, $7.33$, and $5.68$ percent
on average over the two splits, and removing the IID selection by $11.64$ percent, against $13.68$ percent for
the frozen agent (Table~\ref{tab:app-retained}).

\subsection{Foundation Model and Zero-Shot Performance}
\label{app:clarify-model}

Every system of the benchmark uses the same foundation model for orchestration: Qwen3.5-27B with open
weights, served by a self-hosted inference server with thinking disabled and temperature 0
(Table~\ref{tab:impl-models}). The zero-shot reference is the Frozen agent, which runs the initial
program $A_0$ with this model. Its OOD macro score is $77.72\%$ (Table~\ref{tab:app-evolution}), and
ScienceClaw reaches $91.30\%$ after seven rounds.

Evolution uses GPT-5.4 in the main configuration
(Table~\ref{tab:impl-models}). The Qwen-only variant runs evolution on the orchestrator
model instead. It retains $43.8\%$ of the gain of the full system (Table~\ref{tab:app-retained}), and it is better
than the frozen agent in 44 of the 46 discipline--split cells; the exceptions are FoR40 and FoR44 on OOD. A
stronger evolution model therefore raises the quality of the
candidates, while the gate and the replay set are the same for both models. Every baseline also runs its
evolution calls on GPT-5.4 (Appendix~\ref{app:impl-baselines}), so the Qwen-only variant is the one comparison in
which the evolution model differs from that of the other evolving methods.

\subsection{Instance Counts and Sampling Error}
\label{app:clarify-sampling}

Every discipline contributes 64 instances to each of the ID and OOD splits (16 items in each of four
episodes; Table~\ref{tab:bench-splits}). The MacroSR over the 23 disciplines is a mean over these
instances, and Proposition~\ref{prop:se} bounds its standard error by $1.30$ pp at $n=64$ per discipline, and the
standard error of a difference of two methods on independent draws by $1.84$ pp.
The margins of ScienceClaw over the strongest baseline in a discipline range from $1.82$ to $3.73$ percent
(ID) and from $3.05$ to $6.21$ percent (OOD) on relative scale (Table~\ref{tab:app-margin}), and the margin
over the frozen agent ranges from $7.96$ to $16.45$ (ID) and from $11.57$ to $23.34$ (OOD).

\subsection{Reading the Notation}
\label{app:clarify-notation}

Table~\ref{tab:clarify-notation} gives, for the equations of Sections~\ref{sec:task} and~\ref{sec:method},
the plain-language reading and the object that realizes the symbol in the implementation
(Table~\ref{tab:impl-map}).

\begin{table*}[!t]
\centering
\caption{Reading guide to the equations of the method.}
\label{tab:clarify-notation}
\fontsize{7.4}{8.8}\selectfont
\setlength{\tabcolsep}{3pt}
\renewcommand{\arraystretch}{1.08}
\begin{tabular}{@{}L{54pt} L{196pt} L{196pt}@{}}
\toprule
Equation & Reading & Realized by \\
\midrule
Eq.~(\ref{eq:task-solving}) & The fixed model solves task $t$ with program $A_r$ and returns the final workflow, the output, and the trace & one episode of the solver; the three results are saved with the run \\
Eq.~(\ref{eq:typed-workflow}) & A workflow is a graph of Operators and tools whose connected ports must agree in type, shape, and unit & the workflow graph and its schema check at every edge \\
Eq.~(\ref{eq:edit-policy}) & At each step the policy proposes one atomic edit and names the Skills and Operators that contributed & the policy call that returns a JSON edit \\
Eq.~(\ref{eq:workflow-transition}) & Applying an edit reruns only the affected nodes and records diagnostics & the executor with checkpoints \\
Eq.~(\ref{eq:execution-evidence}) & Evidence is regenerated from a reset environment and scored by the evaluator & the replay and the evaluator of the task \\
Eqs.~(\ref{eq:repair-interval}) and (\ref{eq:repair-attribution}) & A repair is the sequence of edits between the last failure and the first success, split into strategy edits and executable edits & the evolution engine \\
Eqs.~(\ref{eq:skill-candidate}) and (\ref{eq:operator-candidate}) & A Skill candidate rewrites attributed instructions, and an Operator candidate packages a connected executable subgraph with its interface and conditions & the evolution engine and the boundary replay \\
Eq.~(\ref{eq:candidate-validation}) & A candidate must reproduce its source repair and be used where expected before it enters the pool & the source reproduction check \\
Eqs.~(\ref{eq:candidate-selection}) and (\ref{eq:program-update}) & Among the feasible candidates choose the one with the best validation score, and update only if it strictly beats the current program & the gate with its noise guard and budgets (Appendix~\ref{app:impl-gate}) \\
\bottomrule
\end{tabular}
\end{table*}

\subsection{Program-Level Self-Evolution and Genetic Programming}
\label{app:clarify-gp}

ScienceClaw evolves the persistent program of an agent, and genetic programming evolves a population of
programs, so the two share the idea of improving executable artifacts by selection. Table~\ref{tab:clarify-gp}
lists the differences that matter for the claims of the paper.

\begin{table}[!htbp]
\centering
\caption{Program-level self-evolution compared with genetic programming.}
\label{tab:clarify-gp}
\fontsize{7.6}{9}\selectfont
\setlength{\tabcolsep}{3pt}
\renewcommand{\arraystretch}{1.08}
\begin{tabular}{@{}L{44pt} L{82pt} L{82pt}@{}}
\toprule
 & Genetic programming & ScienceClaw \\
\midrule
Unit of evolution & a population of whole programs & one lineage $A_r=(\mathcal S_r,\mathcal O_r)$ of Skills and Operators \\
Variation & random crossover and mutation & repairs reproduced by replay and abstracted into candidates \\
Source of change & the search operators & the execution traces of solved tasks \\
Selection & fitness on the training cases & strict gate on validation tasks, replay set, and hard constraints \\
Persistence & the best individual of the run & a versioned program with rollback \\
Parameters of the model & not applicable & fixed, $\Theta_r=\Theta_0$ \\
\bottomrule
\end{tabular}
\end{table}

\subsection{Compute and Cost}
\label{app:clarify-cost}

A full run serves Qwen3.5-27B from self-hosted servers on H800 GPUs, and no model is trained. Cost per OOD gain in Table~\ref{tab:app-reliability} is normalized to ScienceClaw under the protocol, which fixes the budgets of all methods (Appendix~\ref{app:impl-gate}). Per task, averaged over the IID and OOD instances, ScienceClaw makes 6.54 model API calls, Operator calls included (8.98 counting cache hits), and uses 3,155 uncached tokens, the newly processed tokens that normally count toward API usage (4,378 with cached tokens, the full logical context the model processes), 10.2 s of wall time, and 0.0113 GPU-hours on four GPUs.

\section{Complete Results}
\label{app:data}

This appendix reports the per-discipline data behind the figures and tables of
Section~\ref{sec:experiments}: relative margins, all eighteen variants on both
splits, and the continual-evolution, transfer, reliability, and workflow
measurements. Each discipline contributes 64 IID and 64 OOD instances to every
number in this appendix.

\subsection{Relative Margins per Discipline}
\label{app:data-margin}

Table~\ref{tab:app-margin} gives the relative margin of ScienceClaw over the
frozen agent and over the strongest baseline in every discipline and split,
computed from the scores shown in Table~\ref{tab:scienceclaw-23discipline-iid-ood}.
ScienceClaw leads the frozen agent in all 23 disciplines on both splits, and its
margin over the strongest baseline is positive in every discipline: it ranges
from 1.82\% to 3.73\% on IID and from 3.05\% to 6.21\% on OOD. EvoMaster or
EvoScientist is the strongest baseline in most rows, which matches the order
of the baselines by how much of the program they evolve
(Section~\ref{sec:experiments}).

\subsection{Ablation Results on Both Splits}
\label{app:data-ablation}

Table~\ref{tab:app-retained} summarizes every variant by its relative drop below
the full system and by the share of the full system's gain over the frozen agent
that it retains. Figures~\ref{fig:app-ablation-iid} and~\ref{fig:app-ablation-ood}
show the task-native scores of all eighteen variants in each discipline, shaded
within each discipline from the worst to the best variant, and
Tables~\ref{tab:app-ablation-1}--\ref{tab:app-ablation-3} list the same scores
with the best and second best marked. The full system is the best variant in
all 23 disciplines on both IID and OOD. The loss of every ablated variant is
larger on OOD than on IID, in particular for the gate components: removing
source replay, the scientific constraints, or the independent validator
increases the drop 2.4 to 2.8 times from IID to OOD, and removing OOD transfer,
which targets the OOD split, leaves IID nearly unchanged
(a drop of 0.91\%) while raising the OOD drop to 13.09\%.

\begin{table}[!htbp]
\centering
\caption{Relative drop below the full system (\%, mean over 23 disciplines) and retained share of the gain. The drop is $(s_{\mathrm{full}}-s)/s_{\mathrm{full}}$ for higher-is-better and $(s-s_{\mathrm{full}})/s_{\mathrm{full}}$ for lower-is-better metrics. The retained share is $1-\mathrm{drop}/\mathrm{drop}_{\mathrm{Frozen}}$ of the mean of the two splits; negative values mean that the variant is worse than the frozen agent.}\label{tab:app-retained}
\fontsize{7.4}{8.7}\selectfont
\setlength{\tabcolsep}{3pt}
\renewcommand{\arraystretch}{1.22}
\begin{tabular}{@{}L{86pt}rrrr@{}}
\toprule
Variant & IID & OOD & Mean & Retained (\%) \\
\midrule
\multicolumn{5}{@{}l}{\textit{Update scope}} \\
Workflow only & 9.42 & 15.16 & 12.29 & 10.2 \\
w/o Skill evolution & 3.46 & 5.41 & 4.44 & 67.6 \\
w/o Operator evolution & 4.43 & 6.87 & 5.65 & 58.7 \\
w/o Workflow evolution & 2.95 & 5.82 & 4.39 & 67.9 \\
\midrule
\multicolumn{5}{@{}l}{\textit{Linkage and candidate source}} \\
w/o linked evolution & 3.99 & 7.01 & 5.50 & 59.8 \\
RuleEvo & 6.92 & 11.62 & 9.27 & 32.3 \\
Qwen-only & 6.01 & 9.39 & 7.70 & 43.8 \\
\midrule
\multicolumn{5}{@{}l}{\textit{Gate and memory}} \\
w/o source replay & 3.25 & 9.10 & 6.18 & 54.9 \\
w/o scientific constraints & 4.19 & 10.48 & 7.33 & 46.4 \\
w/o independent validator & 3.33 & 8.03 & 5.68 & 58.5 \\
w/o cross-task memory & 2.30 & 7.14 & 4.72 & 65.5 \\
w/o IID selection & 9.27 & 14.00 & 11.64 & 14.9 \\
w/o OOD transfer & 0.91 & 13.09 & 7.00 & 48.9 \\
\midrule
\multicolumn{5}{@{}l}{\textit{Execution}} \\
Fixed workflow & 8.15 & 12.67 & 10.41 & 23.9 \\
Single operator & 12.52 & 17.05 & 14.78 & -8.1 \\
Single turn & 13.84 & 18.82 & 16.33 & -19.3 \\
\midrule
\multicolumn{5}{@{}l}{\textit{Reference}} \\
Frozen & 11.69 & 15.68 & 13.68 & 0.0 \\
\bottomrule
\end{tabular}
\end{table}

\subsection{Continual Evolution, Transfer, and Reliability}
\label{app:data-continual}

\paragraph{Definitions.}
The OOD macro score after round $r$ is MacroSR of the frozen snapshot $A_r$ on
$\mathcal D_{\mathrm{OOD}}$ (Eq.~(\ref{eq:macro-success-rate})), and the round-0
entry is the shared starting program. The gain is the round-7 score minus the
round-0 score. Of the 161 candidates a method submits (7 rounds $\times$ 23
disciplines), \emph{promoted} candidates pass the gate and \emph{rejected}
candidates do not, and the rate is the promoted share. Forward transfer is the
OOD gain on a target family after evolving on a source family, and the
diagonal is the within-family gain. For ScienceClaw, FWT is the mean of the 25 entries of Figure~\ref{fig:eval-module}b. Forgetting is measured on
$\mathcal D_{\mathrm{rep}}$ against the snapshot immediately after the source
task. The five families are Life/health (FoR30, 31, 32, 42), Physical/Earth
(FoR34, 37, 41, 51), Engineering/computing (FoR33, 40, 46, 49),
Social/behavioural (FoR35, 38, 39, 44, 52), and Humanities/law (FoR36, 43, 45,
47, 48, 50).
Backward transfer (BWT) is the mean change of the score on
$\mathcal D_{\mathrm{rep}}$ after later rounds, and forgetting is the average and
maximum decrease of that score. Negative transfer is the share of transfer
measurements with a negative gain. Erroneous promotion is the rate of promotions
that lower the retained score or violate a hard constraint, and Attain.\ is the
continual-target attainment score.

\clearpage
\begin{table*}[!t]
\centering
\captionof{table}{Relative margin of ScienceClaw (\%) in each discipline and split. The margin over the frozen agent is $(s_{\mathrm{SC}}-s_{0})/s_{0}$ for higher-is-better metrics and $(s_{0}-s_{\mathrm{SC}})/s_{0}$ for lower-is-better metrics. The strongest baseline is the best of SkillOpt, TTE, AHE, EvoMaster, and EvoScientist in that discipline and split.}\label{tab:app-margin}
\fontsize{7.6}{8.9}\selectfont
\setlength{\tabcolsep}{4pt}
\renewcommand{\arraystretch}{1.95}
\begin{tabular*}{0.99\textwidth}{@{}>{\raggedright\arraybackslash}p{96pt}@{\extracolsep{\fill}}rrrrrr@{}}
\toprule
 & \multicolumn{3}{c}{IID} & \multicolumn{3}{c}{OOD} \\
\cmidrule(lr){2-4}\cmidrule(lr){5-7}
Discipline (metric) & vs. Frozen & vs. strongest & Strongest & vs. Frozen & vs. strongest & Strongest \\
\midrule
\multicolumn{7}{@{}l}{\textbf{Higher is better ($\uparrow$)}} \\
Agricultural sci. (PQ+) & 10.77 & 2.38 & EvoMaster & 17.13 & 3.48 & EvoScientist \\
Biological sci. (Spearman) & 15.00 & 2.33 & EvoScientist & 13.07 & 3.19 & EvoMaster \\
Biomedical sci. (DSC) & 9.47 & 2.01 & EvoScientist & 16.05 & 4.43 & EvoMaster \\
Chemical sci. (ROC-AUC) & 15.75 & 2.74 & EvoMaster & 20.50 & 4.04 & EvoScientist \\
Creative arts (SDR (dB)) & 14.19 & 2.43 & EvoMaster & 18.06 & 3.20 & EvoScientist \\
Education (10-mask acc.) & 15.31 & 2.34 & EvoMaster & 14.27 & 5.28 & EvoMaster \\
Engineering (DCASE score) & 14.48 & 2.43 & EvoMaster & 12.99 & 5.01 & EvoScientist \\
Health sci. (Clin. utility) & 10.92 & 2.44 & EvoMaster & 14.07 & 4.96 & EvoScientist \\
Indigenous (chrF++) & 12.88 & 3.12 & EvoMaster & 16.45 & 6.21 & EvoMaster \\
Computing sci. (pass@1) & 14.29 & 1.82 & EvoMaster & 20.83 & 3.57 & EvoMaster \\
Language \& culture (LAS) & 11.67 & 3.73 & EvoMaster & 16.04 & 3.66 & EvoMaster \\
Law (mAP) & 13.86 & 3.16 & EvoMaster & 19.82 & 5.42 & EvoMaster \\
Mathematics (Oracle acc.) & 12.96 & 3.39 & EvoMaster & 19.15 & 3.70 & EvoMaster \\
Philosophy (F1) & 16.45 & 2.72 & EvoScientist & 23.34 & 3.95 & EvoMaster \\
Psychology (Micro acc.) & 8.70 & 2.20 & EvoScientist & 17.31 & 3.83 & EvoMaster \\
\midrule
\multicolumn{7}{@{}l}{\textbf{Lower is better ($\downarrow$)}} \\
Built env. (NRMSE \%) & 11.92 & 2.74 & EvoMaster & 15.71 & 3.56 & EvoMaster \\
Commerce (MASE) & 8.47 & 2.21 & EvoMaster & 14.69 & 5.10 & EvoScientist \\
Earth sci. (RMSE (K)) & 7.96 & 2.87 & EvoScientist & 15.70 & 4.24 & TTE \\
Economics (sMAPE (\%)) & 11.05 & 2.60 & EvoMaster & 14.51 & 4.16 & EvoMaster \\
Environmental sci. (CRPS) & 11.60 & 1.91 & EvoMaster & 11.57 & 3.25 & EvoMaster \\
History (cMER-micro) & 10.17 & 3.05 & AHE & 15.22 & 3.84 & EvoMaster \\
Human society (nRMSE) & 11.68 & 2.58 & EvoMaster & 15.04 & 5.27 & EvoScientist \\
Physical sci. (MAE) & 12.20 & 2.15 & EvoScientist & 16.89 & 3.05 & EvoMaster \\
\midrule
Mean over 23 disciplines & 12.25 & 2.58 &  & 16.45 & 4.19 &  \\
Minimum & 7.96 & 1.82 &  & 11.57 & 3.05 &  \\
Maximum & 16.45 & 3.73 &  & 23.34 & 6.21 &  \\
\bottomrule
\end{tabular*}
\end{table*}

\begin{figure*}[!p]
\centering
\includegraphics[width=\textwidth]{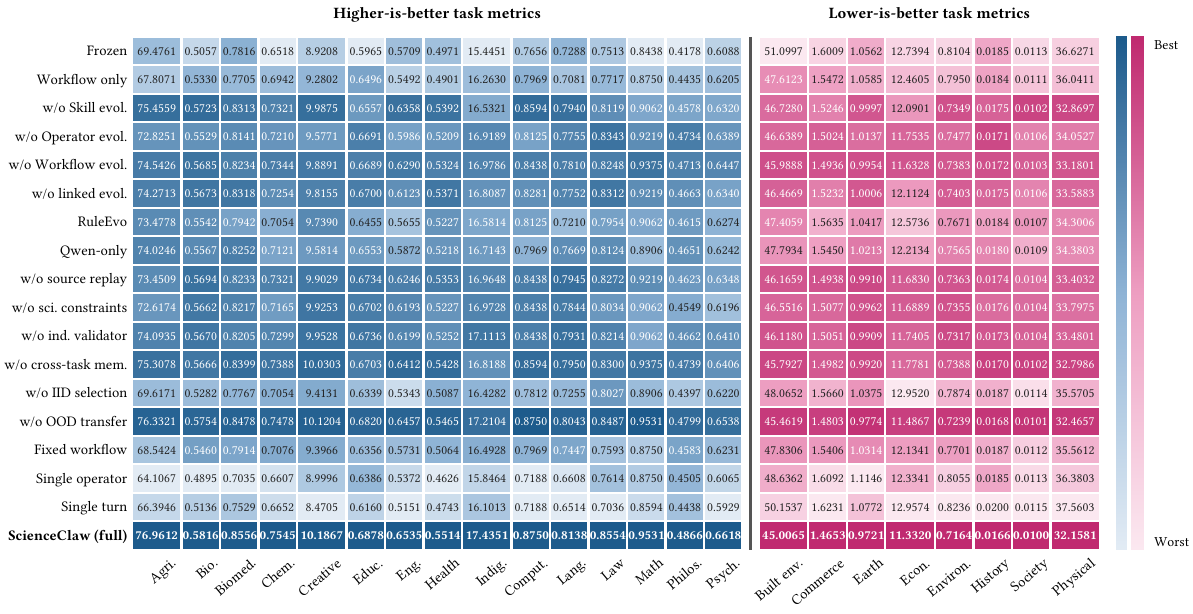}
\caption{IID scores of all eighteen variants in each discipline. Shading runs from the worst (light) to the best (dark) variant within a discipline, and the full system is in bold.}
\label{fig:app-ablation-iid}
\vspace{4pt}
\captionof{table}{Continual evolution of the OOD macro score (\%) over seven rounds with candidate counts over the 161 submitted candidates. The best entry in each column is in bold and the second best is underlined.}\label{tab:app-evolution}
\fontsize{7.6}{9}\selectfont
\setlength{\tabcolsep}{0pt}
\renewcommand{\arraystretch}{1.1}
\renewcommand{\arraystretch}{1.6}
\begin{tabular*}{0.99\textwidth}{@{}l*{8}{@{\hspace{2pt}\extracolsep{\fill}}c}@{\hspace{8pt}\extracolsep{\fill}}c@{\hspace{2pt}\extracolsep{\fill}}c@{\hspace{2pt}\extracolsep{\fill}}c@{\hspace{2pt}\extracolsep{\fill}}c@{}}
\toprule
 & \multicolumn{8}{c}{OOD macro score (\%) after round} & \multicolumn{4}{c}{Round-7 summary} \\
\cmidrule(lr){2-9}\cmidrule(lr){10-13}
Method & 0 & 1 & 2 & 3 & 4 & 5 & 6 & 7 & Gain (pp) & Prom. & Rej. & Rate (\%) \\
\midrule
Frozen & 77.72 & 77.72 & 77.72 & 77.72 & 77.72 & 77.72 & 77.72 & 77.72 & 0.00 & 0 & 0 & -- \\
RuleEvo & 77.72 & 78.40 & 79.21 & 78.87 & 79.62 & 80.37 & 80.03 & 81.32 & 3.60 & 37 & 124 & 22.98 \\
Qwen-only & 77.72 & 79.42 & 80.16 & 79.89 & 81.25 & 81.93 & 81.59 & 83.22 & 5.50 & 49 & 112 & 30.43 \\
EvoScientist & 77.72 & 79.69 & 81.18 & 82.34 & 82.07 & 84.04 & 85.19 & 86.35 & 8.63 & 56 & 105 & 34.78 \\
EvoMaster & 77.72 & \underline{80.03} & \underline{81.59} & \textbf{82.88} & \underline{82.54} & \underline{84.85} & \underline{86.07} & \underline{87.50} & \underline{9.78} & \underline{63} & 98 & \underline{39.13} \\
\midrule
ScienceClaw & 77.72 & \textbf{80.43} & \textbf{82.68} & \underline{82.47} & \textbf{86.14} & \textbf{87.57} & \textbf{89.81} & \textbf{91.30} & \textbf{13.59} & \textbf{70} & 91 & \textbf{43.48} \\
\bottomrule
\end{tabular*}
\vspace{8pt}
\begin{minipage}[t]{\textwidth}\centering
\captionof{table}{Workflow mechanics of the linked system and of the ablated executions. Planner rounds and distinct operators are per task, and the other columns are rates.}\label{tab:app-mechanics}
\fontsize{7.4}{8.7}\selectfont
\setlength{\tabcolsep}{2.4pt}
\renewcommand{\arraystretch}{1.1}
\renewcommand{\arraystretch}{1.6}
\begin{tabular}{@{}lcccccc@{}}
\toprule
System & \shortstack{Plan.\\rounds} & \shortstack{Dist.\\oper.} & \shortstack{Oper.\\divers.} & \shortstack{Feedb.\\repair (\%)} & \shortstack{Ckpt.\\recov.\ (\%)} & \shortstack{Clean\\replay (\%)} \\
\midrule
\textbf{Linked Full} & 4.84 & 4.24 & \textbf{0.78} & \textbf{71} & \textbf{89} & \textbf{96} \\
Unlinked & 4.63 & 3.96 & 0.71 & 62 & 82 & 89 \\
Fixed Workflow & 3.12 & 2.89 & 0.50 & 45 & 73 & 87 \\
Single Operator & 4.07 & 1.07 & 0.20 & 38 & 67 & 84 \\
Single Turn & 1.09 & 2.13 & 0.36 & 0 & 31 & 78 \\
\bottomrule
\end{tabular}
\end{minipage}
\end{figure*}

\begin{figure*}[!p]
\centering
\includegraphics[width=\textwidth]{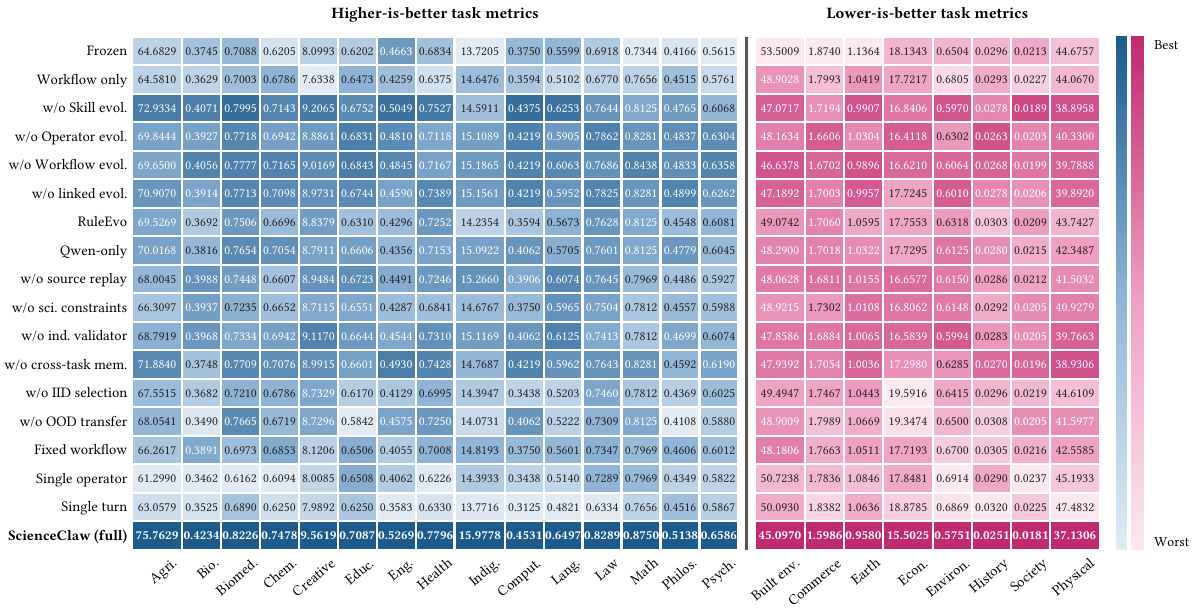}
\caption{OOD scores of all eighteen variants in each discipline, shaded as in Figure~\ref{fig:app-ablation-iid}.}
\label{fig:app-ablation-ood}
\vspace{4pt}
\begin{minipage}[t]{\textwidth}\centering
\captionof{table}{Reliability and efficiency. Cost per OOD gain is normalized to ScienceClaw ($=1$) and is undefined for the frozen agent. The best entry in each column is in bold and the second best is underlined.}\label{tab:app-reliability}
\footnotesize
\setlength{\tabcolsep}{2.4pt}
\renewcommand{\arraystretch}{1.48}
\begin{tabular*}{\textwidth}{@{}l@{\extracolsep{\fill}}rrrrr@{}}
\toprule
Method & \shortstack[r]{Target\\attain.\ $\uparrow$} & \shortstack[r]{Hard-constr.\\pass (\%) $\uparrow$} & \shortstack[r]{Erroneous\\prom.\ (\%) $\downarrow$} & \shortstack[r]{Cost per\\OOD gain $\downarrow$} & \shortstack[r]{Planner\\time (\%) $\downarrow$} \\
\midrule
Frozen & 88.57 & 91.51 & -- & -- & 21.89 \\
Qwen-only & 92.99 & 93.89 & 7.41 & 1.62 & 23.65 \\
SkillOpt & 92.67 & 93.14 & 6.52 & 2.50 & 20.56 \\
TTE & 93.63 & 92.53 & 7.25 & 1.83 & \underline{19.38} \\
AHE & 94.26 & 94.09 & 5.60 & 1.92 & 27.23 \\
EvoMaster & \underline{94.99} & \underline{95.65} & \underline{3.80} & \underline{1.37} & 25.16 \\
EvoScientist & 94.43 & 95.18 & 4.58 & 1.87 & 28.74 \\
\midrule
ScienceClaw & \textbf{96.42} & \textbf{98.23} & \textbf{1.23} & \textbf{1.00} & \textbf{18.40} \\
\bottomrule
\end{tabular*}
\end{minipage}\\[8pt]\begin{minipage}[t]{\textwidth}\centering
\captionof{table}{Transfer, forgetting, and reliability of the evolving methods. FWT and BWT are forward and backward transfer (pp), Neg.\ the share of negative transfer (\%), Forg.\ the average and maximum forgetting (pp), Attain.\ the target-attainment score, and Err.\ the erroneous-promotion rate (\%). The best entry in each column is in bold and the second best is underlined.}\label{tab:app-transfer-rel}
\footnotesize
\setlength{\tabcolsep}{2.4pt}
\renewcommand{\arraystretch}{1.48}
\begin{tabular*}{\textwidth}{@{}l@{\extracolsep{\fill}}ccccc@{}}
\toprule
Method & \shortstack{FWT/BWT\\$\uparrow$} & \shortstack{Neg.\\$\downarrow$} & \shortstack{Forg.\ $\downarrow$\\avg/max} & \shortstack{Attain.\\$\uparrow$} & \shortstack{Err.\\$\downarrow$} \\
\midrule
SkillOpt & 1.42/0.26 & 31.39 & 0.88/2.45 & 92.67 & 6.52 \\
TTE & 1.75/0.41 & 26.97 & 0.76/2.17 & 93.63 & 7.25 \\
EvoScientist & 2.74/1.23 & 17.32 & 0.42/1.23 & 94.43 & 4.58 \\
EvoMaster & \underline{3.19/1.53} & \underline{14.33} & \underline{0.33/1.05} & \underline{94.99} & \underline{3.80} \\
ScienceClaw & \textbf{4.20/3.66} & \textbf{11.89} & \textbf{0.13/0.51} & \textbf{96.42} & \textbf{1.23} \\
\bottomrule
\end{tabular*}
\end{minipage}
\end{figure*}

\begin{table*}[!p]
\centering
\caption{Per-discipline ablation results (1/3): update-scope variants. Each row gives the task-native score of one discipline on one split. The best score in each row is in bold and the second best is underlined, taken over all 18 variants.}
\label{tab:app-ablation-1}
\fontsize{7.4}{8.7}\selectfont
\setlength{\tabcolsep}{0pt}
\renewcommand{\arraystretch}{1.3}
\begin{tabular*}{0.99\textwidth}{@{}>{\raggedright\arraybackslash}p{88pt}@{\hspace{4pt}}l@{\extracolsep{\fill}}cccccc@{}}
\toprule
Discipline (metric) & Split & \multicolumn{1}{c}{\makebox[3.2em][c]{\fontsize{6.1}{6.9}\selectfont\shortstack[c]{Frozen}}} & \multicolumn{1}{c}{\makebox[3.2em][c]{\fontsize{6.1}{6.9}\selectfont\shortstack[c]{Workflow\\only}}} & \multicolumn{1}{c}{\makebox[3.2em][c]{\fontsize{6.1}{6.9}\selectfont\shortstack[c]{w/o Skill\\evol.}}} & \multicolumn{1}{c}{\makebox[3.2em][c]{\fontsize{6.1}{6.9}\selectfont\shortstack[c]{w/o Oper.\\evol.}}} & \multicolumn{1}{c}{\makebox[3.2em][c]{\fontsize{6.1}{6.9}\selectfont\shortstack[c]{w/o Work.\\evol.}}} & \multicolumn{1}{c}{\makebox[3.2em][c]{\fontsize{6.1}{6.9}\selectfont\shortstack[c]{ScienceClaw\\(full)}}} \\
\midrule
\multicolumn{8}{@{}l}{\textbf{Higher is better ($\uparrow$)}} \\
Agricultural sci. (PQ+) & IID & 69.4761 & 67.8071 & 75.4559 & 72.8251 & 74.5426 & \textbf{76.9612} \\
 & OOD & 64.6829 & 64.5810 & \underline{72.9334} & 69.8444 & 69.6500 & \textbf{75.7629} \\
Biological sci. (Spearman) & IID & 0.5057 & 0.5330 & 0.5723 & 0.5529 & 0.5685 & \textbf{0.5816} \\
 & OOD & 0.3745 & 0.3629 & \underline{0.4071} & 0.3927 & 0.4056 & \textbf{0.4234} \\
Biomedical sci. (DSC) & IID & 0.7816 & 0.7705 & 0.8313 & 0.8141 & 0.8234 & \textbf{0.8556} \\
 & OOD & 0.7088 & 0.7003 & \underline{0.7995} & 0.7718 & 0.7777 & \textbf{0.8226} \\
Chemical sci. (ROC-AUC) & IID & 0.6518 & 0.6942 & 0.7321 & 0.7210 & 0.7344 & \textbf{0.7545} \\
 & OOD & 0.6205 & 0.6786 & 0.7143 & 0.6942 & \underline{0.7165} & \textbf{0.7478} \\
Creative arts (SDR (dB)) & IID & 8.9208 & 9.2802 & 9.9875 & 9.5771 & 9.8891 & \textbf{10.1867} \\
 & OOD & 8.0993 & 7.6338 & \underline{9.2065} & 8.8861 & 9.0169 & \textbf{9.5619} \\
Education (10-mask acc.) & IID & 0.5965 & 0.6496 & 0.6557 & 0.6691 & 0.6689 & \textbf{0.6878} \\
 & OOD & 0.6202 & 0.6473 & 0.6752 & 0.6831 & \underline{0.6843} & \textbf{0.7087} \\
Engineering (DCASE score) & IID & 0.5709 & 0.5492 & 0.6358 & 0.5986 & 0.6290 & \textbf{0.6535} \\
 & OOD & 0.4663 & 0.4259 & \underline{0.5049} & 0.4810 & 0.4845 & \textbf{0.5269} \\
Health sci. (Clin. utility) & IID & 0.4971 & 0.4901 & 0.5392 & 0.5209 & 0.5324 & \textbf{0.5514} \\
 & OOD & 0.6834 & 0.6375 & \underline{0.7527} & 0.7118 & 0.7167 & \textbf{0.7796} \\
Indigenous (chrF++) & IID & 15.4451 & 16.2630 & 16.5321 & 16.9189 & 16.9786 & \textbf{17.4351} \\
 & OOD & 13.7205 & 14.6476 & 14.5911 & 15.1089 & 15.1865 & \textbf{15.9778} \\
Computing sci. (pass@1) & IID & 0.7656 & 0.7969 & \underline{0.8594} & 0.8125 & 0.8438 & \textbf{0.8750} \\
 & OOD & 0.3750 & 0.3594 & \underline{0.4375} & 0.4219 & 0.4219 & \textbf{0.4531} \\
Language \& culture (LAS) & IID & 0.7288 & 0.7081 & 0.7940 & 0.7755 & 0.7810 & \textbf{0.8138} \\
 & OOD & 0.5599 & 0.5102 & \underline{0.6253} & 0.5905 & 0.6063 & \textbf{0.6497} \\
Law (mAP) & IID & 0.7513 & 0.7717 & 0.8119 & 0.8343 & 0.8248 & \textbf{0.8554} \\
 & OOD & 0.6918 & 0.6770 & 0.7644 & \underline{0.7862} & 0.7686 & \textbf{0.8289} \\
Mathematics (Oracle acc.) & IID & 0.8438 & 0.8750 & 0.9062 & 0.9219 & \underline{0.9375} & \textbf{0.9531} \\
 & OOD & 0.7344 & 0.7656 & 0.8125 & 0.8281 & \underline{0.8438} & \textbf{0.8750} \\
Philosophy (F1) & IID & 0.4178 & 0.4435 & 0.4578 & 0.4734 & 0.4713 & \textbf{0.4866} \\
 & OOD & 0.4166 & 0.4515 & 0.4765 & 0.4837 & 0.4833 & \textbf{0.5138} \\
Psychology (Micro acc.) & IID & 0.6088 & 0.6205 & 0.6320 & 0.6389 & 0.6447 & \textbf{0.6618} \\
 & OOD & 0.5615 & 0.5761 & 0.6068 & 0.6304 & \underline{0.6358} & \textbf{0.6586} \\
\midrule
\multicolumn{8}{@{}l}{\textbf{Lower is better ($\downarrow$)}} \\
Built env. (NRMSE \%) & IID & 51.0997 & 47.6123 & 46.7280 & 46.6389 & 45.9888 & \textbf{45.0065} \\
 & OOD & 53.5009 & 48.9028 & 47.0717 & 48.1634 & \underline{46.6378} & \textbf{45.0970} \\
Commerce (MASE) & IID & 1.6009 & 1.5472 & 1.5246 & 1.5024 & 1.4936 & \textbf{1.4653} \\
 & OOD & 1.8740 & 1.7993 & 1.7194 & \underline{1.6606} & 1.6702 & \textbf{1.5986} \\
Earth sci. (RMSE (K)) & IID & 1.0562 & 1.0585 & 0.9997 & 1.0137 & 0.9954 & \textbf{0.9721} \\
 & OOD & 1.1364 & 1.0419 & 0.9907 & 1.0304 & \underline{0.9896} & \textbf{0.9580} \\
Economics (sMAPE (\%)) & IID & 12.7394 & 12.4605 & 12.0901 & 11.7535 & 11.6328 & \textbf{11.3320} \\
 & OOD & 18.1343 & 17.7217 & 16.8406 & \underline{16.4118} & 16.6210 & \textbf{15.5025} \\
Environmental sci. (CRPS) & IID & 0.8104 & 0.7950 & 0.7349 & 0.7477 & 0.7383 & \textbf{0.7164} \\
 & OOD & 0.6504 & 0.6805 & \underline{0.5970} & 0.6302 & 0.6064 & \textbf{0.5751} \\
History (cMER-micro) & IID & 0.0185 & 0.0184 & 0.0175 & 0.0171 & 0.0172 & \textbf{0.0166} \\
 & OOD & 0.0296 & 0.0293 & 0.0278 & \underline{0.0263} & 0.0268 & \textbf{0.0251} \\
Human society (nRMSE) & IID & 0.0113 & 0.0111 & 0.0102 & 0.0106 & 0.0103 & \textbf{0.0100} \\
 & OOD & 0.0213 & 0.0227 & \underline{0.0189} & 0.0203 & 0.0199 & \textbf{0.0181} \\
Physical sci. (MAE) & IID & 36.6271 & 36.0411 & 32.8697 & 34.0527 & 33.1801 & \textbf{32.1581} \\
 & OOD & 44.6757 & 44.0670 & \underline{38.8958} & 40.3300 & 39.7888 & \textbf{37.1306} \\
\bottomrule
\end{tabular*}
\end{table*}

\begin{table*}[!p]
\centering
\caption{Per-discipline ablation results (2/3): linkage, candidate-source, and gate variants. Each row gives the task-native score of one discipline on one split. The best score in each row is in bold and the second best is underlined, taken over all 18 variants.}
\label{tab:app-ablation-2}
\fontsize{7.4}{8.7}\selectfont
\setlength{\tabcolsep}{0pt}
\renewcommand{\arraystretch}{1.3}
\begin{tabular*}{0.99\textwidth}{@{}>{\raggedright\arraybackslash}p{88pt}@{\hspace{4pt}}l@{\extracolsep{\fill}}cccccccc@{}}
\toprule
Discipline (metric) & Split & \multicolumn{1}{c}{\makebox[3.2em][c]{\fontsize{6.1}{6.9}\selectfont\shortstack[c]{Frozen}}} & \multicolumn{1}{c}{\makebox[3.2em][c]{\fontsize{6.1}{6.9}\selectfont\shortstack[c]{w/o link.\\evol.}}} & \multicolumn{1}{c}{\makebox[3.2em][c]{\fontsize{6.1}{6.9}\selectfont\shortstack[c]{RuleEvo}}} & \multicolumn{1}{c}{\makebox[3.2em][c]{\fontsize{6.1}{6.9}\selectfont\shortstack[c]{Qwen-only}}} & \multicolumn{1}{c}{\makebox[3.2em][c]{\fontsize{6.1}{6.9}\selectfont\shortstack[c]{w/o src.\\replay}}} & \multicolumn{1}{c}{\makebox[3.2em][c]{\fontsize{6.1}{6.9}\selectfont\shortstack[c]{w/o sci.\\constr.}}} & \multicolumn{1}{c}{\makebox[3.2em][c]{\fontsize{6.1}{6.9}\selectfont\shortstack[c]{w/o ind.\\valid.}}} & \multicolumn{1}{c}{\makebox[3.2em][c]{\fontsize{6.1}{6.9}\selectfont\shortstack[c]{ScienceClaw\\(full)}}} \\
\midrule
\multicolumn{10}{@{}l}{\textbf{Higher is better ($\uparrow$)}} \\
Agricultural sci. (PQ+) & IID & 69.4761 & 74.2713 & 73.4778 & 74.0246 & 73.4509 & 72.6174 & 74.0935 & \textbf{76.9612} \\
 & OOD & 64.6829 & 70.9070 & 69.5269 & 70.0168 & 68.0045 & 66.3097 & 68.7919 & \textbf{75.7629} \\
Biological sci. (Spearman) & IID & 0.5057 & 0.5673 & 0.5542 & 0.5567 & 0.5694 & 0.5662 & 0.5670 & \textbf{0.5816} \\
 & OOD & 0.3745 & 0.3914 & 0.3692 & 0.3816 & 0.3988 & 0.3937 & 0.3968 & \textbf{0.4234} \\
Biomedical sci. (DSC) & IID & 0.7816 & 0.8318 & 0.7942 & 0.8252 & 0.8233 & 0.8217 & 0.8205 & \textbf{0.8556} \\
 & OOD & 0.7088 & 0.7713 & 0.7506 & 0.7654 & 0.7448 & 0.7235 & 0.7334 & \textbf{0.8226} \\
Chemical sci. (ROC-AUC) & IID & 0.6518 & 0.7254 & 0.7054 & 0.7121 & 0.7321 & 0.7165 & 0.7299 & \textbf{0.7545} \\
 & OOD & 0.6205 & 0.7098 & 0.6696 & 0.7054 & 0.6607 & 0.6652 & 0.6942 & \textbf{0.7478} \\
Creative arts (SDR (dB)) & IID & 8.9208 & 9.8155 & 9.7390 & 9.5814 & 9.9029 & 9.9253 & 9.9528 & \textbf{10.1867} \\
 & OOD & 8.0993 & 8.9731 & 8.8379 & 8.7911 & 8.9484 & 8.7115 & 9.1170 & \textbf{9.5619} \\
Education (10-mask acc.) & IID & 0.5965 & 0.6700 & 0.6455 & 0.6553 & 0.6734 & 0.6702 & 0.6736 & \textbf{0.6878} \\
 & OOD & 0.6202 & 0.6744 & 0.6310 & 0.6606 & 0.6723 & 0.6551 & 0.6644 & \textbf{0.7087} \\
Engineering (DCASE score) & IID & 0.5709 & 0.6123 & 0.5655 & 0.5872 & 0.6246 & 0.6193 & 0.6199 & \textbf{0.6535} \\
 & OOD & 0.4663 & 0.4590 & 0.4296 & 0.4356 & 0.4491 & 0.4287 & 0.4544 & \textbf{0.5269} \\
Health sci. (Clin. utility) & IID & 0.4971 & 0.5371 & 0.5227 & 0.5218 & 0.5353 & 0.5227 & 0.5252 & \textbf{0.5514} \\
 & OOD & 0.6834 & 0.7389 & 0.7252 & 0.7153 & 0.7246 & 0.6841 & 0.7310 & \textbf{0.7796} \\
Indigenous (chrF++) & IID & 15.4451 & 16.8087 & 16.5814 & 16.7143 & 16.9648 & 16.9728 & 17.1113 & \textbf{17.4351} \\
 & OOD & 13.7205 & 15.1561 & 14.2354 & 15.0922 & \underline{15.2660} & 14.6767 & 15.1169 & \textbf{15.9778} \\
Computing sci. (pass@1) & IID & 0.7656 & 0.8281 & 0.8125 & 0.7969 & 0.8438 & 0.8438 & 0.8438 & \textbf{0.8750} \\
 & OOD & 0.3750 & 0.4219 & 0.3594 & 0.4062 & 0.3906 & 0.3750 & 0.4062 & \textbf{0.4531} \\
Language \& culture (LAS) & IID & 0.7288 & 0.7752 & 0.7210 & 0.7669 & 0.7945 & 0.7844 & 0.7931 & \textbf{0.8138} \\
 & OOD & 0.5599 & 0.5952 & 0.5673 & 0.5705 & 0.6074 & 0.5965 & 0.6125 & \textbf{0.6497} \\
Law (mAP) & IID & 0.7513 & 0.8312 & 0.7954 & 0.8124 & 0.8272 & 0.8034 & 0.8214 & \textbf{0.8554} \\
 & OOD & 0.6918 & 0.7825 & 0.7628 & 0.7601 & 0.7645 & 0.7504 & 0.7413 & \textbf{0.8289} \\
Mathematics (Oracle acc.) & IID & 0.8438 & 0.9219 & 0.9062 & 0.8906 & 0.9219 & 0.9062 & 0.9062 & \textbf{0.9531} \\
 & OOD & 0.7344 & 0.8281 & 0.8125 & 0.8125 & 0.7969 & 0.7812 & 0.7812 & \textbf{0.8750} \\
Philosophy (F1) & IID & 0.4178 & 0.4663 & 0.4615 & 0.4651 & 0.4623 & 0.4549 & 0.4662 & \textbf{0.4866} \\
 & OOD & 0.4166 & \underline{0.4899} & 0.4548 & 0.4779 & 0.4486 & 0.4557 & 0.4699 & \textbf{0.5138} \\
Psychology (Micro acc.) & IID & 0.6088 & 0.6340 & 0.6274 & 0.6242 & 0.6348 & 0.6196 & 0.6410 & \textbf{0.6618} \\
 & OOD & 0.5615 & 0.6262 & 0.6081 & 0.6045 & 0.5927 & 0.5988 & 0.6074 & \textbf{0.6586} \\
\midrule
\multicolumn{10}{@{}l}{\textbf{Lower is better ($\downarrow$)}} \\
Built env. (NRMSE \%) & IID & 51.0997 & 46.4669 & 47.4059 & 47.7934 & 46.1659 & 46.5516 & 46.1180 & \textbf{45.0065} \\
 & OOD & 53.5009 & 47.1892 & 49.0742 & 48.2900 & 48.0628 & 48.9215 & 47.8586 & \textbf{45.0970} \\
Commerce (MASE) & IID & 1.6009 & 1.5232 & 1.5635 & 1.5450 & 1.4938 & 1.5077 & 1.5051 & \textbf{1.4653} \\
 & OOD & 1.8740 & 1.7003 & 1.7060 & 1.7018 & 1.6811 & 1.7302 & 1.6884 & \textbf{1.5986} \\
Earth sci. (RMSE (K)) & IID & 1.0562 & 1.0006 & 1.0417 & 1.0213 & 0.9910 & 0.9962 & 0.9909 & \textbf{0.9721} \\
 & OOD & 1.1364 & 0.9957 & 1.0595 & 1.0322 & 1.0155 & 1.0108 & 1.0065 & \textbf{0.9580} \\
Economics (sMAPE (\%)) & IID & 12.7394 & 12.1124 & 12.5736 & 12.2134 & 11.6830 & 11.6889 & 11.7405 & \textbf{11.3320} \\
 & OOD & 18.1343 & 17.7245 & 17.7553 & 17.7295 & 16.6577 & 16.8062 & 16.5839 & \textbf{15.5025} \\
Environmental sci. (CRPS) & IID & 0.8104 & 0.7403 & 0.7671 & 0.7565 & 0.7363 & 0.7355 & 0.7317 & \textbf{0.7164} \\
 & OOD & 0.6504 & 0.6010 & 0.6318 & 0.6125 & 0.6150 & 0.6148 & 0.5994 & \textbf{0.5751} \\
History (cMER-micro) & IID & 0.0185 & 0.0175 & 0.0184 & 0.0180 & 0.0174 & 0.0176 & 0.0173 & \textbf{0.0166} \\
 & OOD & 0.0296 & 0.0278 & 0.0303 & 0.0280 & 0.0286 & 0.0292 & 0.0283 & \textbf{0.0251} \\
Human society (nRMSE) & IID & 0.0113 & 0.0106 & 0.0107 & 0.0109 & 0.0104 & 0.0104 & 0.0104 & \textbf{0.0100} \\
 & OOD & 0.0213 & 0.0206 & 0.0209 & 0.0215 & 0.0212 & 0.0205 & 0.0205 & \textbf{0.0181} \\
Physical sci. (MAE) & IID & 36.6271 & 33.5883 & 34.3006 & 34.3803 & 33.4032 & 33.7975 & 33.4801 & \textbf{32.1581} \\
 & OOD & 44.6757 & 39.8920 & 43.7427 & 42.3487 & 41.5032 & 40.9279 & 39.7663 & \textbf{37.1306} \\
\bottomrule
\end{tabular*}
\end{table*}

\begin{table*}[!p]
\centering
\caption{Per-discipline ablation results (3/3): memory, selection, transfer, and execution variants. Each row gives the task-native score of one discipline on one split. The best score in each row is in bold and the second best is underlined, taken over all 18 variants.}
\label{tab:app-ablation-3}
\fontsize{7.4}{8.7}\selectfont
\setlength{\tabcolsep}{0pt}
\renewcommand{\arraystretch}{1.3}
\begin{tabular*}{0.99\textwidth}{@{}>{\raggedright\arraybackslash}p{88pt}@{\hspace{4pt}}l@{\extracolsep{\fill}}cccccccc@{}}
\toprule
Discipline (metric) & Split & \multicolumn{1}{c}{\makebox[3.2em][c]{\fontsize{6.1}{6.9}\selectfont\shortstack[c]{Frozen}}} & \multicolumn{1}{c}{\makebox[3.2em][c]{\fontsize{6.1}{6.9}\selectfont\shortstack[c]{w/o cross-\\task mem.}}} & \multicolumn{1}{c}{\makebox[3.2em][c]{\fontsize{6.1}{6.9}\selectfont\shortstack[c]{w/o IID\\select.}}} & \multicolumn{1}{c}{\makebox[3.2em][c]{\fontsize{6.1}{6.9}\selectfont\shortstack[c]{w/o OOD\\transfer}}} & \multicolumn{1}{c}{\makebox[3.2em][c]{\fontsize{6.1}{6.9}\selectfont\shortstack[c]{Fixed\\workflow}}} & \multicolumn{1}{c}{\makebox[3.2em][c]{\fontsize{6.1}{6.9}\selectfont\shortstack[c]{Single\\operator}}} & \multicolumn{1}{c}{\makebox[3.2em][c]{\fontsize{6.1}{6.9}\selectfont\shortstack[c]{Single\\turn}}} & \multicolumn{1}{c}{\makebox[3.2em][c]{\fontsize{6.1}{6.9}\selectfont\shortstack[c]{ScienceClaw\\(full)}}} \\
\midrule
\multicolumn{10}{@{}l}{\textbf{Higher is better ($\uparrow$)}} \\
Agricultural sci. (PQ+) & IID & 69.4761 & 75.3078 & 69.6171 & \underline{76.3321} & 68.5424 & 64.1067 & 66.3946 & \textbf{76.9612} \\
 & OOD & 64.6829 & 71.8840 & 67.5515 & 68.0541 & 66.2617 & 61.2990 & 63.0579 & \textbf{75.7629} \\
Biological sci. (Spearman) & IID & 0.5057 & 0.5666 & 0.5282 & \underline{0.5754} & 0.5460 & 0.4895 & 0.5136 & \textbf{0.5816} \\
 & OOD & 0.3745 & 0.3748 & 0.3682 & 0.3490 & 0.3891 & 0.3462 & 0.3525 & \textbf{0.4234} \\
Biomedical sci. (DSC) & IID & 0.7816 & 0.8399 & 0.7767 & \underline{0.8478} & 0.7914 & 0.7035 & 0.7529 & \textbf{0.8556} \\
 & OOD & 0.7088 & 0.7709 & 0.7210 & 0.7665 & 0.6973 & 0.6162 & 0.6890 & \textbf{0.8226} \\
Chemical sci. (ROC-AUC) & IID & 0.6518 & 0.7388 & 0.7054 & \underline{0.7478} & 0.7076 & 0.6607 & 0.6652 & \textbf{0.7545} \\
 & OOD & 0.6205 & 0.7076 & 0.6786 & 0.6719 & 0.6853 & 0.6094 & 0.6250 & \textbf{0.7478} \\
Creative arts (SDR (dB)) & IID & 8.9208 & 10.0303 & 9.4131 & \underline{10.1204} & 9.3966 & 8.9996 & 8.4705 & \textbf{10.1867} \\
 & OOD & 8.0993 & 8.9915 & 8.7329 & 8.7296 & 8.1206 & 8.0085 & 7.9892 & \textbf{9.5619} \\
Education (10-mask acc.) & IID & 0.5965 & 0.6703 & 0.6339 & \underline{0.6820} & 0.6356 & 0.6386 & 0.6160 & \textbf{0.6878} \\
 & OOD & 0.6202 & 0.6601 & 0.6170 & 0.5842 & 0.6506 & 0.6508 & 0.6250 & \textbf{0.7087} \\
Engineering (DCASE score) & IID & 0.5709 & 0.6412 & 0.5343 & \underline{0.6457} & 0.5731 & 0.5372 & 0.5151 & \textbf{0.6535} \\
 & OOD & 0.4663 & 0.4930 & 0.4129 & 0.4575 & 0.4055 & 0.4062 & 0.3583 & \textbf{0.5269} \\
Health sci. (Clin. utility) & IID & 0.4971 & 0.5428 & 0.5087 & \underline{0.5465} & 0.5064 & 0.4626 & 0.4743 & \textbf{0.5514} \\
 & OOD & 0.6834 & 0.7428 & 0.6995 & 0.7250 & 0.7008 & 0.6226 & 0.6330 & \textbf{0.7796} \\
Indigenous (chrF++) & IID & 15.4451 & 16.8188 & 16.4282 & \underline{17.2104} & 16.4928 & 15.8464 & 16.1013 & \textbf{17.4351} \\
 & OOD & 13.7205 & 14.7687 & 14.3947 & 14.0731 & 14.8193 & 14.3933 & 13.7716 & \textbf{15.9778} \\
Computing sci. (pass@1) & IID & 0.7656 & \underline{0.8594} & 0.7812 & \textbf{0.8750} & 0.7969 & 0.7188 & 0.7188 & \textbf{0.8750} \\
 & OOD & 0.3750 & 0.4219 & 0.3438 & 0.4062 & 0.3750 & 0.3438 & 0.3125 & \textbf{0.4531} \\
Language \& culture (LAS) & IID & 0.7288 & 0.7950 & 0.7255 & \underline{0.8043} & 0.7447 & 0.6608 & 0.6514 & \textbf{0.8138} \\
 & OOD & 0.5599 & 0.5962 & 0.5203 & 0.5222 & 0.5601 & 0.5140 & 0.4821 & \textbf{0.6497} \\
Law (mAP) & IID & 0.7513 & 0.8300 & 0.8027 & \underline{0.8487} & 0.7593 & 0.7614 & 0.7036 & \textbf{0.8554} \\
 & OOD & 0.6918 & 0.7643 & 0.7460 & 0.7309 & 0.7347 & 0.7289 & 0.6334 & \textbf{0.8289} \\
Mathematics (Oracle acc.) & IID & 0.8438 & \underline{0.9375} & 0.8906 & \textbf{0.9531} & 0.8750 & 0.8750 & 0.8594 & \textbf{0.9531} \\
 & OOD & 0.7344 & 0.8281 & 0.7812 & 0.8125 & 0.7969 & 0.7969 & 0.7656 & \textbf{0.8750} \\
Philosophy (F1) & IID & 0.4178 & 0.4739 & 0.4397 & \underline{0.4799} & 0.4583 & 0.4505 & 0.4438 & \textbf{0.4866} \\
 & OOD & 0.4166 & 0.4592 & 0.4369 & 0.4108 & 0.4606 & 0.4349 & 0.4516 & \textbf{0.5138} \\
Psychology (Micro acc.) & IID & 0.6088 & 0.6406 & 0.6220 & \underline{0.6538} & 0.6231 & 0.6065 & 0.5929 & \textbf{0.6618} \\
 & OOD & 0.5615 & 0.6190 & 0.6025 & 0.5880 & 0.6012 & 0.5822 & 0.5867 & \textbf{0.6586} \\
\midrule
\multicolumn{10}{@{}l}{\textbf{Lower is better ($\downarrow$)}} \\
Built env. (NRMSE \%) & IID & 51.0997 & 45.7927 & 48.0652 & \underline{45.4619} & 47.8306 & 48.6362 & 50.1537 & \textbf{45.0065} \\
 & OOD & 53.5009 & 47.9392 & 49.4947 & 48.9009 & 48.1806 & 50.7238 & 50.0930 & \textbf{45.0970} \\
Commerce (MASE) & IID & 1.6009 & 1.4982 & 1.5660 & \underline{1.4803} & 1.5406 & 1.6092 & 1.6231 & \textbf{1.4653} \\
 & OOD & 1.8740 & 1.7054 & 1.7467 & 1.7989 & 1.7663 & 1.7836 & 1.8382 & \textbf{1.5986} \\
Earth sci. (RMSE (K)) & IID & 1.0562 & 0.9920 & 1.0375 & \underline{0.9774} & 1.0314 & 1.1146 & 1.0772 & \textbf{0.9721} \\
 & OOD & 1.1364 & 1.0036 & 1.0443 & 1.0669 & 1.0511 & 1.0846 & 1.0636 & \textbf{0.9580} \\
Economics (sMAPE (\%)) & IID & 12.7394 & 11.7781 & 12.9520 & \underline{11.4867} & 12.1341 & 12.3341 & 12.9574 & \textbf{11.3320} \\
 & OOD & 18.1343 & 17.2980 & 19.5916 & 19.3474 & 17.7193 & 17.8481 & 18.8785 & \textbf{15.5025} \\
Environmental sci. (CRPS) & IID & 0.8104 & 0.7388 & 0.7874 & \underline{0.7239} & 0.7701 & 0.8055 & 0.8236 & \textbf{0.7164} \\
 & OOD & 0.6504 & 0.6285 & 0.6415 & 0.6500 & 0.6700 & 0.6914 & 0.6869 & \textbf{0.5751} \\
History (cMER-micro) & IID & 0.0185 & 0.0170 & 0.0187 & \underline{0.0168} & 0.0187 & 0.0185 & 0.0200 & \textbf{0.0166} \\
 & OOD & 0.0296 & 0.0270 & 0.0296 & 0.0308 & 0.0305 & 0.0290 & 0.0320 & \textbf{0.0251} \\
Human society (nRMSE) & IID & 0.0113 & 0.0102 & 0.0114 & \underline{0.0101} & 0.0112 & 0.0113 & 0.0115 & \textbf{0.0100} \\
 & OOD & 0.0213 & 0.0196 & 0.0219 & 0.0205 & 0.0216 & 0.0237 & 0.0225 & \textbf{0.0181} \\
Physical sci. (MAE) & IID & 36.6271 & 32.7986 & 35.5705 & \underline{32.4657} & 35.5612 & 36.3803 & 37.5603 & \textbf{32.1581} \\
 & OOD & 44.6757 & 38.9306 & 44.6109 & 41.5977 & 42.5585 & 45.1933 & 47.4832 & \textbf{37.1306} \\
\bottomrule
\end{tabular*}
\end{table*}

\end{document}